\PassOptionsToPackage{numbers,compress}{natbib}
\documentclass{article}

\usepackage[preprint]{neurips_2026}
\usepackage[utf8]{inputenc}
\usepackage[T1]{fontenc}
\usepackage{hyperref}
\usepackage{url}
\usepackage{booktabs}
\usepackage{amsmath,amssymb,amsfonts,amsthm}
\usepackage{nicefrac}
\usepackage{microtype}
\microtypesetup{nopatch=eqnum}
\usepackage[dvipsnames]{xcolor}
\usepackage{mdframed}
\usepackage{graphicx}
\usepackage{wrapfig}
\usepackage{algorithm}
\usepackage{algorithmic}
\usepackage{multirow}
\usepackage{alphalph}
\usepackage{titletoc}

\graphicspath{{figure/}}

\makeatletter
\renewcommand{\tagform@}[1]{\maketag@@@{(\ignorespaces#1\unskip)}}
\renewcommand\appendix{%
  \par
  \setcounter{section}{0}%
  \setcounter{subsection}{0}%
  \gdef\thesection{\AlphAlph{\value{section}}}%
  \gdef\theHsection{appendix.\AlphAlph{\value{section}}}%
  \gdef\theHsubsection{appendix.\AlphAlph{\value{section}}.\arabic{subsection}}}
\makeatother

\newtheorem{theorem}{Theorem}
\newtheorem{proposition}{Proposition}
\newtheorem{definition}{Definition}
\newtheorem{remark}{Remark}
\newtheorem{corollary}{Corollary}
\newtheorem{lemma}{Lemma}

\definecolor{theoremgray}{gray}{0.94}
\mdfdefinestyle{theoremgraybox}{
  backgroundcolor=theoremgray,
  linecolor=theoremgray,
  linewidth=0pt,
  roundcorner=0pt,
  leftmargin=0pt,
  rightmargin=0pt,
  innerleftmargin=0pt,
  innerrightmargin=0pt,
  innertopmargin=0pt,
  innerbottommargin=0pt,
  skipabove=8pt,
  skipbelow=8pt
}
\definecolor{questiongray}{gray}{0.96}
\mdfdefinestyle{questionboxstyle}{
  backgroundcolor=questiongray,
  linecolor=black!18,
  linewidth=0.4pt,
  roundcorner=0pt,
  leftmargin=0pt,
  rightmargin=0pt,
  innerleftmargin=8pt,
  innerrightmargin=8pt,
  innertopmargin=7pt,
  innerbottommargin=7pt,
  skipabove=8pt,
  skipbelow=8pt
}
\newenvironment{questionbox}{\begin{mdframed}[style=questionboxstyle]\small}{\end{mdframed}}
\newcommand{\applytheorembox}[1]{\surroundwithmdframed[style=theoremgraybox]{#1}}
\newcommand{\E}{\mathbb{E}}
\newcommand{\Eqref}[1]{Eq.~\ref{#1}}
\newcommand{\seedmark}[1]{\hspace{0.08em}\textsuperscript{[#1]}}
\newcommand{\benchval}[3]{$(#1 \pm #2)\times 10^{#3}$}
\newcommand{\benchvalseed}[4]{$(#1 \pm #2)\times 10^{#3}$\seedmark{#4}}
\newcommand{\benchvalbf}[3]{\textbf{\boldmath$(#1 \pm #2)\times 10^{#3}$}}
\newcommand{\benchvalbfseed}[4]{\textbf{\boldmath$(#1 \pm #2)\times 10^{#3}$}\seedmark{#4}}
\newcommand{\benchvalul}[3]{\underline{$(#1 \pm #2)\times 10^{#3}$}}

\newcommand{\UAM}{\texorpdfstring{\textcolor{Emerald}{UAM}}{UAM}}
\newcommand{\CAM}{\texorpdfstring{\textcolor{Violet}{CAM}}{CAM}}
\newcommand{\LCAM}{\texorpdfstring{\textcolor{ProcessBlue}{LCAM}}{LCAM}}
\newcommand{\FAMOUAM}{FAMO+\UAM}
\newcommand{\GNUAM}{GN+\UAM}
\newcommand{\FAMOLCAM}{FAMO+\LCAM}
\newcommand{\GNLCAM}{GN+\LCAM}
\applytheorembox{definition}

\title{Per-Loss Adapters for Gradient Conflict\\
in Physics-Informed Neural Networks}
\author{
  \parbox{0.95\textwidth}{\centering\normalfont
    \textbf{Bum Jun Kim}\textsuperscript{1,*}
    \qquad
    \textbf{Gnankan Landry Regis N'guessan}\textsuperscript{2,3,4}\\[0.35em]
    \textsuperscript{1}\,The University of Tokyo, Japan\\
    \textsuperscript{2}\,Axiom Research Group\\
    \textsuperscript{3}\,Department of Applied Mathematics and Computational Science,
    NM-AIST,
    Tanzania\\
    \textsuperscript{4}\,African Institute for Mathematical Sciences (AIMS),
    Research and Innovation Centre, Rwanda\\[0.35em]
    \textsuperscript{*}\,Corresponding author:
    \texttt{bumjun.kim@weblab.t.u-tokyo.ac.jp}
  }
}
\date{}

\begin{document}
\maketitle

\begin{abstract}
Physics-informed neural networks (PINNs) train a single neural approximation by minimizing multiple physics- and data-derived losses, but the gradients of these losses often interfere and can stall optimization. Existing remedies typically treat this pathology either through scalar loss balancing or full-parameter-space gradient surgery, leaving it unclear which intervention is most appropriate. We show that PINN gradient conflict is not a uniform failure mode with one universal remedy. Instead, we identify distinct PINN gradient-conflict regimes, each associated with a different intervention class. Persistent directional conflict may require separate loss-indexed parameter subspaces, magnitude imbalance often favors scalar reweighting, and low or transient conflict may require no extra mitigation. To select between scalar reweighting and a lightweight architectural intervention, we propose a diagnostic-first framework. It profiles a 1000-step unmodified PINN run and, when intervention is warranted, uses one low-rank adapter per loss to create explicit loss-indexed parameter subspaces attached to a shared PINN trunk, providing each loss with a direct gradient pathway. Across more than 60 PDE configurations, including forward, inverse, multi-physics, parameter-varying, and high-dimensional problems up to 50D, persistent directional conflict dominates standard forward $K=3$ benchmarks and a natural $K=4$ thermoelastic system, where adapters combined with reweighting yield significant improvements. In contrast, $K=3$ inverse problems and natural $K=5$ and $K=6$ multi-physics systems are largely magnitude-dominated and often favor reweighting alone, while full-parameter-space gradient surgery can fail on heterogeneous parameter spaces. A regime-transition theorem and a blockwise neural tangent kernel analysis explain why gains from these separate adapter parameter subspaces arise selectively in persistent-conflict regimes.
\end{abstract}

\section{Introduction}
\label{sec:intro}
Physics-informed neural networks (PINNs)~\citep{raissi2019physics, karniadakis2021physics} embed physical laws directly into neural network training by minimizing a composite loss:
\begin{equation}
\label{eq:pinn-loss}
  \mathcal{L}(\theta) = \sum_{k=1}^{K} \mathcal{L}_k(\theta),
\end{equation}
where $\mathcal{L}_k$ includes partial differential equation (PDE) residuals $\mathcal{L}_{\mathrm{PDE}}(\theta)$, boundary-condition (BC) terms $\mathcal{L}_{\mathrm{BC}}(\theta)$, initial-condition (IC) terms $\mathcal{L}_{\mathrm{IC}}(\theta)$, and data-fidelity terms for inverse problems. Despite their elegance, PINNs suffer from well-documented training difficulties~\citep{wang2021understanding, wang2022and, krishnapriyan2021characterizing, leiteritz2021trivial, daw2023mitigating, wang2023expertguide}. The gradients $g_k = \nabla_\theta \mathcal{L}_k$ frequently point in conflicting directions, causing destructive interference that impedes convergence; see Figure~\ref{fig:gradient-conflict}.

\begin{wrapfigure}{r}{0.45\textwidth}
  \vspace{-0.6\baselineskip}
  \centering
  \includegraphics[width=\linewidth]{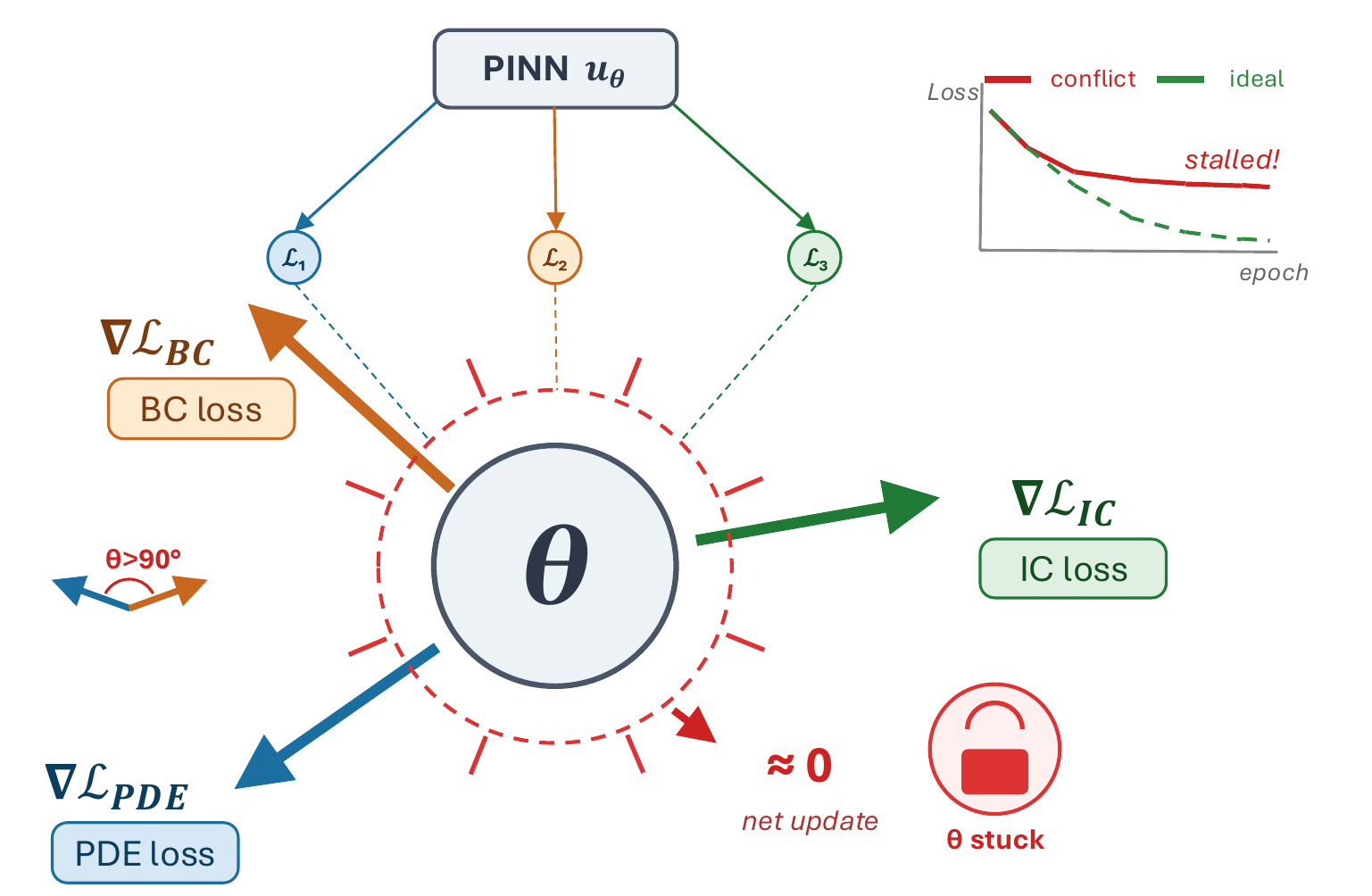}
  \caption{%
    Gradient conflict in PINNs.
    A PINN $u_{\boldsymbol{\theta}}$ is trained with
    multiple loss terms, whose gradients
    $\nabla_{\boldsymbol{\theta}}\mathcal{L}_{k}$ may point in opposing
    directions,
    which causes the training to stall.%
  }
  \label{fig:gradient-conflict}
  \vspace{-0.7\baselineskip}
\end{wrapfigure}

Existing approaches to addressing gradient conflicts in PINNs fall into two categories: scalar loss-balancing methods and full-parameter-space gradient-surgery methods.
The first category includes loss reweighting methods such as learning rate annealing, gradient normalization (GradNorm), SoftAdapt, and Fast Adaptive Multitask Optimization (FAMO)~\citep{wang2021understanding, chen2018gradnorm, heydari2019softadapt, liu2024famo}, which adjust the scalar weights $\lambda_k$ in $\sum_k \lambda_k \mathcal{L}_k$ to balance loss magnitudes but do not modify gradient directions.
The second category includes gradient-surgery methods such as projected conflicting gradient (PCGrad), conflict-averse gradient descent (CAGrad), Gradient Vaccine, RotoGrad, and Nash multi-task learning (NashMTL)~\citep{yu2020gradient, liu2021conflict, wang2021gradvac, javaloy2022rotograd, navon2022multi}, which project, rotate, or otherwise transform conflicting gradients to remove negative interference but operate on the full parameter space, are computationally expensive, and can fail catastrophically in heterogeneous PINN settings, especially inverse problems that combine network weights with learnable physical constants. Tables~\ref{tab:inverse_10k} and~\ref{tab:surgery_mitigation} isolate this failure mode through inverse-problem results and grouped-surgery mitigation.
Recent PINN-specific analyses and optimizers likewise make gradient conflict explicit~\citep{hwang2024dualcone, wang2025gradientalignment}, but they still do not provide a regime-level rule for when architectural separation is preferable to scalar reweighting.

Rather than advocating a single method, our work argues that PINN gradient conflict should be treated as a regime-dependent optimization pathology. Different PDE-loss systems may require different classes of interventions, even when all exhibit multi-loss interference.
We characterize the full landscape of gradient conflict in PINNs across directional conflict, magnitude imbalance, and their $K$-dependent interaction, providing a principled framework for regime-dependent method selection.
Here, a conflict regime refers to the dominant obstruction observed in a short profile, namely persistent directional misalignment, gradient-norm imbalance, low conflict, or transient conflict. This regime determines whether scalar reweighting, architectural separation, or no additional intervention is appropriate.

The central distinction is that PINN losses are not independent tasks in the usual multi-task learning (MTL) sense, because the PDE residual, BC, IC, and data terms all constrain the same physical field $u(x)$.
Their gradient geometry is therefore tied to the PDE operator and the physics-loss decomposition, with all losses acting through a shared physical output instead of distinct task heads.
This distinction motivates the guiding question of our study:
\begin{questionbox}
\centering
\emph{When can PINN gradient conflicts be resolved through scalar reweighting, and when do PDE-induced persistent directional conflicts require explicit separate parameter subspaces?}
\end{questionbox}
No prior work provides a theory-guided answer to determine when each algorithm class is appropriate, validated across different operator families, dimensions, and architectures.
We fill this gap with a diagnostic framework that determines whether a conflict warrants intervention. This framework is supported by regime-transition theory and a kernel-level interpretation of per-loss adapters, and it adopts per-loss low-rank adapters~\citep{hu2022lora} as lightweight architectural probes that explicitly instantiate separate loss-indexed parameter subspaces with a 15\% parameter overhead. Their gains precisely track the predicted conflict regime.

Our contributions are as follows. First, we show that PINN gradient conflict forms regimes tied to loss count and the problem class. Standard forward $K=3$ PDEs and natural $K=4$ thermoelasticity are dominated by directional conflict, whereas $K=3$ inverse problems and natural $K=5$ and $K=6$ multi-physics problems are mostly magnitude-dominated.
Second, we explain this transition with Theorem~\ref{thm:regime_transition}, operator-type guidance in Remark~\ref{thm:pde_conflict}, and a blockwise neural tangent kernel (NTK) view showing how loss-indexed low-rank adapters create separate parameter subspaces without requiring separate physical outputs.
Third, we validate a 1000-step diagnostic workflow across 60+ PDE configurations, architecture variants, MTL benchmarks, and NTK-weighting comparisons, identifying when scalar reweighting, per-loss adapters, or no additional intervention is appropriate.

\section{Method}
\label{sec:method}

\subsection{Problem Setup}
\label{sec:problem}
Consider a PDE-constrained problem defined on a domain $\Omega \subset \mathbb{R}^d$ with boundary $\partial\Omega$, written at the residual level as
\begin{align*}
\begin{aligned}
  \mathcal{R}_{\mathrm{PDE}}[u](x)
  := \mathcal{F}[u](x)-f(x) &= 0
  && \text{for } x \in \Omega, \\
  \mathcal{R}_m[u](x) &= 0
  && \text{for } x \in \Gamma_m,\quad m=1,\ldots,M.
\end{aligned}
\end{align*}
Here $\mathcal{F}$ is the differential operator, $f$ is a forcing or source term, and the residuals $\mathcal{R}_m$ on the sets $\Gamma_m$ encode BC, IC, interface, or observation constraints.
A PINN approximates the unknown physical field $u$, the state variable or physical quantity constrained by the governing equations, using a neural network $u_\theta$ trained by minimizing the composite objective in \Eqref{eq:pinn-loss}.
In the PDE setting, its components typically include $\mathcal{L}_{\text{PDE}}(\theta)$, $\mathcal{L}_{\text{BC}}(\theta)$, $\mathcal{L}_{\text{IC}}(\theta)$, $\ldots$.

We formalize the resulting multi-loss training pathology as follows.

\begin{definition}[Gradient conflict]
\label{def:conflict}
Two loss terms $\mathcal{L}_i, \mathcal{L}_j$ exhibit gradient conflict at parameters $\theta$ if their nonzero gradients have a negative cosine similarity:
\begin{equation}
  \cos(g_i, g_j) = \frac{g_i^\top g_j}{\|g_i\| \|g_j\|} < 0,
\end{equation}
where $g_k = \nabla_\theta \mathcal{L}_k(\theta)$.
We further define the regularized, magnitude-aware conflict score used in profiling and reporting as:
\begin{equation}
\label{eq:conflict-score}
  C_{ij} =
  \left[
    \frac{g_i^\top g_j}{(\|g_i\| + \varepsilon_g)(\|g_j\| + \varepsilon_g)}
  \right]_- \cdot
  \left(
    1 + \left|\log \frac{\|g_i\| + \varepsilon_g}{\|g_j\| + \varepsilon_g}\right|
  \right),
\end{equation}
where $[\cdot]_- = \max(0, -\cdot)$ retains only the negative cosine similarity, and $\varepsilon_g > 0$ is a fixed numerical floor. The magnitude ratio term accounts for gradient norm imbalance, which is a critical factor in PINNs because PDE residual gradients can be orders of magnitude larger than boundary gradients~\citep{wang2021understanding}.
\end{definition}

In all experiments, we use $\varepsilon_g = 10^{-12}$ only to avoid division by zero or $\log 0$ in degenerate cases. When both gradient norms are strictly positive and much larger than $\varepsilon_g$, \Eqref{eq:conflict-score} reduces to the usual cosine and log-ratio expression.

Definition~\ref{def:conflict} is used as a diagnostic tool for identifying a pathology observed during PINN training.
Remark~\ref{thm:pde_conflict} and Appendix~\ref{app:proof_pde_conflict} explain how PDE residuals, BC errors, operator coupling, and network-Jacobian correlations can induce negative alignment between loss gradients.
Empirical evidence in the appendix connects persistent conflict to structured errors, adapter specialization, thermoelastic energy mismatch, and unresolved directional conflict under reweighting.
For $K$ loss terms, the number of pairwise conflict terms grows as $\binom{K}{2}$, creating $\binom{K}{2}$ potential interference sources that must be managed.

\subsection{Profiling Gradient Conflict for Regime Selection}
\label{sec:profiling}
To diagnose gradient conflicts before intervention, we propose a profiling stage for the unmodified PINN.
We run a short Vanilla optimization window of $T_{\mathrm{prof}}=1000$ steps and compute the loss-specific gradients $g_k^{(t)} := \nabla_\theta \mathcal{L}_k(\theta_t)$.
For each step, we summarize directional conflict by the fraction and average strength of negative cosine pairs,
\begin{align*}
  f_{\mathrm{neg}}^{(t)}
  := \frac{|\{(i,j): i<j, \cos(g_i^{(t)},g_j^{(t)})<0\}|}{\binom{K}{2}},
  \qquad
  D^{(t)} := f_{\mathrm{neg}}^{(t)} |\bar c_-^{(t)}|,
\end{align*}
where $\bar c_-^{(t)}$ is the mean cosine over the negative pairs and is set to $0$ when no negative pair exists.
We summarize magnitude imbalance by the coefficient of variation of gradient norms,
\begin{align*}
  M^{(t)} :=
  \frac{\operatorname{std}_k(\|g_k^{(t)}\|)}
       {K^{-1}\sum_k \|g_k^{(t)}\|+\varepsilon_g}.
\end{align*}
The profiled summaries are
\begin{align*}
  \widehat f_{\mathrm{neg}} := \frac{1}{T_{\mathrm{prof}}}\sum_t f_{\mathrm{neg}}^{(t)},\quad
  \widehat D := \frac{1}{T_{\mathrm{prof}}}\sum_t D^{(t)},\quad
  \widehat M := \frac{1}{T_{\mathrm{prof}}}\sum_t M^{(t)},\quad
  \widehat R := \frac{\widehat D}{\widehat M+\varepsilon_R},
\end{align*}
where $\varepsilon_R=10^{-8}$.
The ratio $\widehat R$ is a practical surrogate for the theory-level directional-to-magnitude ratio $\mathcal U_K$, not a plug-in estimator: $\widehat M$ uses a coefficient-of-variation proxy, whereas Theorem~\ref{thm:regime_transition} analyzes a max-to-mean imbalance.
For deployment, we also measure persistence by comparing the first and last thirds of the profile,
\begin{align*}
  P :=
  \frac{\bar f_{\mathrm{neg}}^{\mathrm{late}}}
       {\max(\bar f_{\mathrm{neg}}^{\mathrm{early}},\varepsilon_P)},
\end{align*}
where $\varepsilon_P=10^{-8}$; we also compute the least-squares slope of $t \mapsto f_{\mathrm{neg}}^{(t)}$.
Low $\widehat f_{\mathrm{neg}}$ or rapidly decaying conflict indicates that scalar reweighting is usually sufficient, whereas persistent conflict motivates the adapter intervention described below.
Appendix~\ref{app:diagnostic_defs} gives the zero-gradient convention and full profile definitions, and Appendix Algorithm~\ref{alg:selection} gives executable pseudocode.
These profiled diagnostics determine whether scalar reweighting is sufficient or whether explicit loss-indexed architectural separation is needed.

The architectural intervention used in persistent-conflict regimes has two core components. The first is a shared feature trunk, meaning the common network body that computes the hidden representation used by all losses. The second is a set of per-loss low-rank adapters, whose parameters form explicit loss-indexed subspaces and whose outputs are mixed back into the shared representation. The default fixed mixer is Uniform Adapter Mixing (\UAM{}), which uses equal adapter weights. For $K\geq4$, these components can be augmented with Conflict-aware Adapter Mixing (\CAM{}) variants, including Layerwise Conflict-aware Adapter Mixing (\LCAM{}). The \CAM{} and \LCAM{} definitions apply to any $K$; our $K=3$ stress-test ablations use them only to assess whether dynamic mixing is worthwhile in the low-$K$ regime. See Appendix~\ref{app:method_details} for the shared trunk architecture, orthogonality regularization, conflict-aware mixing variant, and the full training algorithm in Algorithm~\ref{alg:training}; the main text focuses on the shared-output adapter construction and the default method choice.

\subsection{Per-Loss Low-Rank Adapters}
\label{sec:adapters}
\paragraph{Why per-loss adapters for PINNs.} In PINNs, all losses constrain the same physical field $u(x)$ rather than representing independent tasks with separate outputs.
This shared-output structure means that the shared trunk naturally captures the common representation, while adapters provide targeted adjustments to the gradient flow.
When the PDE residual gradient $g_{\text{PDE}}$ conflicts with the BC gradient $g_{\text{BC}}$, the standard shared-parameter approach forces a destructive compromise.
Per-loss adapters instead introduce structured low-rank parameter blocks whose Jacobian contributions are encouraged to occupy distinct adapter subspaces through the orthogonality loss in \Eqref{eq:ortho} and are interpreted via the blockwise NTK decomposition in Appendix~\ref{sec:blockwise_ntk}; because these blocks are mixed into the hidden state, as shown in \Eqref{eq:adapted}, they can decorrelate the Jacobians of different losses while maintaining a coherent shared output through the trunk.
Thus, the proposed adapters explicitly create separate loss-indexed parameter subspaces within a single-output PINN, rather than replacing the shared physical field with separate task heads.
Figure~\ref{fig:per_loss_adapter_pipeline} summarizes this block-level pipeline.

\begin{figure}[t!]
  \centering
  \includegraphics[width=\linewidth]{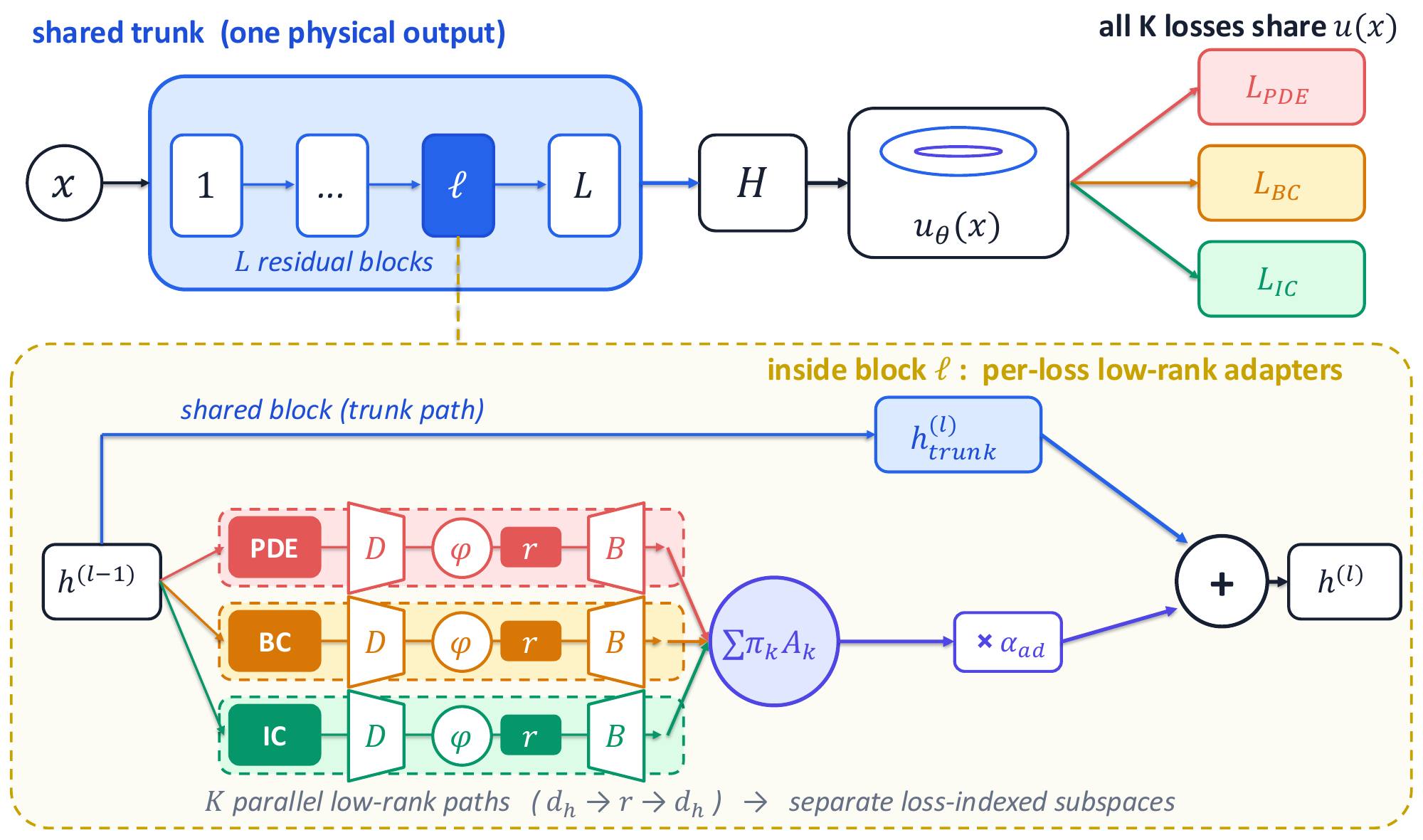}
  \caption{%
    Block-level pipeline for per-loss low-rank adapters in a shared-output PINN.
    At residual block $\ell$, the shared trunk produces $h_{\text{trunk}}^{(\ell)}$, while each loss-indexed adapter maps the previous hidden state through a rank-$r$ bottleneck formed by a down-projection followed by an up-projection.
    The mixer forms the weighted adapter correction $\sum_k \pi_k^{(\ell)} A_k^{(\ell)}$ and adds it to the shared trunk state, as shown in \Eqref{eq:adapted}.
    A single readout $H$ then produces the physical field $u_\theta(x)$, from which all PDE, BC, IC, or data losses are evaluated; the adapters therefore separate gradient subspaces without creating separate physical outputs.%
  }
  \label{fig:per_loss_adapter_pipeline}
\end{figure}

The core innovation is attaching $K$ low-rank adapters to the shared trunk, with one for each loss term.
At each residual block $\ell$, loss $k$ has an associated adapter:
\begin{equation}
  A_k^{(\ell)}(h) = B_k^{(\ell)} \cdot \phi(D_k^{(\ell)} h),
\end{equation}
where $D_k^{(\ell)} \in \mathbb{R}^{r \times d_h}$ is the down-projection, $B_k^{(\ell)} \in \mathbb{R}^{d_h \times r}$ is the up-projection, $\phi$ is the adapter nonlinearity, and $r = 16$ is the adapter rank.
Let $h_{\text{trunk}}^{(\ell)}$ denote the output of the $\ell$th shared residual block when applied to the current hidden state $h^{(\ell-1)}$.
Each adapter is assigned a latent mixing score $\rho_k^{(\ell)} \in [0,1]$ for loss $k$ at layer $\ell$.
The normalized rule below depends only on the relative values of these scores and is well-defined whenever at least one adapter is active.
In fixed \UAM{}, we use uniform adapter weights $\pi_k^{(\ell)} = 1/K$; equivalently, any common constant score $\rho_k^{(\ell)} \equiv c < 1$ gives the same weights, and our implementation stores $c=0.5$.
We then combine the adapters through the normalized mixing weights
\begin{equation}
  \pi_k^{(\ell)} := \frac{1-\rho_k^{(\ell)}}{\sum_{j=1}^K (1-\rho_j^{(\ell)})}
\end{equation}
These weights satisfy $\sum_{k=1}^K \pi_k^{(\ell)} = 1$, and they define the actual shared hidden-state update
\begin{equation}
\label{eq:adapted}
  h^{(\ell)} = h_{\text{trunk}}^{(\ell)} + \alpha_{\text{ad}} \sum_{k=1}^K \pi_k^{(\ell)} A_k^{(\ell)}(h^{(\ell-1)}),
\end{equation}
where $\alpha_{\text{ad}}$ controls the overall strength of the adapter correction added to the shared trunk representation.
The final output is the single physical field $u_\theta(x) = H(h^{(L)})$, where $H$ is the final linear readout. Each loss $\mathcal{L}_k$ is evaluated on quantities derived from this same shared output, such as boundary values, PDE residuals, or IC mismatches.
The loss-indexed mixer therefore modulates only how strongly each adapter subspace perturbs the common representation; it does not create separate physical outputs.

\subsection{Method variants and recommended defaults}\label{sec:variants}
For compact reporting, GradNorm (GN) denotes the existing GradNorm method, while \CAM{} refers to our global mixing variant, and \LCAM{} to its layer-wise variant.
\FAMOUAM{} combines uniform adapter mixing with loss-value-based reweighting and is the recommended default, offering the best accuracy-cost trade-off with a training-time cost of about $1.3\times$ Vanilla.
\GNUAM{} is an alternative for solving spectral-bias-dominated PDEs such as Helmholtz.
Alternative adapter architectures and mixing variants are summarized in Appendix Tables~\ref{tab:alt_adapters} and~\ref{tab:mixing_ablation}.

\subsection{Complementarity}\label{sec:complementarity}
Adapters resolve directional conflict, defined by $\cos(g_i, g_j) < 0$, through separate parameter subspaces. Reweighting resolves magnitude imbalance, defined by $\|g_i\| \gg \|g_j\|$, through scalar multipliers $\lambda_k$.
Because these mechanisms act on different aspects of the multi-loss update, they can be complementary in practice (Section~\ref{sec:experiments}).

\section{Theoretical Analysis}
\label{sec:theory}
We organize the theory around the directional-to-magnitude ratio
\begin{align*}
  \mathcal U_K := \frac{D_K}{M_{\mathrm{eff}}(K)},
\end{align*}
where $D_K$ denotes pairwise negative directional conflict and $M_{\mathrm{eff}}(K)$ denotes update-level magnitude imbalance in the model analyzed below. The theorem below gives the formal transition statement, while Appendix~\ref{app:diagnostic_defs} gives the profiling definitions used for the empirical diagnostic quantities. See Appendix~\ref{app:theory_details} for additional supporting theory.

The empirical observation that the balance between directional conflict and magnitude imbalance shifts with the number of loss terms is formalized here through a fixed-model transition theorem for the ratio $\mathcal U_K$. Appendix~\ref{sec:directional_magnitude_support} gives the supporting conditional concentration result for $D_K$, interpretation remarks, and the expectation-level finite-range corollary.

\begin{theorem}[Regime transition in a fixed spherical model]
\label{thm:regime_transition}
Fix a subspace $V \subseteq \mathbb{R}^p$ with
$\dim(V) = d_{\mathrm{eff}}$. Assume $g_k = m_k u_k$, where $u_1, \ldots, u_K$ are independent and uniformly distributed on
$V \cap \mathbb{S}^{p-1}$. In orthonormal coordinates on $V$, these are
independent and uniformly distributed on $\mathbb{S}^{d_{\mathrm{eff}}-1}$.
The magnitudes $m_1, \ldots, m_K$ are positive i.i.d. random variables, independent of
$u_1, \ldots, u_K$, and satisfy
\begin{align*}
  \mu_m := \E[m_1] \in (0,\infty),
  \qquad
  \mu_d := \frac{\Gamma(d_{\mathrm{eff}}/2)}
                {2\sqrt{\pi}\Gamma((d_{\mathrm{eff}}+1)/2)}.
\end{align*}
Write
\begin{align*}
  A_K &:= \max_{1 \le k \le K} m_k,
  \qquad
  \bar m := \frac{1}{K}\sum_{k=1}^K m_k,\\
  M_{\mathrm{eff}}(K) &:= \frac{A_K}{\bar m},
  \qquad
  \mathcal U_K := \frac{D_K}{M_{\mathrm{eff}}(K)}.
\end{align*}
Then, we have
\begin{align}
\label{eq:uk-factorization}
  \E[\mathcal U_K]
  = \mu_d \E\left[\frac{\bar m}{A_K}\right].
\end{align}
Moreover, the following rate statements hold.

\paragraph{Exponential-tail magnitudes.} Suppose there exist constants $\nu > 0$, $0 < a \le b < \infty$, and
$t_0 > 0$ such that for all $t \ge t_0$,
\begin{align*}
  \exp\left(-\frac{b t}{\nu}\right)
  \le \Pr(m_1 > t)
  \le \exp\left(-\frac{a t}{\nu}\right).
\end{align*}
Then
\begin{align*}
  \E[\mathcal U_K]
  = \Theta(d_{\mathrm{eff}}^{-1/2} / \log K).
\end{align*}

\paragraph{Pareto-tail magnitudes.} Suppose
\begin{align*}
  \Pr(m_1 > t) = \left(\frac{t_0}{t}\right)^\alpha,
  \qquad
  t \ge t_0,
\end{align*}
for some $\alpha > 1$. Then
\begin{align*}
  \E[\mathcal U_K]
  = \Theta(d_{\mathrm{eff}}^{-1/2} K^{-1/\alpha}).
\end{align*}
\end{theorem}

See Appendix~\ref{sec:directional_magnitude_support} for supporting statements and Appendix~\ref{app:proof_regime} for the full proofs.
Empirically, we use the short-run profile described in Section~\ref{sec:profiling} and the ratio $\widehat R := \widehat D / (\widehat M + \varepsilon_R)$ as a practical surrogate for $\mathcal U_K$.
Appendix~\ref{app:diagnostic_defs} gives the zero-gradient convention and additional reporting details for the profiled summaries.

\section{Experiments}
\label{sec:experiments}
We evaluate on 60+ PDE configurations spanning forward, inverse, multi-physics, parameter-varying, and high-dimensional settings. See Appendix~\ref{app:exp_setup} and Appendix~\ref{app:main_text_diagnostics} for detailed benchmark definitions, the natural $K$-scaling setup, compared methods, evaluation protocol, computational-cost analysis, and supplementary diagnostics. The main text focuses on the core empirical findings and the resulting deployment heuristic. The forward-benchmark tables below refer to the conflict-free gradient method (ConFIG) and aligned multi-task learning (AlignedMTL) by their standard short names, alongside the previously defined PCGrad, CAGrad, NashMTL, \UAM{}, GN, and \LCAM{} labels.

\subsection{Main Results on Forward Benchmarks}
Tables~\ref{tab:main_results_a} and~\ref{tab:main_results_b} present the full result matrices for five main PDE benchmarks at 10K epochs. Table~\ref{tab:main_results_a} groups Burgers, Helmholtz, and Allen--Cahn, while Table~\ref{tab:main_results_b} groups convection-diffusion (Conv-Diff), Klein--Gordon, and the average rank.

\begin{table}[t!]
\centering
\caption{10K-epoch benchmark I. Relative $L_2$ error on Burgers, Helmholtz, and Allen--Cahn. We report the mean $\pm$ std with $d_h=128$ and 5 seeds unless otherwise noted. Best values are bold and second-best values are underlined.}
\label{tab:main_results_a}
{\small
\begin{tabular}{lccc}
\toprule
Method & Burgers & Helmholtz & Allen--Cahn \\
\midrule
\multicolumn{4}{l}{Baselines} \\
  Vanilla & \benchval{8.4}{2.4}{-3} & \benchval{2.8}{1.1}{-2} & \benchval{1.60}{0.33}{-4} \\
\midrule
\multicolumn{4}{l}{Loss Weighting} \\
  FAMO & \benchvalul{4.7}{0.14}{-3} & \benchval{1.4}{0.63}{-3} & \benchval{2.17}{0.2}{-4} \\
  GradNorm & \benchval{1.3}{1.9}{-2} & \benchvalul{6.40}{6.4}{-4} & \benchval{5.7}{0.83}{-3} \\
  Causal & \benchvalseed{7.5}{3}{-3}{3} & \benchvalseed{5.3}{1.1}{-2}{3} & \benchvalseed{1.71}{0.096}{-4}{3} \\
  SelfAdapt & \benchvalseed{1.1}{0.8}{-2}{3} & \benchvalseed{1.2}{0.5}{-2}{3} & \benchvalseed{1.70}{0.39}{-4}{3} \\
  Uncertainty & \benchval{1.1}{0.7}{-2} & \benchvalseed{1.8}{1.7}{-2}{4} & \benchvalul{1.31}{0.27}{-4} \\
\midrule
\multicolumn{4}{l}{Gradient Surgery} \\
  ConFIG & \benchval{9.6}{2.7}{-2} & \benchvalbfseed{2.49}{0.77}{-4}{3} & $1.4\times 10^{-2}$\seedmark{1} \\
  PCGrad & \benchval{4.8}{0.11}{-3} & \benchvalseed{9.06}{2.8}{-4}{4} & $2.25\times 10^{-4}$\seedmark{1} \\
  CAGrad & \benchval{1.5}{1}{-2} & \benchval{3.3}{1.5}{-2} & $1.96\times 10^{-4}$\seedmark{1} \\
  NashMTL & \benchval{6.5}{6.7}{-2} & \benchval{1.65}{1.11}{-1} & \benchval{3.13}{0.43}{-4} \\
  AlignedMTL & \benchval{5.1}{0.34}{-3} & \benchvalseed{1.1}{0.33}{-3}{4} & $1.81\times 10^{-4}$\seedmark{1} \\
\midrule
\multicolumn{4}{l}{Architecture (Ours)} \\
  \UAM & \benchval{4.9}{0.075}{-3} & \benchval{1.5}{0.1}{-2} & \benchvalbf{1.19}{0.21}{-4} \\
  \FAMOUAM & \benchval{4.8}{0.082}{-3} & \benchval{1.1}{0.46}{-3} & \benchval{1.84}{0.28}{-4} \\
  \GNUAM & \benchvalbf{4.6}{0.12}{-3} & \benchval{1.0}{0.44}{-3} & \benchval{1.6}{0.21}{-3} \\
  \FAMOLCAM & \benchval{4.8}{0.17}{-3} & \benchval{1.3}{0.41}{-3} & \benchval{2.59}{0.39}{-4} \\
  \GNLCAM & \benchvalbf{4.6}{0.12}{-3} & \benchval{7.81}{2.8}{-4} & \benchval{2.3}{0.45}{-3} \\
\bottomrule
\multicolumn{4}{l}{\footnotesize Superscript $[n]$ denotes seed count when $<$5.}
\end{tabular}
}
\end{table}

\begin{table}[t!]
\centering
\caption{10K-epoch benchmark II. Relative $L_2$ error on Conv-Diff and Klein--Gordon together with the average rank. We report mean $\pm$ std with $d_h=128$ and 5 seeds, unless otherwise noted.}
\label{tab:main_results_b}
{\small
\begin{tabular}{lccr}
\toprule
Method & Conv-Diff & Klein--Gordon & Avg. Rank \\
\midrule
\multicolumn{4}{l}{Baselines} \\
  Vanilla & \benchval{3.55}{0.74}{-4} & \benchval{4.6}{1}{-2} & 7.8 \\
\midrule
\multicolumn{4}{l}{Loss Weighting} \\
  FAMO & \benchval{3.95}{1.4}{-4} & \benchvalul{2.5}{0.39}{-3} & 6.4 \\
  GradNorm & \benchval{1.8}{0.46}{-3} & \benchval{1.13}{0.57}{-1} & 11.4 \\
  Causal & \benchvalseed{2.98}{0.39}{-4}{3} & \benchvalseed{9.80}{0.13}{-1}{3} & 9.6 \\
  SelfAdapt & \benchvalseed{4.12}{0.91}{-4}{3} & \benchvalseed{3.8}{0.34}{-2}{3} & 8.1 \\
  Uncertainty & \benchval{3.66}{0.73}{-4} & \benchval{4.89}{4.87}{-1} & 9.3 \\
\midrule
\multicolumn{4}{l}{Gradient Surgery} \\
  ConFIG & \benchval{8.12}{4.44}{-1} & \benchval{2.39}{3}{-1} & 12.6 \\
  PCGrad & \benchval{3.92}{1.7}{-4} & \benchval{9.0}{3.7}{-3} & 6.0 \\
  CAGrad & \benchval{4.63}{1.5}{-4} & \benchval{4.7}{1}{-2} & 11.2 \\
  NashMTL & \benchval{5.41}{1.1}{-4} & \benchval{3.6}{1.1}{-2} & 12.0 \\
  AlignedMTL & \benchval{4.16}{1.4}{-4} & \benchval{1.14}{2.1}{-1} & 8.7 \\
\midrule
\multicolumn{4}{l}{Architecture (Ours)} \\
  \UAM & \benchvalbf{2.64}{0.32}{-4} & \benchval{4.3}{0.7}{-2} & \underline{5.4} \\
  \FAMOUAM & \benchvalul{2.70}{0.21}{-4} & \benchvalbf{1.6}{0.38}{-3} & \textbf{4.4} \\
  \GNUAM & \benchval{1.0}{0.19}{-3} & \benchval{4.8}{1.1}{-2} & 8.5 \\
  \FAMOLCAM & \benchval{3.07}{0.56}{-4} & \benchvalbf{1.6}{0.11}{-3} & 5.9 \\
  \GNLCAM & \benchval{1.4}{0.41}{-3} & \benchval{6.2}{1.7}{-2} & 8.8 \\
\bottomrule
\end{tabular}
}
\end{table}

Architectural and loss-weighting combinations are strongest. On 10K-epoch Burgers, \GNUAM{} and \FAMOUAM{} are statistically tied at about $4.6\times 10^{-3}$. This Burgers tie corresponds to a $1.8\times$ improvement over Vanilla, with a 95\% confidence interval (CI) of $[1, 2]$ and a Cohen's $d=2.1$. On Helmholtz, \GNUAM{} with $1.0\times 10^{-3}$ achieves a $28\times$ improvement over Vanilla, with a CI $[18, 46]$ and $d=3.4$. \FAMOUAM{} achieves $25\times$, with a CI $[16, 39]$ and the same effect size $d=3.4$. The gap persists across width and three-dimensional (3D) evaluations in Table~\ref{tab:helmholtz_scaling}.
All architecture-based methods successfully complete every benchmark without failure, whereas gradient surgery methods such as CAGrad and NashMTL diverge on inverse problems. ConFIG also shows extreme PDE-dependent variance, with $3\times 10^{-4}$ on the Helmholtz but $0.81$ on Conv-Diff, where four out of five seeds diverge.
No single method dominates, and the right method depends on the PDE conflict profile. Both causal training and self-adaptive weights further confirm this PDE dependence. Causal training matches \FAMOUAM{} on Burgers but catastrophically fails on Klein--Gordon, reaching $0.98$ at 10K and $0.92$ at 20K, whereas \FAMOUAM{} reaches $1.6\times 10^{-3}$. Self-adaptive weights also show high variance across PDEs.
Additional PDE checks in Appendix~\ref{sec:heldout} confirm these findings. On Klein--Gordon, \FAMOUAM{} achieves a $29\times$ gain, with a CI $[24, 36]$ and $d=6.1$. On Darcy flow, the gain is $2.0\times$, with a CI $[1, 3]$ and $d=2.3$.
Appendix Figure~\ref{fig:qual_adapter_specialization_regime} provides a qualitative view of the mechanism underlying regime dependence by visualizing the learned adapter contributions in a persistent-conflict case and a low-conflict control.

\subsection{Inverse Problems and Robustness}

\begin{table}[t!]
\centering
\caption{Inverse problem results at 10K epochs with $d_h=128$ and 3 to 5 seeds. Full-parameter-space gradient surgery methods diverged on these benchmarks and are omitted here. $\dagger$GradNorm entries use a simplified loss-ratio variant rather than the reference gradient-norm algorithm used in Tables~\ref{tab:main_results_a} and~\ref{tab:main_results_b}.}
\label{tab:inverse_10k}
\small
\begin{tabular}{lccc}
\toprule
Method & Inverse Burgers ($K=3$) & Inverse Heat ($K=4$) & Inverse Poisson ($K=3$) \\
\midrule
Vanilla & $(8.0 \pm 0.7)\times 10^{-2}$ & $(2.2 \pm 3.8)\times 10^{-2}$ & $(2.71 \pm 0.65)\times 10^{-1}$ \\
FAMO & \underline{$(6.5 \pm 0.4)\times 10^{-2}$} & $(1.1 \pm 2.3)\times 10^{-2}$ & \textbf{\boldmath$(2.3 \pm 1.6)\times 10^{-3}$} \\
GradNorm$^\dagger$ & \textbf{\boldmath$(5.9 \pm 0.2)\times 10^{-2}$} & $(1.1 \pm 0.1)\times 10^{-3}$ & $(1.0 \pm 1.5)\times 10^{-2}$ \\
\midrule
\FAMOUAM & $(7.2 \pm 0.4)\times 10^{-2}$ & \underline{$(2.8 \pm 2.6)\times 10^{-4}$} & \underline{$(2.8 \pm 0.9)\times 10^{-3}$} \\
\GNUAM$^\dagger$ & $(6.6 \pm 0.4)\times 10^{-2}$ & $(8.7 \pm 0.9)\times 10^{-4}$ & $(3.4 \pm 0.9)\times 10^{-3}$ \\
\FAMOLCAM & $(7.3 \pm 0.2)\times 10^{-2}$ & \textbf{\boldmath$(2.6 \pm 0.74)\times 10^{-4}$} & $(3.2 \pm 0.7)\times 10^{-3}$ \\
\bottomrule
\end{tabular}
\end{table}

Full-parameter-space gradient surgery methods, namely, PCGrad, CAGrad, NashMTL, and ConFIG, produce invalid values within the first 50 to 200 epochs during inverse-problem runs due to the $10^4\times$ gradient norm mismatch between network weights and scalar physical constants. Table~\ref{tab:surgery_mitigation} in the appendix gives the details.
At 10K epochs in Table~\ref{tab:inverse_10k}, a $K$-dependent pattern emerges. For $K=3$ inverse problems, FAMO and GradNorm without adapters outperform adapter combinations by 10\% to 25\%. On the $K=4$ inverse heat problem, adapter variants substantially improve stability compared with Vanilla, with \FAMOLCAM{} numerically best and \FAMOUAM{} close behind, while Vanilla fails catastrophically on one of three seeds.
Parameter-grouped surgery variants, such as PCGrad-Grouped and CAGrad-Grouped, resolve divergence but remain far worse than the direct baselines shown in Table~\ref{tab:inverse_10k}, which confirms that the failure is structural. Gradient surgery corrupts the low-dimensional physical-parameter learning signal.

\subsection{\texorpdfstring{Multi-Physics and $K$-Scaling}{Multi-Physics and K-Scaling}}
\label{sec:kscaling}
The natural scaling studies clarify the regime picture.
In the naturally coupled thermoelastic system with $K=4$, reported in Table~\ref{tab:natural_k4} in the appendix, the adapter-and-weighting combination outperforms either component alone. This thermoelastic result confirms the complementarity in a multi-physics setting where the loss decomposition is intrinsic to the underlying physics.
Appendix Figure~\ref{fig:qual_natural_k4_energy_exchange} shows that this improvement is physically structured. The adapter variant mainly improves the coupled displacement branch and the associated elastic-energy exchange rather than merely reducing a pixel-wise error metric.

\paragraph{Natural higher-$K$ validation.} Natural multi-physics systems with $K=5$ reactive transport and $K=6$ magnetohydrodynamics (MHD), as discussed in Appendix~\ref{app:kscaling_extended}, fall into the empirically magnitude-dominated regime. FAMO alone outperforms \FAMOUAM{}, consistent with the low-profile proxy $\widehat R$.
See Appendix~\ref{app:forward_extensions} for additional 3D and broader forward-validation results.

\subsection{Method Selection and Practical Guidelines}
\label{sec:method_selection}
The results above reveal a clear regime structure that motivates a practical, profiling-based deployment heuristic.
In practice, budget-constrained settings favor FAMO. Forward problems with $K=3$ often benefit from \FAMOUAM{}. Inverse problems with $K=3$ favor FAMO or GradNorm without adapters, while inverse problems with $K=4$ commonly benefit from \FAMOUAM{}. The natural higher-$K$ multi-physics systems in our suite are better handled by FAMO alone.
The resulting deployment rule is diagnostic-first.
Run the 1000-step Vanilla profile from Section~\ref{sec:profiling}; if the profiled Vanilla error is already below $10^{-3}$ or $\widehat f_{\mathrm{neg}}<0.05$, use FAMO alone.
If the conflict persists with $P>0.8$, use \FAMOUAM{}; if it decays with $P<0.5$ and slope $<-0.02$, use FAMO alone.
When conflict is present but ambiguous, we default to \FAMOUAM{}.
The ratio $\widehat R$ is reported as a compact direction-to-magnitude summary, while $P$ distinguishes persistent conflict from transient negative pairs.
Appendix Algorithm~\ref{alg:selection} gives the full pseudocode; the computational-cost analysis and the qualitative profiling-to-outcome bridge are provided in Appendix~\ref{app:main_text_diagnostics}.

\paragraph{Baseline fairness.} All baselines share the same trunk, the AdamW optimizer with a learning rate of $10^{-3}$, the same scheduler, and a nominal training budget of 10K epochs.
Most entries use 5 seeds. Any results with fewer seeds are explicitly marked in Tables~\ref{tab:main_results_a} and~\ref{tab:main_results_b}.
NashMTL is additionally swept at a learning rate of $3\times 10^{-4}$. ConFIG is swept over $\{10^{-4}, 5\times 10^{-4}, 10^{-3}, 2\times 10^{-3}, 5\times 10^{-3}\}$, as reported in Appendix~\ref{app:config}. GradNorm follows the reference algorithm, with gradient norms computed with respect to the last shared layer.
Implementation details appear in Appendix~\ref{app:details}. The source code is provided in the supplementary material.

\paragraph{Domain decomposition and inference.} Extended physics-informed neural networks (XPINNs) partition the input space~\citep{jagtap2020extended}, whereas per-loss adapters partition the gradient space. These are orthogonal interventions, as Tables~\ref{tab:xpinn_comparison} and~\ref{tab:xpinn_sensitivity} show.
Adapters increase the model parameter count by 15\% to 41\% in the configurations reported in Tables~\ref{tab:efficiency} and~\ref{tab:xpinn_comparison}.
See Appendix~\ref{app:main_text_diagnostics} for training-curve, NTK-weighting, and rank-ablation diagnostics.

\section{Conclusion}
\label{sec:conclusion}
We characterized gradient conflict in PINNs and showed that its structure is regime-dependent rather than uniformly requiring a single mitigation strategy. Across 60+ PDE configurations, directional conflict dominates the standard forward $K=3$ benchmarks and the natural $K=4$ thermoelastic system, making loss-indexed adapters combined with reweighting the most effective choice. In contrast, $K=3$ inverse problems and natural $K=5$ and $K=6$ multi-physics systems are largely magnitude-dominated, often favoring reweighting alone, and expose the unreliability of full-parameter-space gradient surgery on heterogeneous parameter spaces. A 1000-step profiling rule turns this observation into a practical method-selection heuristic, while the regime-transition theorem and the blockwise NTK view explain why adapter gains appear only in the appropriate conflict regime. Overall, our results support a diagnostic-first workflow for selecting PINN training interventions.

\bibliographystyle{plainnat}
\bibliography{references}

\appendix
\clearpage
\startcontents[appendix]
\pdfbookmark[1]{Appendix Table of Contents}{app-toc}
\section*{Appendix Table of Contents}
\printcontents[appendix]{l}{1}{\setcounter{tocdepth}{3}}
\clearpage

\section{List of Notation}
\label{app:notation}

We list the symbols used throughout the paper.

\begin{table}[t!]
\centering
\caption{List of notation}
\label{tab:notation}
\small
\begin{tabular}{ll}
\toprule
Symbol & Meaning \\
\midrule
$\Omega, \partial\Omega, x, d$ & Domain, boundary, point, dimension. \\
$\mathcal{F}, f, \mathcal{R}_{\mathrm{PDE}}, \mathcal{R}_m, \Gamma_m, u, u_\theta$ & PDE operator, source, residuals, constraint sets, and fields. \\
$\theta, \theta_{\text{trunk}}, \theta_k$ & All, trunk, and $k$th adapter parameters. \\
$\mathcal{L}, \mathcal{L}_k, \mathcal{L}_{\mathrm{PDE}}, \mathcal{L}_{\mathrm{BC}}, \mathcal{L}_{\mathrm{IC}}, K, \lambda_k$ & Losses, loss count, and weights. \\
$\nu, \kappa, \alpha$ & Physical coefficients in inverse and coupled problems. \\
$g_k, \cos(g_i,g_j), C_{ij}, \varepsilon_g$ & Gradient, cosine, conflict score, floor. \\
$g_k^{(\ell)}, C_k^{(\ell)}$ & Layer gradient and conflict. \\
$C_k, \bar r$ & Mean conflict, average norm ratio. \\
$h^{(\ell)}, H, L, d_h$ & Layer feature, readout, depth, width. \\
$h_{\text{trunk}}^{(\ell)}$ & Pre-adapter trunk feature. \\
$A_k^{(\ell)}(h), D_k^{(\ell)}, B_k^{(\ell)}, r$ & Adapter, projections, rank. \\
$\rho_k^{(\ell)}, \pi_k^{(\ell)}, \alpha_{\text{ad}}$ & Mixing score, adapter mixing weight, adapter strength. \\
$\mathcal{R}_{\text{ortho}}, \lambda_{\text{self}}, \lambda_{\text{cross}}$ & Orthogonality regularizer and weights. \\
$\Phi(x), \omega_f, F, \omega_{\max}$ & Fourier features, frequencies, bands, max frequency. \\
$D_K, M_{\mathrm{eff}}(K), \mathcal{U}_K, K^\star_{\mathrm E}$ & Conflict, imbalance, regime ratio, endpoint. \\
$T_{\mathrm{prof}}, \widehat f_{\mathrm{neg}}, \widehat D, \widehat M, \widehat R, P$ & Profile length, negative-pair fraction, summaries, persistence. \\
$J_\theta(x), r_i, e_j$ & Jacobian, pointwise PDE residual, pointwise BC error. \\
$\Theta_{ij}, \Theta_{ij}^{\text{shared}}, \Theta_{ij}^{(k)}$ & Block NTK, shared and adapter terms. \\
\bottomrule
\end{tabular}
\end{table}

\section{Related Work}
\label{sec:related}

\paragraph{Training difficulties and loss reweighting in PINNs.} \citet{raissi2019physics} introduced PINNs, and \citet{karniadakis2021physics} reviewed the broader physics-informed machine-learning landscape. Subsequent work identified gradient pathologies~\citep{wang2021understanding}, NTK-driven spectral bias~\citep{wang2022and, rahaman2019spectral}, soft-constraint trivial solutions~\citep{leiteritz2021trivial}, poor propagation~\citep{daw2023mitigating}, and broader failure modes~\citep{krishnapriyan2021characterizing}. Recent benchmarks and guides systematize these observations~\citep{hao2024pinnacle, wang2023expertguide}. Loss reweighting methods include learning rate annealing~\citep{wang2021understanding}, NTK-based weighting~\citep{wang2022and}, self-adaptive weights~\citep{mcclenny2023self, xiang2022selfadaptive}, GradNorm~\citep{chen2018gradnorm}, SoftAdapt~\citep{heydari2019softadapt}, FAMO~\citep{liu2024famo}, uncertainty weighting~\citep{kendall2018multi}, and causal training~\citep{wang2024respecting}. These methods adjust when or how much each loss contributes, but they do not directly resolve directional gradient conflict. Causal training, for example, enforces temporal causality by progressively activating loss terms from earlier to later time points.

Other PINN-side interventions modify the loss structure or update geometry itself. Gradient-enhanced PINNs add derivative-based residual terms~\citep{yu2022gpinn}, exact boundary imposition removes soft BC penalties via distance functions~\citep{sukumar2022exact}, and recent cone-based or gradient-alignment analyses study non-conflicting PINN updates directly~\citep{hwang2024dualcone, wang2025gradientalignment}. \citet{bischof2021multi} showed that multi-objective loss balancing can match gradient surgery. These methods are complementary to our per-loss adapters. They alter the loss schedule, loss scale, or update geometry, while our adapters adjust how each loss influences parameters through explicit loss-indexed subspaces.

\paragraph{Gradient surgery for MTL.} Representative MTL conflict-mitigation and multi-objective methods include the multiple-gradient descent algorithm~\citep{sener2018multi}, Pareto MTL~\citep{lin2019pareto}, PCGrad~\citep{yu2020gradient}, CAGrad~\citep{liu2021conflict}, NashMTL~\citep{navon2022multi}, Gradient Vaccine~\citep{wang2021gradvac}, RotoGrad~\citep{javaloy2022rotograd}, ConFIG~\citep{liu2025config}, and AlignedMTL~\citep{senushkin2023independent}.
These methods assume homogeneous shared parameters; however, this assumption makes them prone to failure on PINN problems with heterogeneous parameter types. In our setting, this heterogeneity combines network weights with learnable physical constants and produces a $10^4\times$ norm mismatch, as we systematically demonstrate in Section~\ref{sec:experiments}.
Indeed, ConFIG's unit-vector normalization discards gradient magnitude information critical for PDE loss landscapes where gradient magnitude ratios can exceed $10^3$ in our evaluation suites. AlignedMTL's singular value decomposition (SVD)-based alignment also incurs $O(p^2)$ cost per step.

\paragraph{Architecture-level approaches and adapters.} Domain decomposition methods such as conservative PINNs~\citep{jagtap2020conservative}, XPINNs~\citep{jagtap2020extended}, finite basis PINNs~\citep{moseley2023finite}, and augmented PINNs~\citep{he2023augmented} partition the input space into subdomains, each handled by a local PINN with interface constraints.
Variational formulations such as variational PINNs~\citep{kharazmi2019variational} and their hp-variants~\citep{kharazmi2021hpvpinns} instead modify the weak-form residual and can be combined with domain decomposition.
These methods address spatial complexity, weak-form structure, or discontinuities, but they do not resolve gradient conflict between different loss terms within each subdomain or formulation.
Our per-loss adapters are orthogonal to this line of work. They augment the shared trunk with loss-indexed low-rank parameter blocks and can be integrated with domain decomposition by placing such adapters inside each XPINN subdomain when a problem exhibits both spatial and loss-term complexity.

Beyond the PINN setting, broader MTL studies architecture-level soft sharing, where representations, routing weights, or parameter subsets are partially shared across tasks rather than tied through a single monolithic trunk.
Representative examples include Cross-Stitch Networks~\citep{misra2016crossstitch}, task-attention mechanisms~\citep{liu2019mtan}, and Multi-gate mixture-of-experts (MoE)~\citep{ma2018mmoe}, in addition to generic sparse MoE~\citep{shazeer2017outrageously} and task grouping~\citep{fifty2021efficiently, standley2020which}.
Task-specific adapters were popularized in transfer learning~\citep{houlsby2019parameter}, and Low-Rank Adaptation (LoRA) introduced rank-controlled adapters for parameter-efficient large language model fine-tuning~\citep{hu2022lora}. We use this low-rank parameterization not for task adaptation, but to instantiate loss-indexed perturbation subspaces inside a single shared-output PINN.
Our loss-indexed adapter mixing is derived from gradient conflict rather than input features, which makes the method complementary to both domain decomposition and MoE approaches.

\paragraph{NTK analysis and spectral diagnostics in deep learning.} The NTK~\citep{jacot2018neural} provides a spectral characterization of neural network training dynamics.
\citet{wang2022and} applied NTK analysis to PINNs, revealing how eigenvalue disparity between PDE and BC kernels causes spectral bias and contributes to training difficulty.
Our work complements this line by using a blockwise NTK perspective to interpret how per-loss architectural modifications can reduce destructive interference between losses.

\paragraph{Neural operators.} Neural operators such as the Deep Operator Network~\citep{lu2021deeponet} and the Fourier Neural Operator~\citep{li2021fourier} learn function-to-function solution maps, typically from paired solution data, although physics-informed and hybrid variants can also include residual losses.
This setting is complementary to ours. We study data-scarce PINN training in which a single physical field is optimized using PDE, BC, IC, and optional data losses, making inter-loss gradient geometry the central bottleneck.

\paragraph{Positioning our contribution.} Prior work provides strong individual tools for PINN loss balancing, gradient surgery, domain decomposition, adapters, and NTK analysis, but it does not, to our knowledge, characterize the conflict regimes that determine when scalar reweighting, architectural separation, or no additional intervention is appropriate.
Our novelty is therefore not the adapter primitive itself, but the diagnostic-first regime characterization linking conflict structure to method choice.
We support this characterization with regime-transition theory, operator-structure arguments, a blockwise NTK interpretation of per-loss parameter subspaces, and validation across 60+ PDE configurations, MTL benchmarks, NTK-weighting baselines~\citep{wang2022and}, and architecture-transfer checks spanning multilayer perceptron (MLP), Modified MLP, and Transformer backbones.

\section{Supplementary Method Details}
\label{app:method_details}

\subsection{Shared Feature Trunk}
\label{sec:trunk}

All methods share a common feature trunk that maps inputs $x \in \mathbb{R}^d$ to hidden features, consisting of Fourier feature encoding~\citep{tancik2020fourier} followed by residual blocks:
\begin{equation}
  \Phi(x) = [\sin(\omega_1^\top x), \cos(\omega_1^\top x), \ldots, \sin(\omega_F^\top x), \cos(\omega_F^\top x)],
\end{equation}
where $\omega_f \in \mathbb{R}^d$ are fixed frequency vectors with linearly spaced magnitudes
$\|\omega_f\|_2 = f\omega_{\max}/F$ for $f = 1, \ldots, F$,
with $F=10$ bands and $\omega_{\max} = 10\pi$.
The trunk consists of $L=4$ residual blocks with hidden dimension $d_h = 128$ and Tanh activations:
\begin{equation}
  h^{(\ell)} = h^{(\ell-1)} + \tanh(W_2^{(\ell)} \tanh(W_1^{(\ell)} h^{(\ell-1)} + b_1^{(\ell)}) + b_2^{(\ell)}).
\end{equation}
A final linear layer maps to the scalar output $u_\theta(x) \in \mathbb{R}$, or to $\mathbb{R}^{d_{\text{out}}}$ for systems.
This trunk has 140,033 parameters and serves as the Vanilla baseline.
In adapter variants, the same residual trunk is retained as the shared backbone; the adapters are additional loss-indexed low-rank modules inserted at each block.

\subsection{Conflict Metrics and Profiled Summaries}
\label{app:diagnostic_defs}

The aggregate conflict score for loss $k$ is $C_k = (1/(K-1))\sum_{j \neq k} C_{kj}$, measuring how much loss $k$ conflicts with the others on average.
For conflict-aware mixing and per-layer analyses, we analogously define
\begin{align*}
  C_k^{(\ell)} := \frac{1}{K-1}\sum_{j \neq k} C_{kj}^{(\ell)},
\end{align*}
where $C_{ij}^{(\ell)}$ is obtained by applying \Eqref{eq:conflict-score} to the layer-$\ell$ parameter gradients $g_k^{(\ell)} := \nabla_{\theta^{(\ell)}} \mathcal{L}_k(\theta)$.
For per-problem summaries, we also report the average pairwise magnitude ratio
\begin{align*}
  \bar r := \frac{1}{\binom{K}{2}} \sum_{i<j}
  \frac{\max(\|g_i\|, \|g_j\|) + \varepsilon_g}{\min(\|g_i\|, \|g_j\|) + \varepsilon_g},
\end{align*}
which is finite and at least $1$ by construction.

\paragraph{Profiled summaries.} For a profiling window $t = 1, \ldots, T_{\mathrm{prof}}$, let $g_k^{(t)} := \nabla_\theta \mathcal{L}_k(\theta_t)$ and define the pairwise cosine with a zero-gradient convention
\begin{align*}
  c_{ij}^{(t)} :=
  \begin{cases}
    \cos(g_i^{(t)}, g_j^{(t)}), & \|g_i^{(t)}\| \|g_j^{(t)}\| > 0,\\
    0, & \text{otherwise,}
  \end{cases}
\end{align*}
and define
\begin{align*}
  N_t^- &:= \{(i,j) : 1 \le i < j \le K, c_{ij}^{(t)} < 0\},
  &\quad
  f_{\mathrm{neg}}^{(t)}
  &:= \frac{|N_t^-|}{\binom{K}{2}}, \\
  \bar c_-^{(t)} &:=
  \begin{cases}
    \frac{1}{|N_t^-|}\sum_{(i,j)\in N_t^-} c_{ij}^{(t)}, & |N_t^-| > 0,\\
    0, & |N_t^-| = 0,
  \end{cases}
  &\quad
  D^{(t)}
  &:= f_{\mathrm{neg}}^{(t)} |\bar c_-^{(t)}|, \\
  M^{(t)}
  &:= \frac{\operatorname{std}_k(\|g_k^{(t)}\|)}
          {K^{-1}\sum_k \|g_k^{(t)}\| + \varepsilon_g}.
\end{align*}
We report the profiled summaries
\begin{align*}
  \widehat f_{\mathrm{neg}} &:= \frac{1}{T_{\mathrm{prof}}}\sum_{t=1}^{T_{\mathrm{prof}}} f_{\mathrm{neg}}^{(t)}, \quad
  \widehat D := \frac{1}{T_{\mathrm{prof}}}\sum_{t=1}^{T_{\mathrm{prof}}} D^{(t)}, \quad
  \widehat M := \frac{1}{T_{\mathrm{prof}}}\sum_{t=1}^{T_{\mathrm{prof}}} M^{(t)}.
\end{align*}
We also report the empirical proxy
\begin{align*}
  \widehat R := \frac{\widehat D}{\widehat M + \varepsilon_R},
\end{align*}
where $\varepsilon_R = 10^{-8}$.
Here $\widehat M$ is a coefficient-of-variation surrogate for magnitude imbalance, whereas Theorem~\ref{thm:regime_transition} uses the max-to-mean ratio $M_{\mathrm{eff}}(K)$.
The profiled ratio $\widehat R$ should therefore be read as an empirical surrogate for the theorem-level quantity $\mathcal U_K$, not as a plug-in estimator of it.
These profiled summaries are the quantities reported in the experiments.

\subsection{Additional Adapter Details}

\paragraph{Adapter orthogonality.} To encourage the independence of the adapter row subspaces in $\mathbb{R}^{d_h}$, we regularize both self-orthogonality within each down-projection and cross-orthogonality between different down-projections:
\begin{equation}
\label{eq:ortho}
  \mathcal{R}_{\text{ortho}} = \lambda_{\text{self}} \sum_{k,\ell} \|D_k^{(\ell)} D_k^{(\ell)\top} - I_r\|_F^2
  + \lambda_{\text{cross}} \sum_{\substack{i < j \\\ell}} \|D_i^{(\ell)} D_j^{(\ell)\top}\|_F^2.
\end{equation}
The first term encourages the rows of each $D_k^{(\ell)}$ to be orthonormal, while the second penalizes overlap between the row spaces of different adapters. When the row spaces are exactly orthogonal and each $D_k^{(\ell)}$ has full row rank $r$, the cross term vanishes.

\subsection{Conflict-Aware Adapter Mixing}
\label{sec:mixing}

The latent mixing scores $\rho_k^{(\ell)}$ can be fixed or dynamically adjusted based on gradient conflict.
\UAM{} uses the uniform rule $\pi_k^{(\ell)} = 1/K$.
For global \CAM{}, we use the aggregate conflict summary $C_k$ and set a layer-independent score $\rho_k = \sigma(a_0 - b \cdot C_k)$, with $\rho_k^{(\ell)} \equiv \rho_k$ in the normalization above.
For \LCAM{}, we instead use layerwise conflict summaries and set $\rho_k^{(\ell)} = \sigma(a_0 + \delta_\ell - b \cdot C_k^{(\ell)})$ before converting these scores into the normalized adapter weights $\pi_k^{(\ell)}$.
In both cases, larger conflict decreases $\rho$ and therefore increases the corresponding adapter weight.
Here $a_0 = 1.0$, $b = 2.0$, and the \LCAM{} offset $\delta_\ell$ interpolates from $+0.3$ at shallow layers, which biases toward sharing, to $-0.3$ at deep layers, which biases toward specialization. The conflict scores are exponential-moving-average (EMA) smoothed with $\beta=0.95$.
In practice, fixed \UAM{} captures nearly all of the benefit. On rerun $K=3$ stress tests in Helmholtz and Klein--Gordon, reported in Table~\ref{tab:mixing_ablation}, dynamic \LCAM{} stays within single-digit percentage differences and fixed uniform mixing is slightly better in both cases.
We therefore recommend \UAM{} with uniform $\pi$ as the default, while \CAM{} remains a higher-complexity variant for settings in which fixed mixing is insufficient.

\subsection{Training Algorithm}
\label{sec:training}

Training follows Algorithm~\ref{alg:training}.
A warm-up of $T_{\text{warm}}=200$ epochs with equal weights and uniform adapter mixing $\pi_k = 1/K$ establishes a stable baseline.
The optional mixing step is used only by dynamic \CAM{} and \LCAM{} variants and is skipped for fixed \UAM{}.
After warm-up, every $\tau$ epochs, we compute the current aggregate conflict scores $C_k$ for \CAM{} or layerwise scores $C_k^{(\ell)}$ for \LCAM{}, update their EMA, set the corresponding latent scores $\rho$ using Section~\ref{sec:mixing} with clipping to $[\rho_{\min}, \rho_{\max}]$, and recompute $\pi$ through the normalization in Section~\ref{sec:adapters}.
The latent scores are bookkeeping variables for adapter mixing rather than AdamW-optimized model parameters.

\begin{algorithm}[t!]
\caption{Per-loss adapter training with \UAM{} as the default and \CAM{} optional}
\label{alg:training}
\begin{algorithmic}[1]
\REQUIRE PDE problem with $K$ losses, single-output network $u_\theta$ with $K$ adapters
\FOR{epoch $t = 1, \ldots, T$}
  \STATE Sample collocation, BC, and IC points
  \STATE Run the forward pass through the shared trunk and aggregated adapter correction $\to u_\theta(x)$ as in \Eqref{eq:adapted}
  \STATE Compute losses $\{\mathcal{L}_k\}_{k=1}^K$ from the same $u_\theta$
  \STATE For \CAM{} and \LCAM{} only, after warm-up and every $\tau$ epochs, update EMA conflict scores, refresh $\rho$, and recompute $\pi$
  \STATE Total loss: $\mathcal{L} = \sum_k \lambda_k \mathcal{L}_k + \mathcal{R}_{\text{ortho}}$
  \STATE Update $\theta$ via AdamW with gradient clipping
\ENDFOR
\end{algorithmic}
\end{algorithm}

\section{Supplementary Theory Details}
\label{app:theory_details}

We collect the supplementary theory details referenced in the main text.
First, we give the supporting directional-to-magnitude results for the regime-transition analysis.
We then give a blockwise NTK interpretation of per-loss adapters as structured perturbations of the shared representation.
Finally, we justify the operator-structure arguments summarized in Remark~\ref{thm:pde_conflict}.

\subsection{Supporting Directional-to-Magnitude Results}
\label{sec:directional_magnitude_support}

\begin{proposition}[Conditional concentration of $D_K$]
\label{prop:dk_concentration}
Let $g_1, \ldots, g_K$ be nonzero gradients in $\mathbb{R}^p$, write
$\hat g_k = g_k / \|g_k\|$, and define
\begin{align*}
  X_{ij} := \max(0, -\hat g_i^\top \hat g_j),
  \qquad
  D_K := \frac{1}{\binom{K}{2}} \sum_{i<j} X_{ij}.
\end{align*}
Suppose there exists a $\sigma$-algebra $\Lambda$ such that,
conditional on $\Lambda$, the family $\{X_{ij}\}_{i<j}$ is
exchangeable with
\begin{align*}
  \E[X_{12} \mid \Lambda] = \mu(\Lambda),
  \qquad
  \sup_{(i,j)\neq(k,\ell)}
  \left|\mathrm{Cov}(X_{ij}, X_{k\ell} \mid \Lambda)\right|
  \le c(\Lambda),
\end{align*}
where the supremum is over distinct unordered pairs and
$c(\Lambda) \ge 0$ is $\Lambda$-measurable. Then with
$n_K := \binom{K}{2}$, we have
\begin{align*}
  \E[D_K \mid \Lambda] &= \mu(\Lambda),\\
  \mathrm{Var}(D_K \mid \Lambda) &\le \frac{1}{4 n_K} + c(\Lambda),
\end{align*}
and therefore
\begin{equation}
\label{eq:dk-totallaw-new}
  \mathrm{Var}(D_K)
  = \mathrm{Var}(\mu(\Lambda))
  + \E[\mathrm{Var}(D_K \mid \Lambda)]
  \le \mathrm{Var}(\mu(\Lambda))
  + \frac{1}{4 n_K}
  + \E[c(\Lambda)].
\end{equation}

If, in addition, the directions $\hat g_k$ are independent and uniformly
distributed on
$V \cap \mathbb{S}^{p-1}$ for some fixed subspace
$V \subseteq \mathbb{R}^p$ with $\dim(V) = d_{\mathrm{eff}}$, then in
orthonormal coordinates on $V$ they are independent and uniformly
distributed on $\mathbb{S}^{d_{\mathrm{eff}}-1}$. For
\begin{align*}
  \mu_d
  := \frac{\Gamma(d_{\mathrm{eff}}/2)}
          {2\sqrt{\pi}\Gamma((d_{\mathrm{eff}}+1)/2)}
  = \Theta(d_{\mathrm{eff}}^{-1/2}),
\end{align*}
we have
\begin{align*}
  \E[D_K] &= \mu_d,\\
  \mathrm{Var}(D_K)
  &= \frac{\mathrm{Var}(X_{12})}{n_K}
   = \frac{1}{n_K}
      \left(\frac{1}{2 d_{\mathrm{eff}}} - \mu_d^2\right)
   = O((K^2 d_{\mathrm{eff}})^{-1}).
\end{align*}
\end{proposition}

\begin{remark}
\Eqref{eq:dk-totallaw-new} is the quantity that is stable
under latent regime variation. Exchangeability alone does not force
$D_K$ to concentrate without the additional term
$\mathrm{Var}(\mu(\Lambda))$ vanishing. In particular, a short
profiling window should be interpreted as estimating a marginal
average of the directional-conflict level over that window, not a
pointwise constant that is guaranteed to concentrate across all
optimization states.
\end{remark}

\begin{remark}
The next theorem is not a consequence of
Proposition~\ref{prop:dk_concentration} under arbitrary latent regime
variation. It is a separate stylized model in which the directional
regime is fixed, so $\Lambda$ is trivial, and gradient directions
decouple from gradient magnitudes. This fixed-directional-regime model is the setting in which
explicit rate calculations are available.
\end{remark}

\begin{corollary}[Expected beneficial range]
\label{cor:finite_kstar}
For $\tau > 0$, whenever the set is nonempty, define
\begin{align*}
  K^\star_{\mathrm E}(\tau)
  := \sup\{K \ge 2 : \E[\mathcal U_K] > \tau\}.
\end{align*}
Under the hypotheses of Theorem~\ref{thm:regime_transition}, this set
is nonempty for all sufficiently small $\tau$. As $\tau \to 0^+$,
\begin{align*}
  K^\star_{\mathrm E}(\tau)
  =
  \begin{cases}
    \exp(\Theta(d_{\mathrm{eff}}^{-1/2} / \tau)),
      & \text{under the exponential-tail model},\\
    \Theta((\tau d_{\mathrm{eff}}^{1/2})^{-\alpha}),
      & \text{under the Pareto-tail model.}
  \end{cases}
\end{align*}
The constants hidden in $\Theta(\cdot)$ depend on the tail parameters
and on $\mu_m$, but not on $\tau$ or $d_{\mathrm{eff}}$.
\end{corollary}

\subsection{Blockwise NTK View for Per-Loss Adapters}
\label{sec:blockwise_ntk}

This subsection develops the NTK perspective on per-loss adapters, building on the foundational NTK framework~\citep{jacot2018neural} and its PINN-specific analysis by \citet{wang2022and}.

For fixed mixing, write the parameter vector as $\theta = [\theta_{\text{trunk}}, \theta_1, \ldots, \theta_K]$, where $\theta_k$ denotes the parameters of the $k$th adapter block.
For each loss $k$, let $s_k(x)$ denote a scalar quantity entering $\mathcal{L}_k$ and derived from the same shared output $u_\theta$, such as the raw output, a boundary mismatch, or a PDE residual component.
We define the loss-indexed block NTK entry, measuring parameter-space coupling between scalar channels $i$ and $j$ at inputs $x$ and $x'$, by
\begin{equation}
\label{eq:ntk-decomp}
  \Theta_{ij}(x, x') := \nabla_\theta s_i(x)^\top \nabla_\theta s_j(x')
\end{equation}
for $1 \le i,j \le K$.
When the latent mixing scores $\rho_k^{(\ell)}$ are treated as fixed with respect to $\theta$ during differentiation, and equivalently when the normalized weights $\pi_k^{(\ell)}$ are fixed, the block kernel splits additively by parameter block:
\begin{equation}
  \Theta_{ij}(x, x') = \Theta_{ij}^{\text{shared}}(x, x') + \sum_{k=1}^K \Theta_{ij}^{(k)}(x, x'),
\end{equation}
where
\begin{align*}
  \Theta_{ij}^{\text{shared}} &= \nabla_{\theta_{\text{trunk}}} s_i(x)^\top \nabla_{\theta_{\text{trunk}}} s_j(x'), \\
  \Theta_{ij}^{(k)} &= \nabla_{\theta_k} s_i(x)^\top \nabla_{\theta_k} s_j(x').
\end{align*}
In the exact single-output architecture, the adapter terms $\Theta_{ij}^{(k)}$ need not vanish for $i \neq j$ because every $s_i$ depends on the same shared output $u_\theta$, which itself depends on all adapter blocks through \Eqref{eq:adapted}.
In the idealized exact-separation surrogate where $s_i$ depends only on $\theta_{\text{trunk}}$ and $\theta_i$, the adapter contribution instead reduces to a diagonal $\delta_{ij}$ term.
Here, the term ``exact'' refers to the implemented shared-output adapter architecture with fixed mixing, for which the parameter-block decomposition above is exact but no diagonalization is assumed.
The exact-separation surrogate is not the implemented model; it is an idealized limiting case in which cross-adapter sensitivities are neglected.
Thus, the diagonal $\delta_{ij}$ form should be read as a mechanism-level approximation, while the actual architecture may retain off-diagonal adapter couplings that are encouraged, but not forced, to be small by subspace separation and orthogonality regularization.
We use this blockwise picture as bookkeeping for the kernel-level interpretation of per-loss adapters. It identifies the adapter contribution as an additive perturbation of the shared kernel and clarifies why loss-indexed parameter blocks can relieve directional interference without introducing separate physical outputs.

\subsection{Capacity Heuristics and Per-Layer Mixing}

For architectural capacity, the saturation ratio $\Delta_{\mathrm{sat}} := Kr/d_h$ is a useful heuristic proxy. At fixed width, increasing either the number of losses or the adapter rank leaves less room for loss-specific low-rank specialization, and exact cross-adapter orthogonality becomes harder to satisfy.
Per-layer mixing can in principle exploit layer-wise variation in conflict, but in our experiments, its gains are modest at $K=3$ and become more noticeable only as $K$ grows.
We therefore treat capacity, conflict-aware mixing, and orthogonality effects primarily through the rank, mixing, and regularization ablations rather than through separate standalone theorems.

\subsection{PDE Structure and Gradient Conflict}
\label{sec:pde_structure}

Beyond the $K$-dependent regime structure, the type of PDE determines the conflict profile.
The empirical results across the evaluation suites suggest a structural pattern:

\begin{remark}[Operator-type heuristic for gradient coupling]
\label{thm:pde_conflict}
For the PDE and Dirichlet-boundary setting made explicit in Appendix~\ref{app:proof_pde_conflict}, the operator class constrains the state-level influence pattern and therefore provides only a qualitative prior for parameter-gradient coupling. Uniform ellipticity is compatible with domain-wide state-level influence via Green and Poisson representations, so widespread PDE and BC gradient coupling is possible; however, its sign and magnitude still depend on the residual and error coefficients and on the network-Jacobian correlations. Scalar linear parabolic operators satisfying a comparison principle or a maximum principle are order-preserving in the state variable, so the operator alone does not determine the sign of PDE and BC parameter-gradient alignment. Negative alignment can already arise from residual and error signs together with network Jacobian correlations, and can be further amplified by nonlinearities, discretization, or non-monotone parameterizations. Hyperbolic operators localize state-level influence to a domain of dependence through finite propagation speed, but finite propagation alone does not determine the sign of the parameter-gradient cosine. Stiff or highly oscillatory terms can nevertheless change the effective length scales seen by the network.
Hence operator type provides a heuristic prior on conflict structure, but not a universal sign or numerical value of $f_{\mathrm{neg}}$.
\end{remark}

This classification should therefore be read as a heuristic prior rather than a standalone decision rule. Elliptic problems are more compatible with widespread coupling, while parabolic problems require an additional nonlinearity or parameterization check.
However, the following remark shows that operator type alone is still insufficient in high dimensions.

\begin{remark}[Residual decomposition alone does not imply a high-dimensional elliptic monotonicity law]
\label{cor:highdim_elliptic}
For the $d$-dimensional Poisson equation $-\Delta u = f$ on $[0,1]^d$ with manufactured solution $u(x) = \prod_{i=1}^d \sin(\pi x_i)$, the residual decomposition
\begin{equation}
  r_\theta(x) = -\sum_{i=1}^d \partial_{x_i}^2 u_\theta(x) - f(x),
\end{equation}
with forcing amplitude $f(x) = d\pi^2 u(x)$, together with the identity
\begin{equation}
  \nabla_\theta r_\theta(x) = -\sum_{i=1}^d \nabla_\theta \partial_{x_i}^2 u_\theta(x),
\end{equation}
does not by itself imply any universal monotone-in-$d$ law for the instantaneous cosine between $\nabla_\theta \mathcal{L}_{\mathrm{PDE}}$ and $\nabla_\theta \mathcal{L}_{\mathrm{BC}}$ whenever those aggregated gradients are nonzero.
In particular, if the per-coordinate sensitivities $\nabla_\theta \partial_{x_i}^2 u_\theta(x)$ are weakly correlated, then the normalized sum $d^{-1}\sum_{i=1}^d \nabla_\theta \partial_{x_i}^2 u_\theta(x)$ can become more concentrated with $d$, while the residual formula simultaneously contains an explicit forcing contribution of size $O(d)$.
Because the network Laplacian term can either reinforce or cancel that forcing contribution, the residual amplitude need not scale monotonically with $d$.
Therefore, the displayed decomposition alone does not justify a universal monotone-in-$d$ law for the instantaneous PDE and BC cosine.
Any extension of this warning to profiled statistics such as $f_{\mathrm{neg}}(d)$ requires additional assumptions on the boundary-gradient scaling and on the training trajectory used in the profiling procedure.
\end{remark}

Appendix~\ref{app:proof_pde_conflict} justifies both remarks. Accordingly, operator-type heuristics are insufficient. The conflict regime depends on the interplay between the differential operator, forcing term, and spatial dimension, motivating the lightweight profiling criterion in Section~\ref{sec:method_selection} as a general-purpose tool that subsumes operator-type heuristics.

\section{Proofs for the Regime-Transition Results}
\label{app:proof_regime}

\begin{proof}[Proof of Proposition~\ref{prop:dk_concentration}]
By conditional exchangeability,
\begin{align*}
  \E[D_K \mid \Lambda]
  = \frac{1}{n_K}\sum_{i<j}\E[X_{ij}\mid\Lambda]
  = \mu(\Lambda).
\end{align*}
For the variance,
\begin{align*}
  \mathrm{Var}(D_K \mid \Lambda)
  = \frac{1}{n_K^2}
    \sum_{i<j}\sum_{k<\ell}
    \mathrm{Cov}(X_{ij}, X_{k\ell} \mid \Lambda).
\end{align*}
The diagonal contribution satisfies
\begin{align*}
  \frac{1}{n_K^2}\sum_{i<j}\mathrm{Var}(X_{ij}\mid\Lambda)
  \le \frac{1}{n_K^2}\cdot n_K \cdot \frac14
  = \frac{1}{4 n_K},
\end{align*}
because $0 \le X_{ij} \le 1$. The off-diagonal contribution satisfies
\begin{align*}
  \frac{1}{n_K^2}
  \sum_{(i,j)\neq(k,\ell)}
  \left|\mathrm{Cov}(X_{ij}, X_{k\ell} \mid \Lambda)\right|
  \le \frac{n_K(n_K-1)}{n_K^2} c(\Lambda)
  \le c(\Lambda).
\end{align*}
The conditional covariance bound proves the conditional variance bound. \Eqref{eq:dk-totallaw-new} is then the law of total variance.

For the spherical claim, let $V$ be a fixed subspace of
$\mathbb{R}^p$ with $\dim(V) = d := d_{\mathrm{eff}}$, choose orthonormal
coordinates on $V$, and work in that coordinate system. Then the
directions become independent and uniformly distributed on
$\mathbb{S}^{d-1}$. If
$u, v \in \mathbb{S}^{d-1}$ are independent and uniform, and
$c := u^\top v$, then the density of $c$ on $[-1,1]$ is symmetric and
proportional to $(1-z^2)^{(d-3)/2}$, so
\begin{align*}
  \E[X_{12}]
  = \E[(-c)_+]
  = \frac12 \E|c|
  = \frac{\Gamma(d/2)}{2\sqrt{\pi}\Gamma((d+1)/2)}
  = \mu_d.
\end{align*}
Also,
\begin{align*}
  \E[X_{12}^2]
  = \E[c^2 \mathbf{1}_{\{c<0\}}]
  = \frac12 \E[c^2]
  = \frac{1}{2d},
\end{align*}
because $\E[c^2] = 1/d$.

It remains to show that distinct $X_{ij}$ are pairwise independent.
If the pairs are disjoint, pairwise independence is immediate from independence of the
underlying directions. If they share one index, say $(i,j)=(1,2)$ and
$(k,\ell)=(1,3)$, then conditional on $u_1$ the random variables
$X_{12}$ and $X_{13}$ are independent because $u_2$ and $u_3$ are
independent. Moreover, by rotational invariance the conditional law of
$u_1^\top u_2$ given $u_1$ does not depend on $u_1$, so for any bounded
measurable $\varphi$ and $\psi$,
\begin{align*}
  \E[\varphi(X_{12})\psi(X_{13}) \mid u_1]
  = \E[\varphi(X_{12}) \mid u_1]\E[\psi(X_{13}) \mid u_1]
  = \E[\varphi(X_{12})]\E[\psi(X_{13})].
\end{align*}
Taking expectations gives pairwise independence. Therefore
\begin{align*}
  \mathrm{Var}(D_K)
  = \frac{1}{n_K^2}\sum_{i<j}\mathrm{Var}(X_{ij})
  = \frac{\mathrm{Var}(X_{12})}{n_K}
  = \frac{1}{n_K}
    \left(\frac{1}{2d_{\mathrm{eff}}} - \mu_d^2\right),
\end{align*}
which is $O((K^2 d_{\mathrm{eff}})^{-1})$.
\end{proof}

\begin{proof}[Proof of Theorem~\ref{thm:regime_transition}]
Since $D_K$ depends only on the directions $u_1,\dots,u_K$ and
$\bar m / A_K$ depends only on the magnitudes $m_1,\dots,m_K$, the two
terms are independent. Proposition~\ref{prop:dk_concentration} gives
$\E[D_K] = \mu_d$, hence
\begin{align*}
  \E[\mathcal U_K]
  = \E\left[D_K \frac{\bar m}{A_K}\right]
  = \E[D_K] \E\left[\frac{\bar m}{A_K}\right]
  = \mu_d \E\left[\frac{\bar m}{A_K}\right].
\end{align*}
The factorization argument proves \Eqref{eq:uk-factorization}.

\paragraph{Exponential-tail magnitudes.} Fix constants $c_-, c_+ > 0$ such that
\begin{align*}
  0 < c_- < \frac{\nu}{b},
  \qquad
  c_+ > \frac{\nu}{a}.
\end{align*}
Set
\begin{align*}
  q_K^- := c_- \log K,
  \qquad
  q_K^+ := c_+ \log K.
\end{align*}
For sufficiently large $K$, both $q_K^-$ and $q_K^+$ exceed $t_0$, so
the tail bounds apply. By the lower tail bound,
\begin{align*}
  \Pr(A_K < q_K^-)
  &= \Pr(m_1 \le q_K^-)^K\\
  &\le \left(1 - \exp\left(-\frac{b q_K^-}{\nu}\right)\right)^K\\
  &\le \exp\left(-K^{1 - b c_-/\nu}\right)
  = o((\log K)^{-1}).
\end{align*}
Therefore
\begin{align*}
  \E\left[\frac{\bar m}{A_K}\right]
  &\le \E\left[\frac{\bar m}{q_K^-}\mathbf{1}_{\{A_K \ge q_K^-\}}\right]
   + \Pr(A_K < q_K^-)\\
  &\le \frac{\E[\bar m]}{q_K^-}
   + o((\log K)^{-1})\\
  &= O((\log K)^{-1}).
\end{align*}

For the lower bound, the upper tail bound gives
\begin{align*}
  \Pr(A_K > q_K^+)
  &\le K \Pr(m_1 > q_K^+)\\
  &\le K^{1 - a c_+/\nu}
  = o(1).
\end{align*}
Also $\bar m \to \mu_m$ in probability by the law of large numbers, so
\begin{align*}
  \Pr\left(\bar m \ge \frac{\mu_m}{2}\right) \to 1.
\end{align*}
Hence, by the union bound,
\begin{align*}
  \Pr\left(\bar m \ge \frac{\mu_m}{2}, A_K \le q_K^+\right)
  \ge 1 - \Pr\left(\bar m < \frac{\mu_m}{2}\right) - \Pr(A_K > q_K^+)
  = 1 - o(1).
\end{align*}
Hence
\begin{align*}
  \E\left[\frac{\bar m}{A_K}\right]
  &\ge \frac{\mu_m}{2 q_K^+}
     \Pr\left(\bar m \ge \frac{\mu_m}{2}, A_K \le q_K^+\right)\\
  &= \Omega((\log K)^{-1}).
\end{align*}
Combining the two bounds with $\mu_d = \Theta(d_{\mathrm{eff}}^{-1/2})$
proves
\begin{align*}
  \E[\mathcal U_K]
  = \Theta(d_{\mathrm{eff}}^{-1/2} / \log K).
\end{align*}

\paragraph{Pareto-tail magnitudes.} Let $A_K = \max_{1\le k\le K} m_k$ and fix any $q \in (1,\alpha)$.
Write $r := q/(q-1)$. By H\"older's inequality,
\begin{align*}
  \E\left[\frac{\bar m}{A_K}\right]
  \le \E[\bar m^q]^{1/q}
      \E[A_K^{-r}]^{1/r}.
\end{align*}
Since $q < \alpha$, we have $\E[m_1^q] < \infty$. By the convexity of
$t \mapsto t^q$ for $q \ge 1$,
\begin{align*}
  \E[\bar m^q]
  \le \frac{1}{K}\sum_{k=1}^K \E[m_k^q]
  = \E[m_1^q]
  < \infty.
\end{align*}
It remains to evaluate $\E[A_K^{-r}]$. For Pareto tails,
\begin{align*}
  \Pr(A_K \le t)
  = \left(1 - \left(\frac{t_0}{t}\right)^\alpha\right)^K,
  \qquad
  t \ge t_0,
\end{align*}
so $A_K$ has density
\begin{align*}
  f_{A_K}(t)
  = K \alpha t_0^\alpha t^{-\alpha-1}
    \left(1 - \left(\frac{t_0}{t}\right)^\alpha\right)^{K-1},
  \qquad
  t \ge t_0.
\end{align*}
Therefore
\begin{align*}
  \E[A_K^{-r}]
  &= K \alpha t_0^\alpha \int_{t_0}^{\infty}
      t^{-r-\alpha-1}
      \left(1 - \left(\frac{t_0}{t}\right)^\alpha\right)^{K-1}
      dt.
\end{align*}
With the substitution $y = (t_0/t)^\alpha$, the integral becomes
\begin{align*}
  \E[A_K^{-r}]
  &= K t_0^{-r}
     \int_0^1 y^{r/\alpha} (1-y)^{K-1} dy\\
  &= K t_0^{-r} B\left(1 + \frac{r}{\alpha}, K\right)\\
  &= t_0^{-r}\Gamma\left(1 + \frac{r}{\alpha}\right)
     \frac{\Gamma(K+1)}{\Gamma(K+1+r/\alpha)}
   = \Theta(K^{-r/\alpha}).
\end{align*}
Hence
\begin{align*}
  \E\left[\frac{\bar m}{A_K}\right]
  = O(K^{-1/\alpha}).
\end{align*}

For the lower bound, fix any constant $C > t_0$. Then
\begin{align*}
  \Pr(A_K \le C K^{1/\alpha})
  &= \left(1 - \frac{t_0^\alpha}{C^\alpha K}\right)^K
   \longrightarrow \exp\left(-\frac{t_0^\alpha}{C^\alpha}\right)
   > 0.
\end{align*}
Also $\Pr(\bar m \ge \mu_m/2) \to 1$ by the law of large numbers. Hence
\begin{align*}
  \Pr\left(\bar m \ge \frac{\mu_m}{2},
             A_K \le C K^{1/\alpha}\right)
  &\ge \Pr(A_K \le C K^{1/\alpha})
     - \Pr\left(\bar m < \frac{\mu_m}{2}\right)\\
  &\longrightarrow \exp\left(-\frac{t_0^\alpha}{C^\alpha}\right)
   > 0.
\end{align*}
Thus
\begin{align*}
  \E\left[\frac{\bar m}{A_K}\right]
  &\ge \frac{\mu_m}{2 C K^{1/\alpha}}
     \Pr\left(\bar m \ge \frac{\mu_m}{2},
                 A_K \le C K^{1/\alpha}\right)\\
  &= \Omega(K^{-1/\alpha}).
\end{align*}
Combining upper and lower bounds with
$\mu_d = \Theta(d_{\mathrm{eff}}^{-1/2})$ yields
\begin{align*}
  \E[\mathcal U_K]
  = \Theta(d_{\mathrm{eff}}^{-1/2} K^{-1/\alpha}).
\end{align*}
\end{proof}

\begin{proof}[Proof of Corollary~\ref{cor:finite_kstar}]
Under the exponential-tail model,
\begin{align*}
  \E[\mathcal U_K]
  = \Theta(d_{\mathrm{eff}}^{-1/2} / \log K)
\end{align*}
as $K \to \infty$.
Hence, there exist constants $c_1, c_2 > 0$ and $K_0 \ge 2$ such that
for all $K \ge K_0$,
\begin{align*}
  c_1 \frac{d_{\mathrm{eff}}^{-1/2}}{\log K}
  \le \E[\mathcal U_K]
  \le c_2 \frac{d_{\mathrm{eff}}^{-1/2}}{\log K}.
\end{align*}
For sufficiently small $\tau$, define
\begin{align*}
  K_-(\tau)
  &:=
  \left\lfloor
    \exp\left(\frac{c_1 d_{\mathrm{eff}}^{-1/2}}{2\tau}\right)
  \right\rfloor,\\
  K_+(\tau)
  &:=
  \left\lceil
    \exp\left(\frac{2 c_2 d_{\mathrm{eff}}^{-1/2}}{\tau}\right)
  \right\rceil.
\end{align*}
When $\tau$ is small enough, $K_-(\tau) \ge K_0$, and the lower bound
gives $\E[\mathcal U_{K_-(\tau)}] \ge 2\tau > \tau$, so the defining
set for $K^\star_{\mathrm E}(\tau)$ is nonempty. Likewise, for every
$K \ge K_+(\tau)$, the upper bound gives $\E[\mathcal U_K] \le \tau/2$,
so no such $K$ belongs to the set. Therefore
\begin{align*}
  K_-(\tau) \le K^\star_{\mathrm E}(\tau) < K_+(\tau),
\end{align*}
which implies
\begin{align*}
  K^\star_{\mathrm E}(\tau)
  = \exp(\Theta(d_{\mathrm{eff}}^{-1/2} / \tau))
\end{align*}
as $\tau \to 0^+$.

Under the Pareto-tail model,
\begin{align*}
  \E[\mathcal U_K]
  = \Theta(d_{\mathrm{eff}}^{-1/2} K^{-1/\alpha})
\end{align*}
as $K \to \infty$.
Therefore, there exist constants $c_1', c_2' > 0$ and $K_0' \ge 2$ such that
for all $K \ge K_0'$,
\begin{align*}
  c_1' d_{\mathrm{eff}}^{-1/2} K^{-1/\alpha}
  \le \E[\mathcal U_K]
  \le c_2' d_{\mathrm{eff}}^{-1/2} K^{-1/\alpha}.
\end{align*}
For sufficiently small $\tau$, define
\begin{align*}
  \widetilde K_-(\tau)
  &:=
  \left\lfloor
    \left(\frac{c_1' d_{\mathrm{eff}}^{-1/2}}{2\tau}\right)^\alpha
  \right\rfloor,\\
  \widetilde K_+(\tau)
  &:=
  \left\lceil
    \left(\frac{2 c_2' d_{\mathrm{eff}}^{-1/2}}{\tau}\right)^\alpha
  \right\rceil.
\end{align*}
When $\tau$ is small enough, $\widetilde K_-(\tau) \ge K_0'$, and the
lower bound gives $\E[\mathcal U_{\widetilde K_-(\tau)}] \ge 2\tau$.
For every $K \ge \widetilde K_+(\tau)$, the upper bound gives
$\E[\mathcal U_K] \le \tau/2$. Therefore
\begin{align*}
  \widetilde K_-(\tau)
  \le K^\star_{\mathrm E}(\tau)
  < \widetilde K_+(\tau),
\end{align*}
and hence
\begin{align*}
  K^\star_{\mathrm E}(\tau)
  = \Theta((\tau d_{\mathrm{eff}}^{1/2})^{-\alpha})
\end{align*}
as $\tau \to 0^+$.
The preceding endpoint bounds prove the two asymptotic forms.
\end{proof}

\paragraph{Empirical interpretation.} Theorem~\ref{thm:regime_transition} should be read as a two-model rate
statement in a fixed stylized model, and
Corollary~\ref{cor:finite_kstar} is an expectation-level consequence of
that model, not a deterministic threshold for every realization. The
theorem is not an exhaustive taxonomy of all light-tailed distributions. The
relevant empirical issue is whether the profiled per-loss gradient
magnitudes behave closer to a logarithmic-growth maximum regime or to a
power-growth maximum regime over the observed $K$ range. If neither
model fits the data well, the qualitative monotonicity conclusion may
still be useful, but the explicit rates should not be over-interpreted.

\section{Justification of the PDE-Structure Remarks}
\label{app:proof_pde_conflict}

We justify Remarks~\ref{thm:pde_conflict} and~\ref{cor:highdim_elliptic} by formalizing how the PDE operator type constrains the gradient-coupling structure between $\mathcal{L}_{\mathrm{PDE}}$ and a Dirichlet BC loss.
This appendix therefore specializes the general constraint-residual notation introduced in Section~\ref{sec:problem} to the Dirichlet BC case $\Gamma_m=\partial\Omega$ with $\mathcal{R}_m[u](x)=u(x)-g(x)$ so that the BC residual is explicit.
Throughout, let $u_\theta : \Omega \to \mathbb{R}$ be a neural network parameterized by $\theta \in \mathbb{R}^p$, and define the Jacobian $J_\theta(x) = \partial u_\theta(x) / \partial \theta \in \mathbb{R}^{1 \times p}$.
For the pointwise half-squared losses
\begin{align*}
    \ell_{\mathrm{PDE},i} &:= \frac12(\mathcal{F}u_\theta(x_i) - f(x_i))^2, \\
    \ell_{\mathrm{BC},j} &:= \frac12(u_\theta(x_j') - g(x_j'))^2,
\end{align*}
the corresponding gradients are:
\begin{align}
\nabla_\theta \ell_{\mathrm{PDE},i} &= \underbrace{(\mathcal{F}u_\theta(x_i) - f(x_i))}_{r_i} \cdot \nabla_\theta [\mathcal{F}u_\theta(x_i)], \label{eq:grad_pde} \\
\nabla_\theta \ell_{\mathrm{BC},j} &= \underbrace{(u_\theta(x_j') - g(x_j'))}_{e_j} \cdot J_\theta(x_j')^\top, \label{eq:grad_bc}
\end{align}
where $\{x_i\}_{i=1}^{N_r} \subset \Omega$ and $\{x_j'\}_{j=1}^{N_b} \subset \partial\Omega$ are interior and boundary collocation points, and the aggregated losses are $\mathcal{L}_{\mathrm{PDE}} = \sum_i \ell_{\mathrm{PDE},i}$ and $\mathcal{L}_{\mathrm{BC}} = \sum_j \ell_{\mathrm{BC},j}$ up to any common positive averaging factors.

\begin{lemma}[Elliptic PDE and Dirichlet-boundary gradient coupling]
\label{lem:elliptic_coupling}
Let $\mathcal{F} = -\nabla \cdot (A(x) \nabla)$ be uniformly elliptic on a bounded domain $\Omega \subset \mathbb{R}^d$ with smooth boundary.
If $u_\theta$ is twice differentiable in $x$ and $\theta$, then for any interior point $x \in \Omega$ and boundary point $x' \in \partial\Omega$:
\begin{align*}
\nabla_\theta [\mathcal{F}u_\theta(x)] = \mathcal{F}_x [J_\theta(x)^\top],
\end{align*}
where $\mathcal{F}_x$ applies the spatial differential operator to each component of $J_\theta$.
Moreover, provided $\nabla_\theta \mathcal{L}_{\mathrm{PDE}} \neq 0$ and $\nabla_\theta \mathcal{L}_{\mathrm{BC}} \neq 0$, the cosine similarity between the aggregated PDE and Dirichlet-boundary gradients satisfies:
\begin{align*}
\cos(\nabla_\theta \mathcal{L}_{\mathrm{PDE}}, \nabla_\theta \mathcal{L}_{\mathrm{BC}}) = \frac{\sum_i \sum_j r_i e_j \langle \mathcal{F}_x[J_\theta(x_i)^\top], J_\theta(x_j')^\top \rangle}{\|\nabla_\theta \mathcal{L}_{\mathrm{PDE}}\| \cdot \|\nabla_\theta \mathcal{L}_{\mathrm{BC}}\|}.
\end{align*}
\end{lemma}

\begin{proof}
The first claim follows from the linearity of $\mathcal{F}$ in $u$ and the smoothness of $u_\theta$ in $(x, \theta)$, which permits interchange of $\mathcal{F}_x$ and $\partial/\partial\theta$.
The cosine similarity expression follows by expanding $\nabla_\theta \mathcal{L}_{\mathrm{PDE}} = \sum_i r_i \mathcal{F}_x[J_\theta(x_i)^\top]$ and $\nabla_\theta \mathcal{L}_{\mathrm{BC}} = \sum_j e_j J_\theta(x_j')^\top$; any common positive averaging factors cancel from the cosine.
\end{proof}

\begin{proof}[Justification of Remark~\ref{thm:pde_conflict}]
For elliptic operators, a uniformly elliptic Dirichlet problem has a solution operator that is compatible with domain-wide state-level influence through Green-function or Poisson-kernel representations whose kernels have global spatial support.
At the gradient level, Lemma~\ref{lem:elliptic_coupling} gives
\begin{align*}
  \nabla_\theta[\mathcal F u_\theta(x)] = \mathcal F_x[J_\theta(x)^\top].
\end{align*}
Therefore, the cosine between PDE and boundary gradients is controlled by the bilinear terms
\begin{align*}
  \left\langle \mathcal F_x[J_\theta(x_i)^\top], J_\theta(x_j')^\top \right\rangle.
\end{align*}
Ellipticity places no locality restriction on these couplings, but their sign and magnitude still depend on the learned Jacobian correlations rather than on ellipticity alone.
This elliptic coupling argument proves the qualitative compatible-with-widespread-coupling claim.

For parabolic operators, scalar linear equations satisfying a comparison principle or a maximum principle have order-preserving state evolution. Larger BC or IC data cannot create a smaller solution later in time.
Hence, the operator itself does not determine the sign of the parameter-gradient cosine.
Negative alignment can already arise through the residual and error signs together with the Jacobian-correlation terms in the cosine formula, and can be further amplified by nonlinear advection, reaction terms, discretization artifacts, or non-monotone parameterizations of $u_\theta$.
This parabolic conclusion is exactly the qualitative statement in the remark.

For hyperbolic operators, the domain-of-dependence theorem implies finite propagation speed. The state at $(x,t)$ depends only on data inside the backward characteristic cone.
Consequently, the state-level influence pattern is localized to causally connected regions.
This finite-propagation structure constrains where PDE and boundary information can interact, unlike the elliptic case, but it does not by itself determine the sign of the parameter-gradient cosine.
Stiff or highly oscillatory terms do not remove finite propagation, but they can shorten the relevant physical length scales and thereby change the effective coupling range seen by a finite-width network.
This finite-propagation argument proves the qualitative hyperbolic claim.
\end{proof}

\begin{proof}[Justification of Remark~\ref{cor:highdim_elliptic}]
For the $d$-dimensional Poisson equation with $u^*(x) = \prod_{i=1}^d \sin(\pi x_i)$, the forcing term is $f(x) = d\pi^2 u^*(x)$.

We begin with the residual-gradient decomposition.
At a collocation point $x_j$,
\begin{align*}
  r_\theta(x_j) = -\sum_{i=1}^d \partial_{x_i}^2 u_\theta(x_j) - f(x_j),
\end{align*}
and because $f$ is fixed data,
\begin{align*}
  \nabla_\theta r_\theta(x_j)
  = -\sum_{i=1}^d \nabla_\theta \partial_{x_i}^2 u_\theta(x_j).
\end{align*}
Thus, the direction of the PDE residual gradient is determined by a sum of $d$ parameter-sensitivity terms, while the residual formula also contains the forcing term $f(x) = d\pi^2 u^*(x)$, which scales linearly with $d$.

Two competing effects then matter.
If the vectors $\nabla_\theta \partial_{x_i}^2 u_\theta(x_j)$ are weakly correlated across coordinates, then the normalized sum $d^{-1}\sum_i \nabla_\theta \partial_{x_i}^2 u_\theta(x_j)$ can become more concentrated as $d$ grows, which tends to stabilize the residual-gradient direction after rescaling.
At the same time, the scalar residual multiplier $r_\theta(x_j)$ depends on the competition between $-\sum_i \partial_{x_i}^2 u_\theta(x_j)$ and the explicit $O(d)$ forcing term, so its magnitude can either grow, shrink, or partially cancel depending on the approximation error.
Since the PDE-loss gradient has the form
\begin{align*}
  \nabla_\theta \mathcal L_{\mathrm{PDE}}
  \propto
  \sum_j r_\theta(x_j)\nabla_\theta r_\theta(x_j),
\end{align*}
its cosine with the boundary gradient depends on both effects simultaneously.

Taken together, the conclusion is as follows.
The concentration effect can push the coupling toward more stable alignment, while the residual-amplitude term can amplify or damp whatever misalignment remains.
Because these effects act in opposite directions, the residual-gradient decomposition by itself does not fix the net $d$-dependence of the instantaneous PDE and BC cosine.
Without further control of the boundary-gradient scaling or of the training trajectory that defines profiled quantities, this decomposition-level argument also does not yield a monotone law for $f_{\mathrm{neg}}(d)$.
Determining whether the trajectory is decreasing, increasing, or non-monotone requires additional assumptions or explicit constructions tied to the forcing family and parameter-sensitivity structure.
\end{proof}

\section{Experimental Setup}
\label{app:exp_setup}

\subsection{PDE Benchmarks}

We evaluate on 60+ PDE configurations, counting distinct PDE families together with parameter, BC, dimension, inverse, and multi-physics variants:

\paragraph{Forward problems ($K=3$).}
Burgers, Allen--Cahn, Helmholtz, Conv-Diff, Klein--Gordon, and Navier--Stokes make up this six-PDE set.
Each PDE uses PDE residual, BC, and IC losses.
The additional forward set includes 3D Helmholtz, reaction-diffusion, Taylor--Green vortex, and Darcy flow.

\paragraph{Expanded criterion validation for $K=3$ and $K=4$.}
Twenty additional forward-PDE configurations broaden the empirical coverage beyond the main forward suite.
The twelve parameter variations are Helmholtz high-frequency (Helmholtz-HF, $\omega=10$), Burgers high-viscosity (Burgers-HV, $\nu=0.1$), two-dimensional (2D) Poisson (Poisson-2D), one-dimensional (1D) heat (Heat-1D), Korteweg--de Vries (KdV), 2D Laplace (Laplace-2D), sharp Allen--Cahn (Allen--Cahn-Sharp), heavy Klein--Gordon (KG-Heavy, $m=2$), Fisher--Kolmogorov--Petrovsky--Piskunov (Fisher--KPP), Telegrapher, Fokker--Planck, and Burgers-Neumann.
The eight new operator families are Cahn--Hilliard with fourth-order parabolic structure, Kuramoto--Sivashinsky with chaotic fourth-order structure, Sine--Gordon as a nonlinear wave equation, Gray--Scott as a two-component reaction-diffusion system with $K=4$, Porous Medium as nonlinear diffusion, Burgers--Fisher as reaction--convection--diffusion, 2D convection (Convection-2D) as a steady problem, and Biharmonic as a fourth-order elliptic problem with $\nabla^4 u = f$.
These span all major operator classes, namely elliptic, parabolic, hyperbolic, fourth-order, and multi-component settings, and they substantially broaden forward-problem coverage.

\paragraph{Inverse problems with $K=3$ and $K=4$.}
Inverse Burgers with parameter $\nu$, inverse heat with parameter $\kappa$ and $K=4$, and inverse Poisson with parameter $\alpha$ include learnable physical parameters that create the gradient heterogeneity needed to understand surgery failures.

\paragraph{Multi-physics with $K$ from 4 to 6.}
The multi-physics suite includes thermoelastic coupling with $K=4$ terms for the heat PDE, wave PDE, BC, and IC; a $K=5$ heat--elasticity variant with an explicit coupling loss; reactive transport with $K=5$ losses for two-species advection--diffusion--reaction and genuine coupling conflict from the $k_r c_1 c_2$ reaction term; and simplified MHD with $K=6$ losses for mass conservation, momentum, induction, the divergence-free constraint, BC, and IC.

\subsection{\texorpdfstring{$K$-Scaling Experiments}{K-Scaling Experiments}}
\label{sec:kscaling-setup}

To evaluate how performance changes with $K$ in the reported results, we focus on natural multi-loss systems: $K=4$ thermoelastic coupling, $K=5$ heat--elasticity and reactive transport, and $K=6$ MHD.
These systems represent genuinely distinct physics at each loss term rather than artificial residual splitting, directly addressing the concern that $K$-scaling validation may rely on constructed decompositions.
The regime-transition framework is used empirically. Once profiling shows that the surrogate $\widehat R$ has dropped below the adapter-favoring range, scalar reweighting is expected empirically to dominate in these settings.

\subsection{Methods Compared}

We compare a broad method suite that includes Vanilla with 140K parameters and Vanilla-Wide with 162K capacity-matched parameters as baselines; gradient-surgery and gradient-alignment methods PCGrad, CAGrad, NashMTL, ConFIG~\citep{liu2025config}, and AlignedMTL~\citep{senushkin2023independent}; loss-weighting methods GradNorm, FAMO, uncertainty weighting~\citep{kendall2018multi}, causal training~\citep{wang2024respecting}, and self-adaptive weights~\citep{mcclenny2023self}; and domain decomposition through XPINN~\citep{jagtap2020extended}.
ConFIG constructs conflict-free updates, whereas AlignedMTL performs gradient alignment based on SVD or eigendecomposition.
The XPINN baselines use 4 and 8 subdomains with roughly comparable capacity, with interface and flux continuity constraints, providing a direct comparison between gradient-space and input-space decomposition.
We also compare our architectural variants \FAMOUAM{}, which is the recommended default with 161K parameters and a 15\% parameter overhead, and \GNUAM{}.

Additional variants such as \UAM{}, \CAM{}, \LCAM{}, and their loss-weighting combinations serve as comparison controls. We also test architecture generality by adding per-loss adapters to Modified MLP~\citep{wang2021understanding} on 6 PDEs and to a Transformer backbone~\citep{zhao2023pinnsformer} on 4 PDEs, as discussed in Appendix~\ref{sec:arch_generality}.

Causal training~\citep{wang2024respecting} enforces temporal causality via exponentially weighted residual chunks, and self-adaptive weights~\citep{mcclenny2023self} learn loss weights as differentiable parameters. NTK weighting~\citep{wang2022and} is evaluated on a representative subset of four PDEs described in Section~\ref{sec:ntk_comparison} using 3 seeds, with reduced collocation points for computational feasibility.

For inverse problems, we additionally report PCGrad-Grouped and CAGrad-Grouped, which apply surgery only to network parameters while updating learnable physical parameters with vanilla updates.

\subsection{Evaluation and Reproducibility Protocol}

\paragraph{Evaluation.}
We report relative $L_2$ error, namely, $\|u_\theta - u_{\text{ref}}\|_2 / \|u_{\text{ref}}\|_2$, on a held-out grid.
Methods are ranked from 1 to $M$ per PDE, where 1 is best and failed methods receive the worst rank.
Forward-benchmark comparisons are additionally summarized with Bayesian bootstrap comparisons using a Dirichlet prior, which quantify the probability of improvement $P(A < B)$.
In the aggregate significance tables, effect sizes are reported as median relative improvement: $\Delta = \text{median}\left((L_2^{\text{base}} - L_2^{\text{method}})/L_2^{\text{base}}\right)$. For a small number of headline pairwise summaries in the prose, we additionally report Cohen's $d$ on the seed-level $L_2$ values.

\paragraph{Reproducibility Protocol.}
Seed counts vary by suite and are stated in each table or caption. The main 10K forward benchmarks use 5 random seeds from 0 to 4. Many supplementary or higher-cost studies use 3 seeds, and targeted sensitivity checks use fewer when explicitly noted. The selected forward comparisons in Table~\ref{tab:seed_ext} rerun seeds 5 to 7 in addition to the default seeds 0 to 4, giving 8 total seeds for those entries.
All experiments run from a single PyTorch 2.8 codebase on graphics processing unit (GPU) hardware, with suite-specific fixed seeds and deterministic collocation sampling.
The default training setup uses 10,000 epochs, the AdamW optimizer with a learning rate of $10^{-3}$ and a weight decay of $10^{-5}$, cosine annealing with a 200-epoch warm-up, and gradient clipping at a norm of 1.0.
The implementation is hardware-agnostic across standard PyTorch GPU backends.
PDE-specific settings and all hyperparameters are listed in Appendix~\ref{app:details}.

\section{Architecture and Hyperparameter Details}
\label{app:details}

\subsection{Default Hyperparameters}
Unless otherwise stated, all experiments use a hidden dimension $d_h=128$, $L=4$ residual blocks, $F=10$ Fourier bands with maximum frequency $10\pi$, and an adapter rank $r=16$.
We train with AdamW at a learning rate of $10^{-3}$ and a weight decay of $10^{-5}$, with 200 warm-up epochs.
Klein--Gordon uses a higher learning rate of $5\times 10^{-3}$.
Adapter-specific settings were determined via a grid search on Burgers and then fixed across all other benchmarks.
The settings are EMA momentum $\beta_{\mathrm{ema}}=0.95$, mixing update frequency $\tau=20$, mixing bias $a_0=1.0$, mixing sensitivity $b=2.0$, latent mixing score range $[\rho_{\min}, \rho_{\max}] = [0.15, 0.95]$, depth prior range $[\delta_1, \delta_L] = [+0.3, -0.3]$, orthogonality weight $\lambda_{\text{ortho}}=0.01$, cross-orthogonality weight $\lambda_{\text{cross}}=0.001$, a gradient clip norm of $1.0$, and magnitude-aware conflict scoring enabled.
Per-layer methods use a gradient clip norm of $0.8$.

\subsection{PDE-Specific Settings}
The default training budget is 10,000 epochs for the main forward suite and most supplementary evaluations unless otherwise stated.
Klein--Gordon requires a higher learning rate of $5\times 10^{-3}$ because its $u_{tt}$ term produces initial PDE residuals $\sim 10^4$.
For per-layer methods on Klein--Gordon, we use a learning rate of $10^{-3}$ with gradient clipping at a norm of $0.5$ and warm-up extended to 500 epochs.
3D problems use 10,000 epochs with 2,000 collocation points unless otherwise stated. The dedicated Helmholtz-3D validation in Table~\ref{tab:3d_validation} uses 15,000 epochs.
Inverse problems use 10,000 epochs with learnable parameters initialized via log-parameterization.
Multi-physics and $K$-scaling suites use experiment-specific budgets. The corresponding tables cover 5,000 to 15,000 epochs, so their captions should be treated as definitive for those experiments.

\subsection{Experiment-wise Dataset, Task, and Budget Summary}
Main forward benchmarks are Burgers, Helmholtz, Allen--Cahn, Conv-Diff, and Klein--Gordon. These synthetic forward-PDE benchmarks recover the solution field using only physics losses. The runs use 5 seeds, 10K epochs, $n_{\mathrm{coll}}=2000$, $n_{\mathrm{bc}}=n_{\mathrm{ic}}=400$, $10^4$ test points, and the default backbone with $d_h=128$, $L=4$, $F=10$, and $r=16$. Klein--Gordon uses a learning rate of $5\times 10^{-3}$, whereas the others use a learning rate of $10^{-3}$.

Inverse problems include inverse Burgers, inverse heat, and inverse Poisson. These settings require joint recovery of the state and one unknown coefficient, namely, $\nu$, $\kappa$, or $\alpha$. We keep the same network and sampling budget, use 10K epochs with log-parameterized scalar initialization, and report 3 to 5 seeds depending on the method.

The natural multi-physics and $K$-scaling suite includes the natural $K=4$ thermoelastic benchmark, a coupled heat--wave problem on $[-1,1]\times[0,1]$ with a finite-difference reference solution on a $200\times200$ grid. These runs use 2K epochs, 1.5K collocation points, and 5 seeds. The natural $K=5$ and $K=6$ studies use reactive transport and simplified MHD with finite-difference references, 15K epochs, $n_{\mathrm{coll}}=2000$, $n_{\mathrm{bc}}=n_{\mathrm{ic}}=400$, and 5 seeds.

The expanded forward validation adds 20 forward PDE configurations and parameter variations, including Poisson-2D, Heat-1D, KdV, Laplace-2D, 1D wave (Wave-1D), Fisher--KPP, Telegrapher, Fokker--Planck, Cahn--Hilliard, Kuramoto--Sivashinsky, Sine--Gordon, Gray--Scott, Porous Medium, Burgers--Fisher, Convection-2D, and Biharmonic. Unless a table caption states otherwise, these runs use 3 seeds, 10K epochs, $d_h=128$, $r=16$, and the same collocation budget as the main suite.

Architecture transfer evaluates the generality of Modified MLP across 6 PDEs, namely, Helmholtz, Burgers, Klein--Gordon, Conv-Diff, Allen--Cahn, and Wave-1D, with 3 seeds and 10K epochs. It evaluates the generality of the Transformer backbone across 4 PDEs, namely, Helmholtz, Burgers, Allen--Cahn, and Klein--Gordon, with 3 seeds, 15K epochs, $d_h=128$, $r=16$, and $n_{\mathrm{coll}}=2000$.

MTL transfer uses two-task benchmarks built from the Modified National Institute of Standards and Technology (MNIST) and Fashion-MNIST datasets. The benchmarks are Multi-MNIST and Multi-FashionMNIST, respectively. Both overlay two digits or garments in a single $36\times36$ image and predict the left and right labels. They use 60K training samples, 10K test samples, 5 seeds, 100 epochs, a batch size of 256, a hidden width of 256, an adapter rank of 16, and Adam with a learning rate of $10^{-3}$. The Canadian Institute for Advanced Research 100-class dataset (CIFAR-100) transfer experiment uses fine-label and coarse-superclass prediction as a two-task classification problem with a shared residual network (ResNet)-style encoder. It uses 3 seeds, 50 epochs, a batch size of 128, a hidden width of 512, an adapter rank of 32, and Adam with a learning rate of $10^{-3}$.

\subsection{Baseline Implementation Details}
Unless otherwise stated, all baselines use the same trunk architecture, data sampling, outer AdamW optimizer, scheduler, and training budget. Method-specific inner updates, auxiliary optimizers, parameter grouping, and learning-rate sweeps are stated here or in the corresponding sensitivity sections.

FAMO~\citep{liu2024famo} maintains logits $z \in \mathbb{R}^K$ updated via dual ascent. The model loss is $\sum_k w_k \mathcal{L}_k$, where $w = \mathrm{softmax}(z)$ with a lower clamp of $0.01$ followed by renormalization. After each step, the logits are updated componentwise as $z_k \leftarrow z_k + \gamma (\mathcal{L}_k^{(t)} - \mathcal{L}_k^{(t-1)})$ with $\gamma=0.01$, which focuses weight on losses that are worsening.

GradNorm follows the reference algorithm~\citep{chen2018gradnorm}. Task weights $w \in \mathbb{R}^K$, initialized to 1, are represented as an \texttt{nn.Parameter} and optimized by a separate Adam optimizer with $\eta_w=0.025$. At each step, gradient norms $G_W^{(k)} = w_k \| \nabla_{W_\ell} \mathcal{L}_k \|_2$ are computed with respect to the last shared layer $W_\ell$. Their average is $\bar G = K^{-1}\sum_k G_W^{(k)}$, the normalized loss ratios are $\tilde{\mathcal L}_k(t) = \mathcal L_k(t)/\mathcal L_k(0)$ and $\xi_k(t) = \tilde{\mathcal L}_k(t) / ((1/K)\sum_j \tilde{\mathcal L}_j(t))$, and the targets are $G_{\mathrm{target}}^{(k)} = \bar G \xi_k(t)^\alpha$ with $\alpha=1.5$. The GradNorm loss $L_{\mathrm{grad}} = \sum_k | G_W^{(k)} - G_{\mathrm{target}}^{(k)} |$ is then minimized by backpropagation through $w$.

PCGrad~\citep{yu2020gradient} and CAGrad~\citep{liu2021conflict} use per-loss gradients over the same trainable parameter vector for the forward benchmarks. For inverse-problem grouped variants, projection or conflict-averse updates are applied only to network parameters, while learnable physical scalars receive vanilla AdamW updates as described in the parameter-grouped surgery section.

NashMTL uses projected gradient ascent in log-space to maximize the Nash product~\citep{navon2022multi}. The Gram matrix $G_{ij} = \langle g_i, g_j \rangle$ is regularized as $G + 0.01 \cdot \bar{G}_{ii} \cdot I$. We update the weights through $\alpha \leftarrow \alpha \cdot \exp(\eta_\alpha \nabla_\alpha)$ with clamped increments satisfying $|\Delta| \leq 0.5$, use 25 inner iterations, and warm-start from the previous solution. A numerical-instability fallback reverts to uniform weights. This is a GPU-native adaptation of NashMTL that uses gradient-based optimization rather than the original solver-based implementation by \citet{navon2022multi}, avoiding expensive host round-trips on GPU hardware.

AlignedMTL~\citep{senushkin2023independent} computes the top-$k$ SVD of the gradient matrix and projects the average gradient onto the principal subspace.
ConFIG uses iterative projection to find a conflict-free direction, with geometric mean scaling~\citep{liu2025config}.

Uncertainty weighting~\citep{kendall2018multi} learns per-loss log-variance parameters with the standard $\exp(-s_k)\mathcal{L}_k + s_k$ formulation.
Causal training~\citep{wang2024respecting} uses exponentially weighted residual chunks, and self-adaptive weighting~\citep{mcclenny2023self, xiang2022selfadaptive} learns differentiable loss weights. NTK weighting and XPINN settings are described in their dedicated comparison sections. These shared defaults and explicitly stated method-specific deviations define the reproducibility protocol used for the comparisons.

\subsection{Model Sizes}
Vanilla has 140,033 parameters.
Vanilla-Wide has 162,013 parameters with $d_h \approx 138$ to match adapter capacity.
\UAM{} has 160,772 parameters with fixed adapters and a 15\% parameter overhead.
\CAM{} has 160,776 parameters with global conflict-aware mixing and a 15\% parameter overhead.
\LCAM{} has 197,661 parameters with per-layer conflict-aware mixing scores and a 41\% parameter overhead.
These adapter parameters remain active in the forward pass at inference, so the listed parameter overhead is not a training-only cost.
For higher-$K$ settings, adapter count scales linearly with $K$.

\subsection{\texorpdfstring{Natural $K=4$ Thermoelastic System}{Natural K=4 Thermoelastic System}}
The natural $K=4$ benchmark couples a heat equation with a wave equation on $[-1,1] \times [0,1]$:
\begin{align}
u_t &= D u_{xx} + \alpha v, \\
v_{tt} &= c^2 v_{xx} + \beta u,
\end{align}
where $D = 0.01$, $\alpha = 0.5$, $c = 1.0$, and $\beta = 0.3$.
The system uses ICs $u(x,0) = \sin(\pi x)$, $v(x,0) = \cos(\pi x/2)$, $v_t(x,0) = 0$, and Dirichlet BCs $u(\pm 1, t) = v(\pm 1, t) = 0$.
The reference solution is computed by finite differences on a $200 \times 200$ grid, using Crank--Nicolson for the heat equation and Verlet integration for the wave equation.
This thermoelastic formulation produces four natural loss terms. These are $\mathcal{L}_{\text{heat}}$ for the heat PDE residual, $\mathcal{L}_{\text{wave}}$ for the wave PDE residual, $\mathcal{L}_{\text{BC}}$ for the BCs, and $\mathcal{L}_{\text{IC}}$ for the ICs of both fields.
Training uses 2,000 epochs with 1,500 collocation points.

\subsection{Parameter-Grouped Gradient Surgery}
The PCGrad-Grouped and CAGrad-Grouped variants partition the parameter vector $\theta$ into network parameters $\theta_{\text{net}}$, which include hidden weights, biases, and embeddings, and physical parameters $\theta_{\text{phys}}$, which include learnable scalars such as viscosity $\nu$, diffusivity $\kappa$, and source coefficient $\alpha$.
Gradient projection with PCGrad or the conflict-aware update with CAGrad is applied only to $\nabla_{\theta_{\text{net}}} \mathcal{L}_k$.
Physical parameters receive vanilla gradient updates.
This grouped-update scheme avoids the catastrophic scale mismatch between $\|\nabla_{\theta_{\text{net}}} \mathcal{L}_k\| \sim 10^{-2}$ and $\|\nabla_{\theta_{\text{phys}}} \mathcal{L}_k\| \sim 10^{2}$.

\subsection{ConFIG Baseline}
ConFIG~\citep{liu2025config} computes a conflict-free gradient direction by iterative pairwise projections with 20 iterations and scales the update by the geometric mean of per-loss gradient norms.
It uses the same learning rate, weight decay, and architecture as the vanilla baseline.
ConFIG incurs roughly $1.7\times$ to $1.9\times$ empirical training-time overhead in our setup, as Table~\ref{tab:efficiency} shows, due to the extra per-loss gradient computation and iterative projection.


\section{Supplementary Cost, Scalability, and Diagnostic Analyses}
\label{app:main_text_diagnostics}

\subsection{Regime-Aware Method Selection}

\begin{algorithm}[t!]
\caption{Regime-aware method selection}
\label{alg:selection}
\begin{algorithmic}[1]
\REQUIRE A PDE problem with $K$ loss terms, a budget constraint $B_{\max}$
\IF{problem has learnable physical parameters}
  \IF{$K = 3$}
    \RETURN FAMO or GradNorm without adapters
  \ELSIF{$K = 4$}
    \RETURN \FAMOUAM{} \COMMENT{$K = 4$ inverse benefits empirically}
  \ENDIF
\ENDIF
\STATE Run $T_{\mathrm{prof}} = 1000$ steps of Vanilla profiling
  \STATE From $\{f_{\mathrm{neg}}^{(t)}\}_{t=1}^{T_{\mathrm{prof}}}$, compute $\widehat f_{\mathrm{neg}} = (1/T_{\mathrm{prof}})\sum_t f_{\mathrm{neg}}^{(t)}$, $\bar f_{\mathrm{neg}}^{\mathrm{early}}$, and $\bar f_{\mathrm{neg}}^{\mathrm{late}}$ as the means over the first third and last third. Also compute $P = \bar f_{\mathrm{neg}}^{\mathrm{late}} / \max(\bar f_{\mathrm{neg}}^{\mathrm{early}}, \varepsilon_P)$ with $\varepsilon_P = 10^{-8}$, together with the least-squares slope of $t \mapsto f_{\mathrm{neg}}^{(t)}$
\STATE Evaluate Vanilla profiling $L_2$ error $\epsilon_{\mathrm{van}}$
\IF{$\epsilon_{\mathrm{van}} < 10^{-3}$}
  \RETURN FAMO alone \COMMENT{an easy problem, adapters unnecessary}
\ENDIF
  \IF{$\widehat f_{\mathrm{neg}} < 0.05$}
  \RETURN FAMO alone \COMMENT{negligible directional conflict}
\ELSIF{$P > 0.8$}
  \RETURN \FAMOUAM{} with a $1.3\times$ training-time cost \COMMENT{persistent conflict}
\ELSIF{$P < 0.5$ \AND slope $< -0.02$}
  \RETURN FAMO alone \COMMENT{transient conflict resolves}
\ELSE
  \RETURN \FAMOUAM{} \COMMENT{default when conflict detected}
\ENDIF
\end{algorithmic}
\end{algorithm}

\begin{table}[t!]
\centering
\caption{Diagnostic regimes and recommended interventions.}
\label{tab:diagnostic_regimes}
\small
\begin{tabular}{lll}
\toprule
Profile & Choice & Appropriate Methods \\
\midrule
Low conflict or easy profile & FAMO only & \citep{wang2021understanding,liu2024famo} \\
High $\widehat M$, low $\widehat R$ & FAMO or GradNorm & \citep{chen2018gradnorm,liu2024famo} \\
Transient: $P<0.5$, slope $<-0.02$ & Reweight only & \citep{chen2018gradnorm,liu2024famo} \\
Persistent: $P>0.8$ & \FAMOUAM{} & \citep{hu2022lora,hwang2024dualcone,wang2025gradientalignment} \\
Inverse or heterogeneous & Avoid full-space surgery & \citep{yu2020gradient,liu2021conflict,navon2022multi} \\
\bottomrule
\end{tabular}
\end{table}

Algorithm~\ref{alg:selection} therefore does not treat $K$ alone as sufficient evidence for method selection; the natural higher-$K$ multi-physics results provide empirical guidance, while the final choice is made by the profiling diagnostics.

\paragraph{Profiling criterion.}
The primary discriminator is the window-average negative-pair fraction $\widehat f_{\mathrm{neg}} := T_{\mathrm{prof}}^{-1}\sum_t f_{\mathrm{neg}}^{(t)}$ from a 1000-step Vanilla profile: $\widehat f_{\mathrm{neg}} < 0.05$ indicates negligible conflict.
We write $\bar f_{\mathrm{neg}}^{\mathrm{early}}$ and $\bar f_{\mathrm{neg}}^{\mathrm{late}}$ for the averages over the first and last thirds of that profile, define the persistence ratio $P = \bar f_{\mathrm{neg}}^{\mathrm{late}} / \max(\bar f_{\mathrm{neg}}^{\mathrm{early}}, \varepsilon_P)$ with $\varepsilon_P = 10^{-8}$, and define the slope as the least-squares slope of $t \mapsto f_{\mathrm{neg}}^{(t)}$.
We use these quantities as a deployment heuristic for choosing between FAMO and \FAMOUAM{}.
Broader comparisons against Vanilla, \UAM{}, or surgery baselines are reported separately in the main result tables.
When conflict is present, $P > 0.8$ is treated as evidence of persistent conflict, while $P < 0.5$ with negative slope indicates transient conflict likely to resolve without architectural separation.
The full procedure in Algorithm~\ref{alg:selection} should therefore be read as a practical decision rule, not as a guarantee of global optimality.
It combines an easy-problem filter with persistence-based classification.

\subsection{Profiling-to-Outcome Bridge}
\label{sec:profile_outcome_bridge}

Figure~\ref{fig:qual_profile_to_outcome_regime} connects the 1000-step profiling signal used by Algorithm~\ref{alg:selection} to the final local error difference between FAMO and \FAMOUAM{} on representative PDEs.
It shows why the selection rule depends on persistence rather than merely detecting negative cosine pairs. Persistent conflict yields spatially structured regions where adapters help, while transient or low-impact conflict leaves little qualitative change.

\begin{figure*}[t!]
  \centering
  \includegraphics[width=\linewidth]{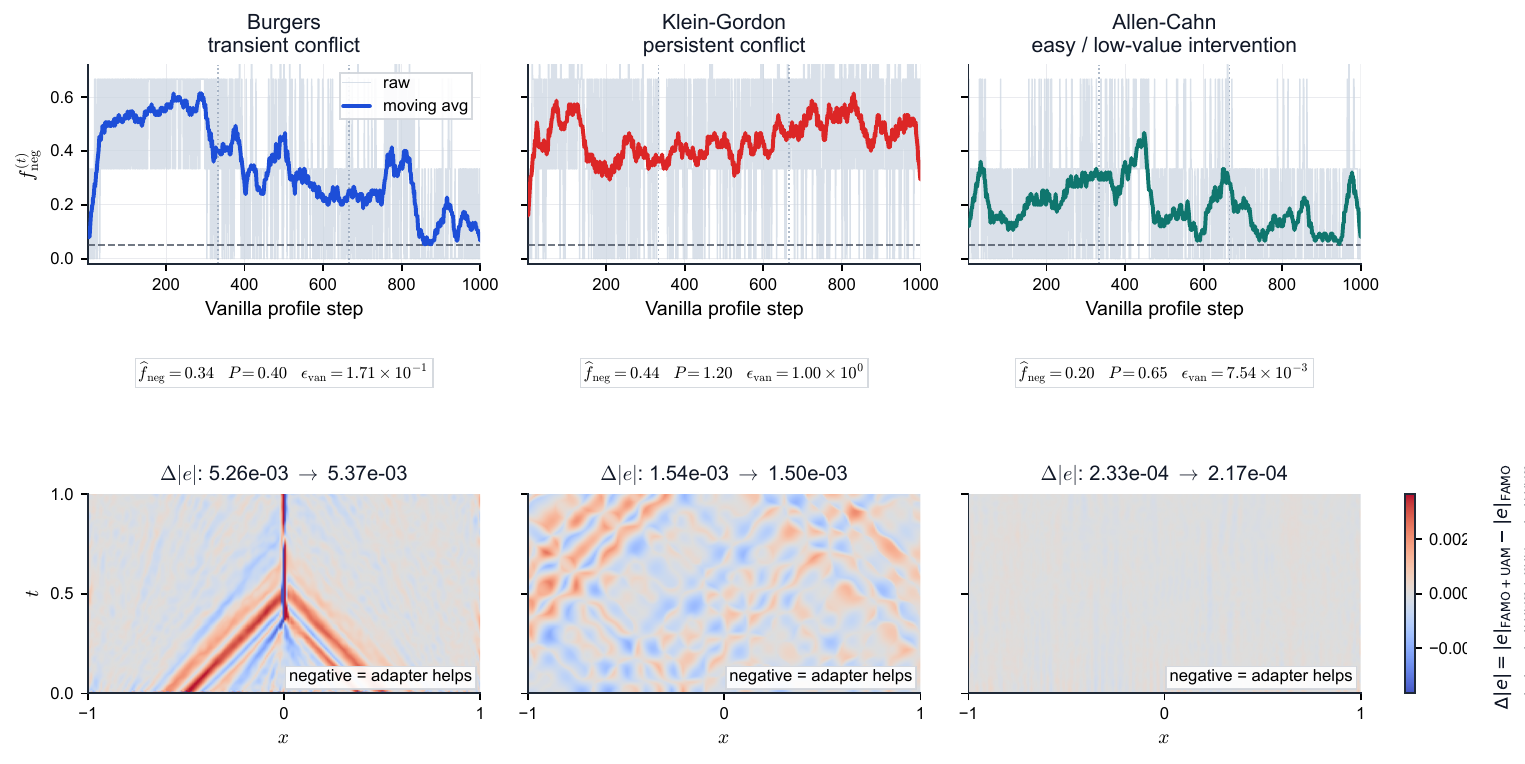}
  \caption{%
    Profiling-to-outcome bridge on three representative forward PDEs.
    The top row shows the 1000-step Vanilla negative-pair trajectory $f_{\mathrm{neg}}^{(t)}$ together with the profile summary $(\widehat f_{\mathrm{neg}}, P, \epsilon_{\mathrm{van}})$, and the bottom row visualizes the local error change $\Delta |e| = |e|_{\text{\FAMOUAM}} - |e|_{\mathrm{FAMO}}$, where negative values indicate regions in which the adapter variant helps.
    Burgers exhibits transient conflict with $P=0.40$, and the adapter effect stays confined to a thin shock-adjacent band, yielding essentially tied final errors of $5.26\times 10^{-3}$ and $5.37\times 10^{-3}$.
    Klein--Gordon maintains persistent conflict with $P=1.20$ and a poor profiled Vanilla error $\epsilon_{\mathrm{van}}=1.00$, and the adapter-induced change concentrates along oscillatory ridges, improving the final error from $1.54\times 10^{-3}$ to $1.50\times 10^{-3}$.
    Allen--Cahn shows weaker conflict with $\widehat f_{\mathrm{neg}}=0.20$ and a much smaller profiled Vanilla error $\epsilon_{\mathrm{van}}=7.54\times 10^{-3}$, and correspondingly the error-difference map is visually negligible.
  }
  \label{fig:qual_profile_to_outcome_regime}
\end{figure*}

\subsection{Adapter Specialization Visualization}
\label{sec:adapter_specialization_visualization}

Figure~\ref{fig:qual_adapter_specialization_regime} complements the shared-output adapter construction in Section~\ref{sec:adapters} and the blockwise NTK interpretation in Appendix~\ref{app:theory_details}.
It shows that, in the adapter-beneficial regime, different loss-indexed adapters learn spatially distinct corrections, whereas the low-conflict control collapses toward a dominant branch; this supports interpreting the gains as conflict-dependent specialization rather than merely added parameter capacity.

\begin{figure*}[t!]
  \centering
  \includegraphics[width=\linewidth]{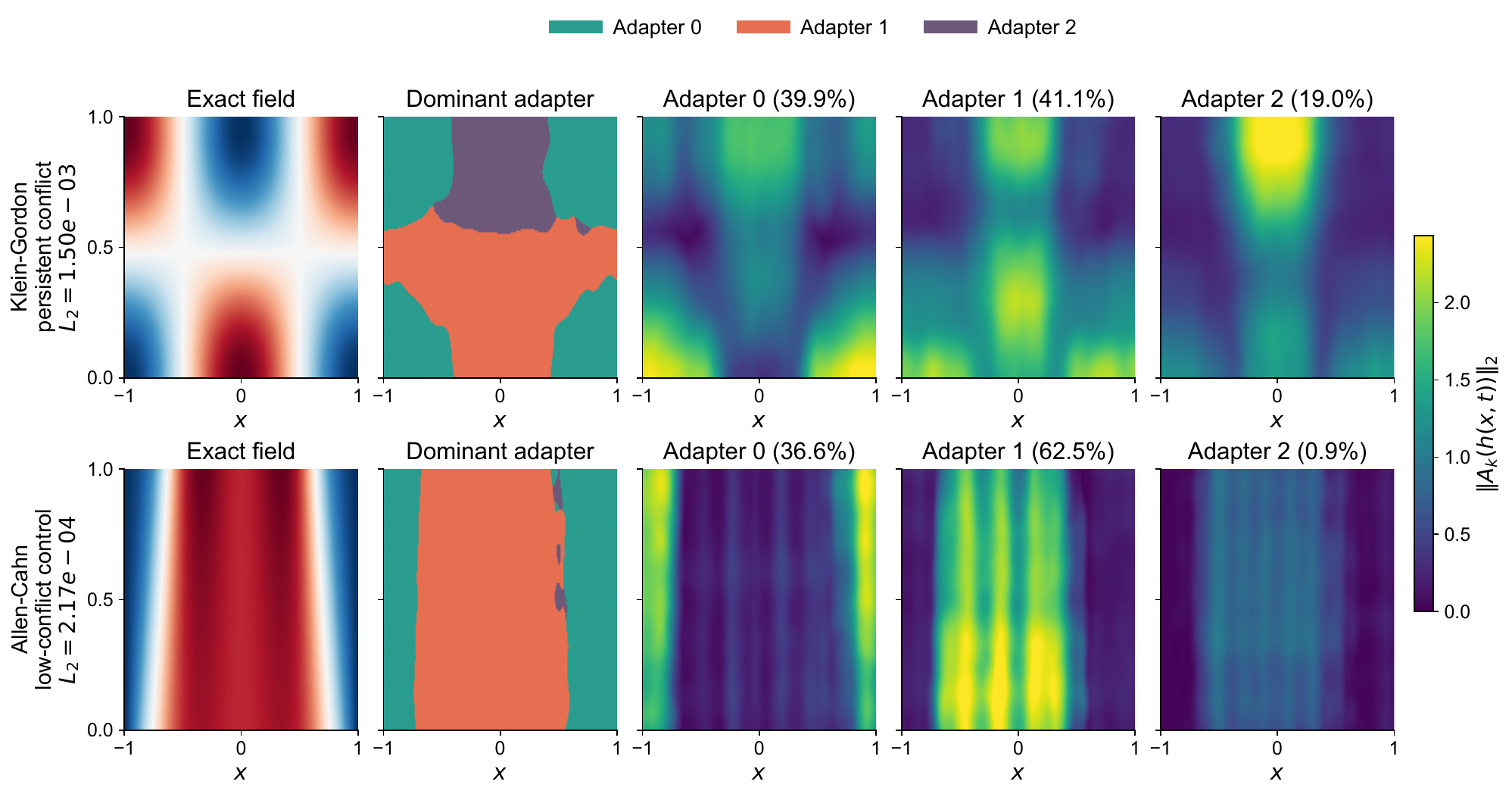}
  \caption{%
    Spatial specialization of the learned per-loss adapters on a persistent-conflict case and a low-conflict control.
    Each row shows the exact field, the dominant adapter index at each space-time point, and the contribution norms $\|A_k(h(x,t))\|_2$ of the three final adapters in the trained \FAMOUAM{} model.
    On Klein--Gordon, all three adapters occupy non-trivial regions of the slab, with shares of 39.9\%, 41.1\%, and 19.0\%, and split the dynamics into boundary- and interior-sensitive correction zones.
    On Allen--Cahn, by contrast, one adapter dominates 62.5\% of the domain and the third adapter is almost unused at 0.9\%.
    Thus, in the adapter-beneficial regime, the additional low-rank branches do more than uniformly rescale the shared trunk; they learn spatially distinct corrections for the same physical field.%
  }
  \label{fig:qual_adapter_specialization_regime}
\end{figure*}

\subsection{Computational Cost and Scalability}
\label{sec:cost_ablation}

\begin{table}[t!]
\centering
\caption{Computational cost comparison at $d_h=128$, 10K epochs.
Wall-clock time (seconds) and overhead relative to Vanilla are shown
separately for each PDE.  All runs use 5 seeds except where noted.
Parameter counts reflect the shared trunk plus any per-loss adapter modules.
GradNorm timings reflect the reference algorithm~\citep{chen2018gradnorm} with gradient-norm computation with respect to the last shared layer.}
\label{tab:efficiency}
\begin{tabular}{l r rr rr}
\toprule
 & & \multicolumn{2}{c}{Helmholtz} & \multicolumn{2}{c}{Burgers} \\
\cmidrule(lr){3-4} \cmidrule(lr){5-6}
Method & \# Params. & Time (s) & Overhead & Time (s) & Overhead \\
\midrule
\multicolumn{6}{l}{Shared-output baselines} \\
Vanilla          & 140,033 &  697 & $1.00\times$ &  440 & $1.00\times$ \\
FAMO             & 140,033 &  692 & $0.99\times$ &  456 & $1.04\times$ \\
GradNorm         & 140,033 &  929 & $1.33\times$ &  613 & $1.39\times$ \\
ConFIG           & 140,033 & 1201 & $1.72\times$ &  848 & $1.93\times$ \\
\midrule
\multicolumn{6}{l}{Uniform adapter mixing (ours)} \\
\UAM        & 160,772 &  933 & $1.34\times$ &  607 & $1.38\times$ \\
\FAMOUAM          & 160,772 &  929 & $1.33\times$ &  615 & $1.40\times$ \\
\GNUAM            & 160,772 & 1262 & $1.81\times$ &  820 & $1.86\times$ \\
\midrule
\multicolumn{6}{l}{Layerwise conflict-aware adapter mixing (ours)} \\
\FAMOLCAM      & 197,661 & 1689 & $2.42\times$ & 1160 & $2.64\times$ \\
\GNLCAM        & 197,661 & 1963 & $2.82\times$ & 1336 & $3.04\times$ \\
\bottomrule
\end{tabular}
\end{table}

Table~\ref{tab:efficiency} shows that \FAMOUAM{} incurs only a $1.3\times$ runtime cost relative to Vanilla, whereas \GNUAM{} rises to $1.8\times$ because of the extra gradient-norm computation.
\LCAM{} variants are substantially more expensive, with overhead ranging from $2.4\times$ to $3.0\times$. We therefore treat \FAMOUAM{} as the default cost-conscious adapter configuration rather than claiming a universal cost optimum across all budgets and PDEs.
Figure~\ref{fig:rank_ablation} shows that higher rank mainly helps on persistent-conflict PDEs, while transient or no-conflict cases are largely rank-insensitive.

\paragraph{Adaptive rank selection.}
The default rank $r=16$ performs well across benchmarks, and the saturation ratio $\Delta_{\mathrm{sat}} = Kr/d_h$ is a useful capacity proxy.
One simple profiling-based heuristic is
\begin{align*}
  r_{\mathrm{heur}} = \left\lceil \frac{d_h \widehat f_{\mathrm{neg}}}{K} \right\rceil,
\end{align*}
where $\widehat f_{\mathrm{neg}}$ is the profiled average fraction of negative pairwise cosines.
This rule is a capacity-matching heuristic rather than a direct theorem consequence, chosen so that $\Delta_{\mathrm{sat}} \approx \widehat f_{\mathrm{neg}}$.
In persistent-conflict cases, the rule suggests a nontrivial adapter rank, whereas in no-conflict cases it collapses to zero, so no adapters are used.
Figure~\ref{fig:rank_ablation} shows that larger ranks help mainly on persistent-conflict PDEs, while transient or no-conflict cases are comparatively insensitive.
We recommend $r=16$ as a conservative mid-range default; practitioners with a profiling step can use $r_{\mathrm{heur}}$ as a coarse starting point.

\subsection{Training Dynamics Visualization}
\label{sec:training_curves}

\begin{figure}[t!]
\centering
\includegraphics[width=\linewidth]{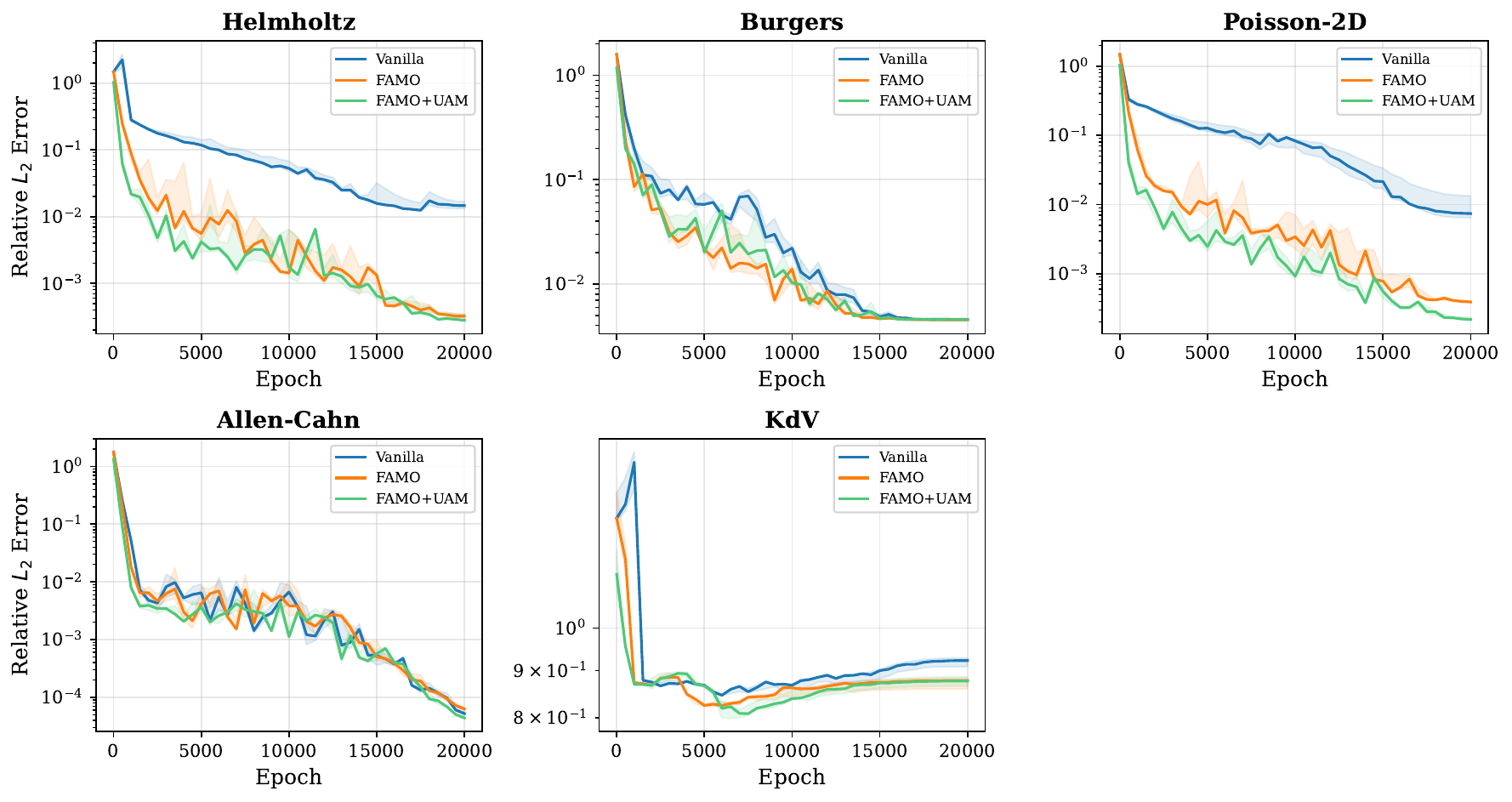}
\caption{Training curves of $L_2$ error by epoch for five PDEs spanning different conflict regimes, with 20K epochs, 3 seeds, and interquartile range (IQR) shading. On Helmholtz and Poisson-2D, which show persistent conflict, \FAMOUAM{} separates from Vanilla early and the gap widens monotonically, reaching about $45\times$ and $33.8\times$ at 20K. On Allen--Cahn with no conflict and Burgers with transient conflict, all methods converge similarly. KdV fails to converge under any method, with $L_2 > 0.8$.}
\label{fig:training_curves}
\end{figure}

Figure~\ref{fig:training_curves} visualizes the training dynamics for five PDEs across the conflict regimes and shows when the performance gap appears and widens.
On Helmholtz, which exhibits persistent conflict with $P > 0.8$, \FAMOUAM{} separates from Vanilla within the first 2000 epochs and the gap widens monotonically to roughly $45\times$ at 20K epochs.
Poisson-2D exhibits the same pattern. \FAMOUAM{} reaches $2.19\times 10^{-4}$ compared with Vanilla's $7.41\times 10^{-3}$, which gives a $33.8\times$ improvement at 20K and confirms that persistent conflict drives sustained adapter benefit.
On Allen--Cahn, where no conflict is detected, all methods converge to the same error of about $5\times 10^{-5}$, confirming that adapters provide no benefit when directional conflict is absent.
On Burgers, which exhibits mixed but transient conflict with $P = 0.63$, all methods converge to nearly identical errors of about $4.5\times 10^{-3}$, with early separation that vanishes at extended training.
KdV fails to converge under any method ($L_2 > 0.8$), indicating a fundamentally hard problem where gradient conflict resolution cannot compensate for architectural limitations.

\subsection{NTK Weighting Comparison}
\label{sec:ntk_comparison}

NTK-based loss weighting~\citep{wang2022and} adjusts loss weights via the NTK trace to balance learning speeds.
We compare NTK weighting against Vanilla, FAMO, and \FAMOUAM{} on 4 representative PDEs spanning different conflict regimes, as reported in Table~\ref{tab:ntk}.
NTK weighting adds negligible training overhead, about $1.06\times$ with periodic kernel recomputation every 100 steps, making it a low-cost alternative to \FAMOUAM{}, whose training-time cost is about $1.3\times$.

Results confirm that the intervention class should match the conflict type.
On the high-conflict Helmholtz PDE, \FAMOUAM{} achieves $5.9\times 10^{-4}$ while NTK weighting reaches $7.5\times 10^{-4}$, showing that persistent directional conflict can require architectural intervention rather than scalar reweighting.
On Klein--Gordon, \FAMOUAM{} ($1.57\times 10^{-3}$) and NTK ($1.58\times 10^{-3}$) are comparable.
On Burgers, which shows moderate and transient conflict, NTK weighting reaches $4.8\times 10^{-3}$ and slightly outperforms \FAMOUAM{} at $6.5\times 10^{-3}$, which is consistent with the regime prediction that magnitude rebalancing suffices for transient conflict.
On Allen--Cahn, where no conflict is detected, Vanilla reaches $1.6\times 10^{-4}$ and outperforms all intervention methods, including NTK at $1.0\times 10^{-3}$, confirming that unnecessary intervention can be harmful.
Figure~\ref{fig:ntk_curves} visualizes the training dynamics. On Helmholtz, \FAMOUAM{} separates from NTK weighting within the first 2000 epochs, confirming that the architectural advantage emerges early and persists.

\begin{table}[t!]
\centering
\caption{NTK weighting comparison on four representative PDEs. Lower relative $L_2$ error is better.}
\label{tab:ntk}
\begin{tabular}{lcccc}
\toprule
PDE & Vanilla & FAMO & NTK Weighting & \FAMOUAM{} (Ours) \\
\midrule
Helmholtz & $4.38\times 10^{-2}$ & $1.25\times 10^{-3}$ & $7.51\times 10^{-4}$ & \textbf{\boldmath$5.93\times 10^{-4}$} \\
Klein--Gordon & $4.62\times 10^{-2}$ & $2.09\times 10^{-3}$ & $1.58\times 10^{-3}$ & \textbf{\boldmath$1.57\times 10^{-3}$} \\
Burgers & $1.36\times 10^{-2}$ & $8.33\times 10^{-3}$ & \textbf{\boldmath$4.83\times 10^{-3}$} & $6.46\times 10^{-3}$ \\
Allen--Cahn & \textbf{\boldmath$1.61\times 10^{-4}$} & $2.22\times 10^{-4}$ & $1.01\times 10^{-3}$ & $1.92\times 10^{-4}$ \\
\bottomrule
\end{tabular}
\end{table}

\begin{figure}[t!]
\centering
\includegraphics[width=\linewidth]{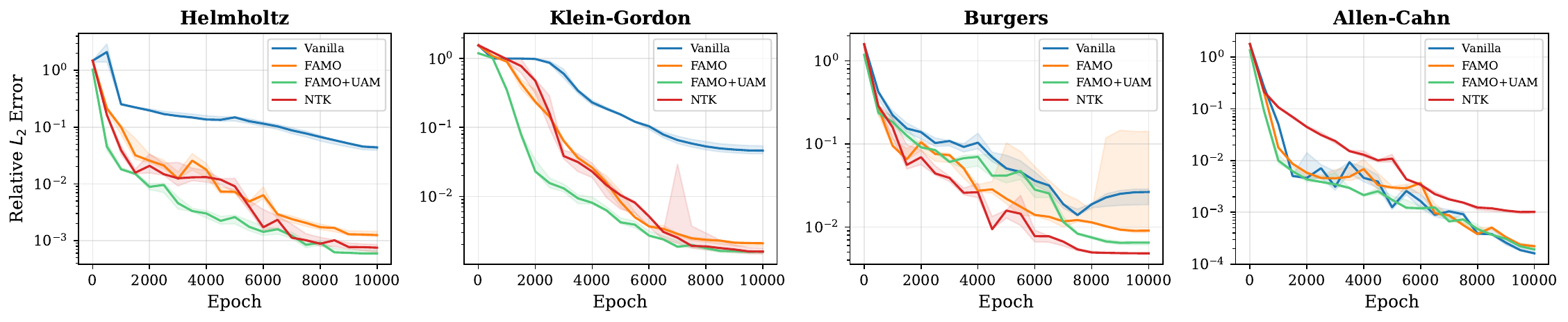}
\caption{Training dynamics for the NTK-weighting comparison with 4 PDEs, 3 seeds, and IQR shading. On Helmholtz and Klein--Gordon, \FAMOUAM{} converges faster and achieves lower final error than NTK weighting. On Burgers, which shows transient conflict, NTK weighting matches or slightly outperforms \FAMOUAM{}. On Allen--Cahn, where no conflict is detected, Vanilla dominates.}
\label{fig:ntk_curves}
\end{figure}

\subsection{Adaptive Rank Selection Validation}
\label{sec:rank_selection}

We use the saturation ratio $\Delta_{\mathrm{sat}} = Kr/d_h$ as a simple capacity proxy for how much adapter specialization a fixed-width trunk can support.
The profiling-based rule
\begin{align*}
  r_{\mathrm{heur}} = \left\lceil \frac{d_h \widehat f_{\mathrm{neg}}}{K} \right\rceil
\end{align*}
is a capacity proxy rather than a theoretical guarantee.
Figure~\ref{fig:rank_ablation} therefore provides end-to-end empirical validation across three PDEs spanning different conflict regimes.
On Helmholtz, which exhibits persistent conflict, performance improves monotonically with rank: $r=4$ gives $6.95\times 10^{-4}$, $r=8$ gives $5.63\times 10^{-4}$, $r=16$ gives $3.72\times 10^{-4}$, $r=32$ gives $3.33\times 10^{-4}$, and $r=64$ gives $3.06\times 10^{-4}$. This Helmholtz rank sweep spans $2.3\times$, with diminishing returns above $r=32$.
On Burgers, which exhibits transient conflict, all five ranks yield indistinguishable errors from $4.57\times 10^{-3}$ to $4.64\times 10^{-3}$. This narrow range corresponds to only a $1.0\times$ span and is consistent with the regime prediction that scalar reweighting suffices.
On Allen--Cahn, where no conflict is detected, rank has minimal effect. Errors range from $6.7\times 10^{-5}$ to $8.6\times 10^{-5}$, which corresponds to a $1.3\times$ span and confirms that adapters are unnecessary.
These results support the heuristic capacity picture. Higher-conflict PDEs benefit from larger adapter subspaces, and the $2.3\times$ gap between $r=4$ and $r=64$ on Helmholtz shows that underpowered adapters can leave performance on the table.

\begin{figure}[t!]
\centering
\includegraphics[width=0.6\linewidth]{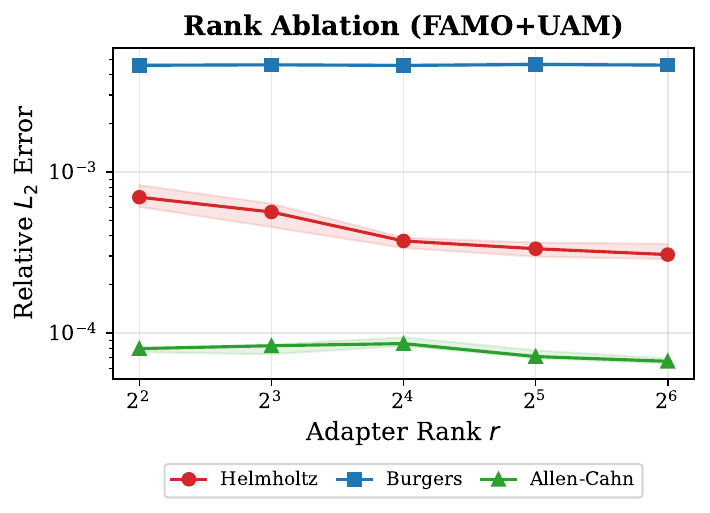}
\caption{Adapter rank ablation across three PDEs spanning different conflict regimes. Optimal rank correlates with conflict persistence. Helmholtz benefits from higher rank, while Allen--Cahn is largely rank-insensitive.}
\label{fig:rank_ablation}
\end{figure}

\subsection{Scalability and Ablation Tables}
\label{app:scalability_tables}

Tables~\ref{tab:scalability} and~\ref{tab:collocation} test whether the adapter benefit persists across changes in network capacity and collocation density rather than depending on a single training configuration.
Several observations follow from these stress tests.
First, increasing trunk capacity does not turn the adapter into a uniformly beneficial add-on. On low-conflict Allen--Cahn, Vanilla remains competitive or best, which is consistent with the profiling rule that recommends no architectural intervention when directional conflict is absent.
Second, Burgers shows the expected transient-conflict pattern.
Across widths, depths, and collocation budgets, FAMO and adapter variants are close, and the best entry alternates rather than being dominated by the adapter; this indicates that scalar loss balancing and ordinary approximation capacity capture most of the remaining difficulty once the early conflict decays.
Third, Helmholtz behaves differently.
The best adapter-and-weighting variants stay between $10^{-4}$ and $10^{-3}$ error across network sizes, while Vanilla remains one to two orders of magnitude worse in most configurations.
The collocation sweep gives the clearest signal: \FAMOUAM{} stays between $1.9\times 10^{-3}$ and $3.7\times 10^{-3}$ from 1K to 20K collocation points, whereas Vanilla stays near $10^{-1}$.
Thus, the Helmholtz improvement is not explained by a fortunate width and depth choice or by sparse residual sampling; it reflects a persistent directional-conflict bottleneck that benefits from loss-indexed parameter subspaces.

\begin{table}[t!]
\centering
\caption{Scalability analysis of network size effects on adapter benefits. Relative $L_2$ error is averaged over 3 seeds. Bold indicates the best result for each configuration. $\dagger$\GNUAM{} scalability data uses a simplified GradNorm variant, whereas the main-forward results in Tables~\ref{tab:main_results_a} and~\ref{tab:main_results_b} use the reference algorithm.}
\label{tab:scalability}
\small
\begin{tabular}{llcccc}
\toprule
PDE & Config & Vanilla & FAMO & \FAMOUAM & \GNUAM$^\dagger$ \\
\midrule
  Allen--Cahn & $d_h=128$, $L=4$ & \textbf{\boldmath$5\times 10^{-4}$} & $7\times 10^{-4}$ & $7\times 10^{-4}$ & $7\times 10^{-4}$ \\
   & $d_h=128$, $L=8$ & \textbf{\boldmath$4\times 10^{-4}$} & $5\times 10^{-4}$ & $6\times 10^{-4}$ & $5\times 10^{-4}$ \\
   & $d_h=256$, $L=4$ & \textbf{\boldmath$2\times 10^{-4}$} & $3\times 10^{-4}$ & \textbf{\boldmath$2\times 10^{-4}$} & \textbf{\boldmath$2\times 10^{-4}$} \\
   & $d_h=256$, $L=8$ & \textbf{\boldmath$1\times 10^{-4}$} & $2\times 10^{-4}$ & $2\times 10^{-4}$ & $5\times 10^{-4}$ \\
   & $d_h=512$, $L=4$ & \textbf{\boldmath$1\times 10^{-4}$} & \textbf{\boldmath$1\times 10^{-4}$} & \textbf{\boldmath$1\times 10^{-4}$} & $4\times 10^{-4}$ \\
\midrule
  Burgers & $d_h=128$, $L=4$ & $3.17\times 10^{-2}$ & $1.08\times 10^{-2}$ & \textbf{\boldmath$1.04\times 10^{-2}$} & $1.06\times 10^{-2}$ \\
   & $d_h=128$, $L=8$ & $6.14\times 10^{-2}$ & $1.33\times 10^{-2}$ & \textbf{\boldmath$9.5\times 10^{-3}$} & $1.00\times 10^{-2}$ \\
   & $d_h=256$, $L=4$ & $4.41\times 10^{-2}$ & \textbf{\boldmath$5.8\times 10^{-3}$} & $1.03\times 10^{-2}$ & $9.4\times 10^{-3}$ \\
   & $d_h=256$, $L=8$ & $2.07\times 10^{-2}$ & \textbf{\boldmath$7.4\times 10^{-3}$} & \textbf{\boldmath$7.4\times 10^{-3}$} & $1.10\times 10^{-2}$ \\
   & $d_h=512$, $L=4$ & $2.93\times 10^{-2}$ & $1.35\times 10^{-2}$ & \textbf{\boldmath$9.9\times 10^{-3}$} & $1.01\times 10^{-2}$ \\
\midrule
  Helmholtz & $d_h=128$, $L=4$ & $1.248\times 10^{-1}$ & $5.3\times 10^{-3}$ & $2.0\times 10^{-3}$ & \textbf{\boldmath$4\times 10^{-4}$} \\
   & $d_h=128$, $L=8$ & $9.62\times 10^{-2}$ & $2.2\times 10^{-3}$ & $1.9\times 10^{-3}$ & \textbf{\boldmath$3\times 10^{-4}$} \\
   & $d_h=256$, $L=4$ & $4.42\times 10^{-2}$ & $1.0\times 10^{-3}$ & $5\times 10^{-4}$ & \textbf{\boldmath$2\times 10^{-4}$} \\
   & $d_h=256$, $L=8$ & $4.48\times 10^{-2}$ & $8\times 10^{-4}$ & $4\times 10^{-4}$ & \textbf{\boldmath$1\times 10^{-4}$} \\
   & $d_h=512$, $L=4$ & $9.6\times 10^{-3}$ & $6\times 10^{-4}$ & $6\times 10^{-4}$ & \textbf{\boldmath$3\times 10^{-4}$} \\
\bottomrule
\end{tabular}
\end{table}

\begin{table}[t!]
\centering
\caption{Scalability with collocation points. Relative $L_2$ error, averaged over 3 seeds. On Helmholtz, \FAMOUAM{} maintains roughly $34\times$ to $72\times$ improvement over Vanilla across all tested $N_{\text{coll}}$ values, where $N_{\text{coll}} \in \{1\text{K}, 2\text{K}, 5\text{K}, 10\text{K}, 20\text{K}\}$, confirming robustness to collocation density.}
\label{tab:collocation}
\small
\begin{tabular}{llccc}
\toprule
PDE & $N_{\text{coll}}$ & Vanilla & FAMO & \FAMOUAM \\
\midrule
  Burgers & 1,000 & $2.60\times 10^{-2}$ & \textbf{\boldmath$1.33\times 10^{-2}$} & $1.50\times 10^{-2}$ \\
   & 2,000 & $2.89\times 10^{-2}$ & $1.43\times 10^{-2}$ & \textbf{\boldmath$7.7\times 10^{-3}$} \\
   & 5,000 & $2.03\times 10^{-2}$ & $1.07\times 10^{-2}$ & \textbf{\boldmath$5.0\times 10^{-3}$} \\
   & 10,000 & $1.83\times 10^{-2}$ & $8.0\times 10^{-3}$ & \textbf{\boldmath$4.7\times 10^{-3}$} \\
   & 20,000 & $1.35\times 10^{-2}$ & $7.7\times 10^{-3}$ & \textbf{\boldmath$4.7\times 10^{-3}$} \\
\midrule
  Helmholtz & 1,000 & $1.070\times 10^{-1}$ & $7.1\times 10^{-3}$ & \textbf{\boldmath$1.9\times 10^{-3}$} \\
   & 2,000 & $1.275\times 10^{-1}$ & $4.8\times 10^{-3}$ & \textbf{\boldmath$3.7\times 10^{-3}$} \\
   & 5,000 & $1.372\times 10^{-1}$ & $5.3\times 10^{-3}$ & \textbf{\boldmath$1.9\times 10^{-3}$} \\
   & 10,000 & $1.092\times 10^{-1}$ & $4.0\times 10^{-3}$ & \textbf{\boldmath$2.0\times 10^{-3}$} \\
   & 20,000 & $1.510\times 10^{-1}$ & $6.1\times 10^{-3}$ & \textbf{\boldmath$2.3\times 10^{-3}$} \\
\bottomrule
\end{tabular}
\end{table}

\section{Supplementary Benchmark Tables}
\label{app:supplementary_benchmarks}

This section collects supplementary benchmark tables that extend the main forward and natural $K=4$ evaluations used to support the regime classification.

\begin{table}[t!]
\centering
\caption{Darcy flow 2D benchmark. We report $L_2$ error as mean $\pm$ std over 3 seeds at 5K epochs. \GNUAM{} achieves the best result with $5.6\times 10^{-2}$, which is a $4.5\times$ improvement over Vanilla and confirms adapter benefits on this standard elliptic PDE benchmark with $K=3$ losses and extreme magnitude imbalance $\bar{r} > 8,500$. $\dagger$GradNorm uses a simplified variant, whereas the main-forward results in Tables~\ref{tab:main_results_a} and~\ref{tab:main_results_b} use the reference algorithm.}
\label{tab:darcy_extended}
\small
\begin{tabular}{lcc}
\toprule
Method & $L_2$ Error & \# Params. \\
\midrule
Vanilla & $(2.48 \pm 0.29)\times 10^{-1}$ & 140K \\
FAMO & $(1.44 \pm 0.43)\times 10^{-1}$ & 140K \\
\FAMOUAM & $(9.7 \pm 0.8)\times 10^{-2}$ & 161K \\
\FAMOLCAM & $(8.8 \pm 0.9)\times 10^{-2}$ & 198K \\
\GNUAM$^\dagger$ & \textbf{\boldmath$(5.6 \pm 0.8)\times 10^{-2}$} & 161K \\
\bottomrule
\end{tabular}
\vspace{-0.5em}
\end{table}

\begin{table}[t!]
\centering
\caption{Natural $K=4$ coupled thermoelastic benchmark. Mean $\pm$ std $L_2$ error over 5 seeds. This naturally coupled system of heat and wave equations provides a $K=4$ validation where the adapter-and-weighting combination outperforms either component alone. Best values are shown in bold. $\dagger$GradNorm uses a simplified variant; the main-forward results in Tables~\ref{tab:main_results_a} and~\ref{tab:main_results_b} use the reference algorithm.}
\label{tab:natural_k4}
\small
\begin{tabular}{lcc}
\toprule
Method & $L_2$ Error & Rank \\
\midrule
Vanilla & $(5.11 \pm 0.83)\times 10^{-1}$ & 10 \\
FAMO & $(1.04 \pm 0.17)\times 10^{-1}$ & 5 \\
GradNorm$^\dagger$ & $(2.28 \pm 0.84)\times 10^{-1}$ & 6 \\
\UAM & $(4.15 \pm 0.58)\times 10^{-1}$ & 7 \\
\FAMOUAM & $(7.95 \pm 2.38)\times 10^{-2}$ & 4 \\
\GNUAM$^\dagger$ & $(5.84 \pm 1.31)\times 10^{-2}$ & 2 \\
\CAM{} & $(4.57 \pm 0.92)\times 10^{-1}$ & 9 \\
\LCAM & $(4.34 \pm 1.06)\times 10^{-1}$ & 8 \\
\FAMOLCAM & $(7.79 \pm 2.1)\times 10^{-2}$ & 3 \\
\GNLCAM$^\dagger$ & \textbf{\boldmath$(4.78 \pm 0.78)\times 10^{-2}$} & \textbf{1} \\
\bottomrule
\end{tabular}
\end{table}

\begin{figure*}[t!]
  \centering
  \includegraphics[width=\linewidth]{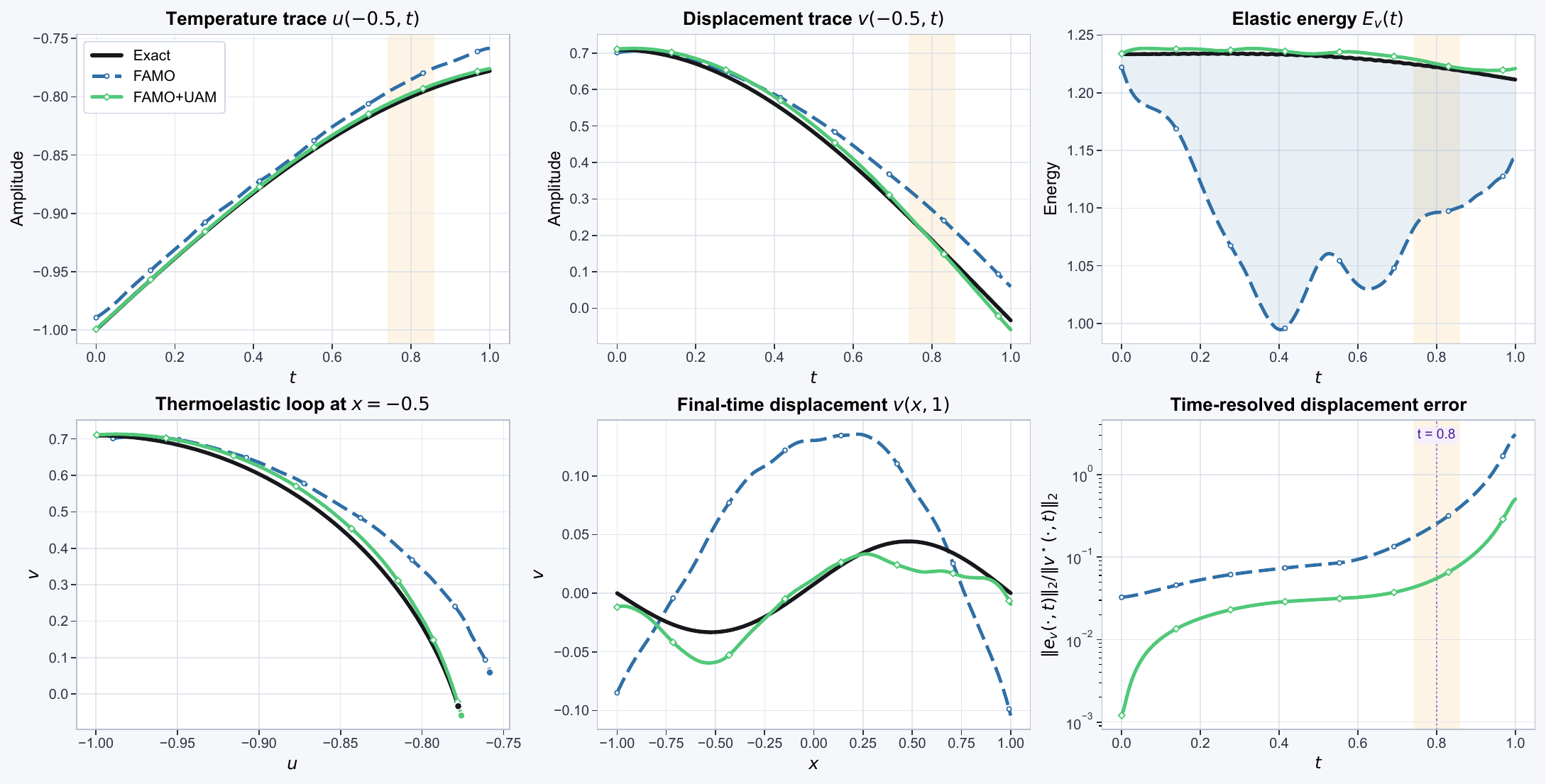}
  \caption{%
    Energy exchange and branchwise dynamics on the natural $K=4$ thermoelastic system.
    The top row compares a representative temperature trace $u(-0.5,t)$, displacement trace $v(-0.5,t)$, and elastic energy $E_v(t)=1/2\int (v_t^2 + c^2 v_x^2) dx$; the bottom row shows the local thermoelastic loop $(u(-0.5,t), v(-0.5,t))$, the final-time displacement profile, and the time-resolved displacement-profile error.
    The qualitative signal is branchwise rather than uniform. FAMO already tracks the thermal branch closely, but it keeps the displacement branch too positive at late time and underestimates the elastic-energy envelope, whereas \FAMOUAM{} restores the signed displacement swing and follows the coupled loop more faithfully.
    On this representative run, the mean relative elastic-energy error drops from $1.15\times 10^{-1}$ to $3.17\times 10^{-3}$, and the displacement-profile error at $t\approx 0.8$ drops from $2.51\times 10^{-1}$ to $5.48\times 10^{-2}$.
  }
  \label{fig:qual_natural_k4_energy_exchange}
\end{figure*}

\section{Statistical Comparison}
\label{app:statistics}

\begin{table}[t!]
\centering
\caption{Bayesian pairwise comparison on forward benchmarks with 5 PDEs, $d_h=128$, and 10K epochs. P(A$<$B) denotes the Dirichlet-weighted bootstrap probability that Method A achieves lower $L_2$ error than Method B, using $n=10,000$ resamples. Each comparison uses 25 matched PDE-seed error pairs. Probabilities above 0.95 are bold.}
\label{tab:bayesian}
\small
\begin{tabular}{llc}
\toprule
Method A & Method B & P(A$<$B) \\
\midrule
\FAMOUAM & FAMO & \textbf{0.991} \\
\FAMOUAM & \GNUAM & \textbf{1.000} \\
\FAMOUAM & GradNorm & \textbf{1.000} \\
\FAMOUAM & \FAMOLCAM & 0.824 \\
\FAMOUAM & \GNLCAM & \textbf{1.000} \\
\FAMOUAM & Vanilla & \textbf{1.000} \\
FAMO & \GNUAM & \textbf{1.000} \\
FAMO & GradNorm & \textbf{1.000} \\
FAMO & \FAMOLCAM & 0.053 \\
FAMO & \GNLCAM & \textbf{1.000} \\
FAMO & Vanilla & \textbf{1.000} \\
\GNUAM & GradNorm & \textbf{1.000} \\
\GNUAM & \FAMOLCAM & 0.000 \\
\GNUAM & \GNLCAM & \textbf{1.000} \\
\GNUAM & Vanilla & \textbf{0.993} \\
GradNorm & \FAMOLCAM & 0.000 \\
GradNorm & \GNLCAM & 0.000 \\
GradNorm & Vanilla & 0.063 \\
\FAMOLCAM & \GNLCAM & \textbf{1.000} \\
\FAMOLCAM & Vanilla & \textbf{1.000} \\
\GNLCAM & Vanilla & 0.773 \\
\bottomrule
\end{tabular}
\end{table}

Table~\ref{tab:bayesian} reports Bayesian bootstrap comparisons with 10,000 Dirichlet-weighted resamples over 25 matched PDE-seed pairs for each method comparison and shows that \FAMOUAM{} consistently outperforms Vanilla with $P > 0.99$, as well as GradNorm and most competitors on the 5 forward PDEs.
The comparison between \FAMOUAM{} and FAMO remains close because FAMO alone is already strong on the less challenging PDEs.
The gap widens on high-conflict problems like Klein--Gordon, where \FAMOUAM{} reaches $29\times$ over Vanilla and FAMO reaches $19\times$.

\section{Extended Validation and Robustness}
\label{app:extended_validation}

\subsection{Extended Forward Validation}
\label{app:forward_extensions}

\subsubsection{3D Validation}
On 3D Helmholtz, shown in Table~\ref{tab:3d_validation}, \FAMOUAM{} achieves a $28.5\times$ improvement over Vanilla.
On Taylor--Green, which shows low conflict, Vanilla matches adapters and confirms that benefit tracks conflict rather than dimensionality.

\begin{table}[t!]
\centering
\small
\caption{3D validation on Helmholtz-3D over $[0,1]^3$ with $K=3$, 5 seeds, $d_h=128$, and 15K epochs. We report relative $L_2$ error as mean $\pm$ std, where lower is better. The adapter benefit persists in 3D with the same regime structure as in 2D.}
\label{tab:3d_validation}
\begin{tabular}{lcc}
\toprule
Method & Relative $L_2$ & Ratio \\
\midrule
Vanilla & $(4.13 \pm 3)\times 10^{-2}$ & $1.0\times$ \\
FAMO & $(3.58 \pm 2.7)\times 10^{-3}$ & $11.5\times$ \\
PCGrad & $(1.80 \pm 0.59)\times 10^{-3}$ & $22.9\times$ \\
ConFIG & $(2.21 \pm 0.14)\times 10^{-3}$ & $18.7\times$ \\
\FAMOUAM & \textbf{\boldmath$(1.45 \pm 0.18)\times 10^{-3}$} & $28.5\times$ \\
\bottomrule
\end{tabular}
\end{table}

The 25-pair Bayesian bootstrap comparisons in Table~\ref{tab:bayesian} support the same pattern. \FAMOUAM{} is favored over Vanilla and most competitors on the forward suite.

\subsubsection{Additional Forward Benchmarks}
Table~\ref{tab:expanded_benchmarks} broadens the forward evaluation across operator families.
The gains remain strongly PDE-dependent. \FAMOUAM{} is clearly best on Poisson-2D, KG-Heavy, Helmholtz-HF, Fisher--KPP, and Burgers-HV, while several easier or magnitude-dominated cases show little or no benefit.
Fokker--Planck exposes a FAMO instability, with a $162\times$ degradation from near-singular loss Hessians, which motivates log-space dual updates in Appendix~\ref{sec:bounded_famo}.
\begin{table}[t!]
\centering
\small
\caption{Expanded PDE benchmarks on 12 additional forward PDEs. We report median $L_2$ error with 3 seeds, $d_h=128$, and 10K epochs unless noted. Lower is better. The operator-type column spells out elliptic, parabolic, hyperbolic, and damped wave. Bold indicates the best method for each PDE. The last column reports the Vanilla-to-best ratio.}
\label{tab:expanded_benchmarks}
\begin{tabular}{lccccr}
\toprule
PDE & Type & Vanilla & FAMO & \FAMOUAM & Ratio \\
\midrule
\multicolumn{6}{l}{Adapter benefit ($\geq 1.5\times$)} \\
Poisson-2D   & Elliptic  & $4.21\times 10^{-2}$ & $1.24\times 10^{-3}$ & \textbf{\boldmath$4.78\times 10^{-4}$} & $88\times$ \\
KG-Heavy     & Hyperbolic  & $5.95\times 10^{-2}$ & $2.64\times 10^{-3}$ & \textbf{\boldmath$1.52\times 10^{-3}$} & $39\times$ \\
Laplace-2D   & Elliptic  & $1.30\times 10^{-2}$ & \textbf{\boldmath$5.81\times 10^{-4}$} & $5.83\times 10^{-4}$ & $22\times$ \\
Telegrapher  & Damped wave & $7.02\times 10^{-1}$ & \textbf{\boldmath$1.58\times 10^{-1}$} & \textbf{\boldmath$1.58\times 10^{-1}$} & $4.5\times$ \\
Helmholtz-HF$^*$ & Elliptic  & $9.39\times 10^{-1}$ & $8.58\times 10^{-1}$ & \textbf{\boldmath$3.24\times 10^{-1}$} & $2.9\times$ \\
Fisher--KPP   & Parabolic  & $6.11\times 10^{-3}$ & $5.27\times 10^{-3}$ & \textbf{\boldmath$2.95\times 10^{-3}$} & $2.1\times$ \\
Burgers-HV   & Parabolic  & $2.78\times 10^{-4}$ & $1.95\times 10^{-4}$ & \textbf{\boldmath$1.53\times 10^{-4}$} & $1.8\times$ \\
\midrule
\multicolumn{6}{l}{No adapter benefit} \\
KdV          & Hyperbolic  & $8.47\times 10^{-1}$ & \textbf{\boldmath$8.13\times 10^{-1}$} & \textbf{\boldmath$8.13\times 10^{-1}$} & $1.04\times$ \\
Burgers-Neumann & Hyperbolic & $2.76\times 10^{-2}$ & \textbf{\boldmath$2.74\times 10^{-2}$} & $2.99\times 10^{-2}$ & $1.01\times$ \\
Heat-1D      & Parabolic  & $4.77\times 10^{-5}$ & $5.95\times 10^{-5}$ & \textbf{\boldmath$4.69\times 10^{-5}$} & $1.02\times$ \\
Allen--Cahn-Sharp & Parabolic & \textbf{\boldmath$1.01\times 10^{-3}$} & $1.14\times 10^{-3}$ & $1.15\times 10^{-3}$ & $1.00\times$ \\
Fokker--Planck$^\dagger$ & Parabolic  & \textbf{\boldmath$4.47\times 10^{-3}$} & $7.26\times 10^{-1}$ & $8.26\times 10^{-1}$ & $1.00\times$ \\
\bottomrule
\end{tabular}
\smallskip

\noindent\small $^*$15K epochs. $^\dagger$FAMO catastrophically fails and becomes $162\times$ worse than Vanilla.
\end{table}

\subsubsection{Criterion Expansion Validation}
To complement the 10K expanded forward-benchmark table with longer-horizon runs, we evaluate 6 representative PDEs at 15K epochs with 3 seeds each, as shown in Figure~\ref{fig:criterion_expansion}.
These are separate 15K validation runs, not part of the 10K profiling and variant analyses discussed in Section~\ref{sec:method_selection}.
Results confirm the regime characterization. KG-Heavy, a hyperbolic case with $K=3$, shows $27.8\times$ adapter benefit. Helmholtz-HF, the high-frequency variant, shows $3.4\times$, and Telegrapher shows $2.0\times$.
Fisher--KPP is mixed at 15K. \FAMOUAM{} achieves a $1.6\times$ improvement over Vanilla, but FAMO attains the best median overall, with $3.10\times 10^{-3}$ compared with $3.39\times 10^{-3}$ for \FAMOUAM{}.
The sharp Allen--Cahn variant, with a ratio of $1.0\times$, confirms no benefit for reaction-dominated PDEs.
Fokker--Planck remains outside the adapter-beneficial regime. FAMO fails catastrophically, with median $3.4\times 10^{-1}$ compared with Vanilla's $1.6\times 10^{-3}$, and \FAMOUAM{} is worse still.
This case should therefore be read as a cautionary limitation of applying the default \FAMOUAM{} combination without screening for loss-scale stiffness rather than as evidence against the diagnostic-first framework.

\begin{figure}[t!]
\centering
\includegraphics[width=\linewidth]{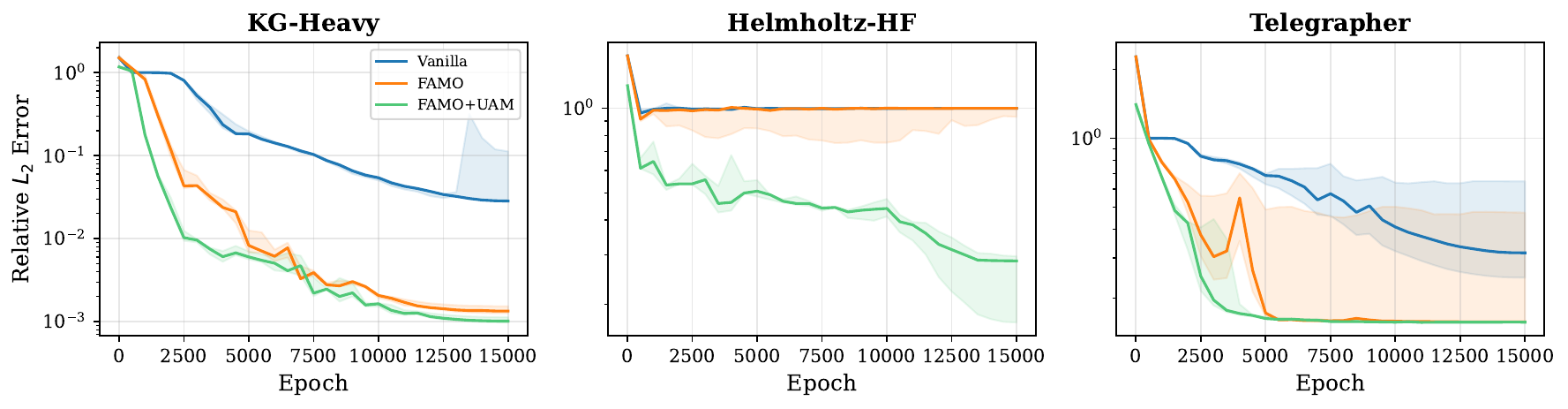}
\caption{Training curves for three criterion-expansion PDEs with 3 seeds and IQR shading. KG-Heavy and Helmholtz-HF show persistent adapter benefit, with gains of $27.8\times$ and $3.4\times$, and the gap widens monotonically.}
\label{fig:criterion_expansion}
\end{figure}

\subsubsection{Helmholtz Scaling Across Width, Dimension, and Training Budget}
\label{sec:scalability}
Helmholtz serves as a stress test for multi-loss methods. As shown in Table~\ref{tab:helmholtz_scaling}, the best listed multi-loss method improves over Vanilla by at least ${19}\times$ across six configurations spanning 2D and 3D settings, $d_h \in \{128,192,256,384\}$, and 10K or 15K epochs, with a peak of $36\times$ at $d_h=256$ and 15K epochs.
The 3D row, drawn from Table~\ref{tab:3d_validation}, achieves $28.5\times$, indicating that the empirical Helmholtz benefit persists beyond 2D and is consistent with the same qualitative persistent-conflict regime, while Remark~\ref{thm:pde_conflict} should not be read as implying dimension-invariant effect sizes.

\begin{table}[t!]
\centering
\caption{Helmholtz improvement across scales. We report relative $L_2$ error for the best multi-loss method compared with Vanilla at different widths, epochs, and dimensions. Across the 10K and 15K settings, the improvement ratio consistently exceeds $19\times$, which confirms that the benefit is structural because it reflects conflict mitigation rather than capacity. Best values are shown in bold.}
\label{tab:helmholtz_scaling}
\small
\begin{tabular}{lllccc}
\toprule
Config & Seeds & Epochs & Vanilla & Best method & Ratio \\
\midrule
$d_h=128$ (2D) & 5 & 10K & $2.8\times 10^{-2}$ & \textbf{\boldmath$1.0\times 10^{-3}$} (\GNUAM) & $28\times$ \\
$d_h=128$ (2D) & 3 & 15K & $1.2\times 10^{-2}$ & \textbf{\boldmath$4.4\times 10^{-4}$} (\FAMOUAM) & $27\times$ \\
$d_h=192$ (2D) & 3 & 15K & $7.3\times 10^{-3}$ & \textbf{\boldmath$2.6\times 10^{-4}$} (\FAMOUAM) & $28\times$ \\
$d_h=256$ (2D) & 3 & 15K & $8.1\times 10^{-3}$ & \textbf{\boldmath$2.2\times 10^{-4}$} (\FAMOUAM) & $36\times$ \\
$d_h=384$ (2D) & 3 & 15K & $3.8\times 10^{-3}$ & \textbf{\boldmath$2.0\times 10^{-4}$} (\FAMOUAM) & $19\times$ \\
$d_h=128$ (3D) & 5 & 15K & $4.1\times 10^{-2}$ & \textbf{\boldmath$1.5\times 10^{-3}$} (\FAMOUAM) & $28.5\times$ \\
\bottomrule
\end{tabular}
\vspace{-0.5em}
\end{table}

\subsubsection{Additional PDE Generalization Checks}
\label{sec:heldout}
Across the supplementary PDE checks in Table~\ref{tab:heldout}, which use 3 to 5 seeds, the conflict profile remains directionally consistent.
Clear adapter-beneficial cases include Klein--Gordon at $29\times$, Poisson-Multiscale at $48\times$, Wave-1D at $19\times$, and Darcy flow at $2.0\times$, while low-conflict or transient cases show little or no adapter advantage in Reaction-Diffusion, Navier--Stokes, Advection, and Eikonal.
On Reaction-Diffusion, \GNUAM{} degrades $4\times$, giving the clearest case where the criterion correctly recommends against adapters.

\begin{table}[t!]
\centering
\caption{Additional PDE evaluation with $d_h=128$. Persistent-conflict PDEs show clear adapter gains, whereas low-conflict or transient-conflict cases show little or no adapter benefit, consistent with regime-dependent selection. Best values are bold and second-best values are underlined.}
\label{tab:heldout}
\small
\resizebox{\linewidth}{!}{%
\begin{tabular}{lcccc}
\toprule
PDE & Vanilla & FAMO & \FAMOUAM & \GNUAM \\
\midrule
\multicolumn{5}{l}{10K evaluation subset with 3 to 5 seeds} \\
Klein--Gordon & $(4.6 \pm 1)\times 10^{-2}$ & \underline{$(2.46 \pm 0.39)\times 10^{-3}$} & \textbf{\boldmath$(1.62 \pm 0.38)\times 10^{-3}$} & $(4.8 \pm 1.1)\times 10^{-2}$ \\
Reaction-Diffusion & \textbf{\boldmath$(8.94 \pm 0.68)\times 10^{-4}$}\seedmark{3} & \underline{$(1.08 \pm 0.063)\times 10^{-3}$} & $(1.12 \pm 0.055)\times 10^{-3}$ & $(3.59 \pm 0.35)\times 10^{-3}$ \\
Darcy flow & $(1.34 \pm 0.49)\times 10^{-1}$\seedmark{3} & $(7.5 \pm 0.8)\times 10^{-2}$ & \underline{$(6.9 \pm 0.3)\times 10^{-2}$} & \textbf{\boldmath$(6.8 \pm 0.5)\times 10^{-2}$} \\
Navier--Stokes & $(2.41 \pm 1.2)\times 10^{-4}$\seedmark{3} & \textbf{\boldmath$(1.75 \pm 0.23)\times 10^{-4}$} & \underline{$(2.45 \pm 0.2)\times 10^{-4}$} & $(4.34 \pm 0.63)\times 10^{-4}$ \\
\midrule
\multicolumn{5}{l}{3-seed extended subset with 15K to 20K epochs} \\
Wave-1D & $(2.69 \pm 0.41)\times 10^{-2}$ & \underline{$(2.11 \pm 0.53)\times 10^{-3}$} & \textbf{\boldmath$(1.44 \pm 0.67)\times 10^{-3}$} & NA \\
Advection & \textbf{\boldmath$(1.79 \pm 0.24)\times 10^{-3}$} & \underline{$(2.09 \pm 0.23)\times 10^{-3}$} & $(2.21 \pm 0.44)\times 10^{-3}$ & NA \\
Eikonal (15K) & \textbf{\boldmath$(2.22 \pm 0.02)\times 10^{-4}$} & $(3.16 \pm 0.1)\times 10^{-4}$ & \underline{$(2.69 \pm 0.076)\times 10^{-4}$} & NA \\
Eikonal (20K) & \textbf{\boldmath$(1.61 \pm 0.034)\times 10^{-4}$} & NA & \underline{$(1.80 \pm 0.071)\times 10^{-4}$} & NA \\
Poisson-Multiscale & $(1.74 \pm 0.09)\times 10^{-1}$ & \underline{$(2.09 \pm 2.3)\times 10^{-2}$} & \textbf{\boldmath$(3.65 \pm 0.046)\times 10^{-3}$} & NA \\
\bottomrule
\multicolumn{5}{l}{\footnotesize Superscript $[n]$ denotes seed count when $<$5, and NA denotes not evaluated.}
\end{tabular}
}
\smallskip
\noindent\footnotesize NA denotes not evaluated.
\end{table}

\subsection{Diagnostic Transfer to Standard MTL Benchmarks}
\label{sec:mtl_validation}

The profiled ratio $\widehat R = \widehat D / (\widehat M + \varepsilon_R)$ is domain-agnostic in the sense that it depends only on loss-specific gradient geometry, not on PDE-specific structure.
It compares implementation-level summaries of directional conflict and a coefficient-of-variation surrogate for magnitude imbalance between loss-specific gradients using the same profiling pipeline as in the PINN experiments.
We test whether it correctly predicts the appropriate intervention class. This validation focuses on whether the regime classification transfers across domains, not on whether adapters improve MTL.
We evaluate three standard $K=2$ MTL benchmarks spanning magnitude-dominated settings, namely, Multi-MNIST, Multi-FashionMNIST~\citep{sener2018multi}, and CIFAR-100.
Across all three benchmarks, the diagnostic is directionally consistent.
No adapter-based variant is the best method on any benchmark, as shown in Table~\ref{tab:mtl_validation}.

\begin{table}[t!]
\centering
\caption{Profile-score diagnostic validation on MTL benchmarks ($K=2$). CIFAR-100 uses a ResNet-18 shared encoder.}
\label{tab:mtl_validation}
\footnotesize
\resizebox{\linewidth}{!}{%
\begin{tabular}{lccccccccc}
\toprule
Method & \multicolumn{3}{c}{Multi-MNIST} & \multicolumn{3}{c}{Multi-FashionMNIST} & \multicolumn{3}{c}{CIFAR-100 (Fine+Coarse)} \\
\cmidrule(lr){2-4} \cmidrule(lr){5-7} \cmidrule(lr){8-10}
 & Task L & Task R & Avg & Task L & Task R & Avg & Fine & Coarse & Avg \\
\midrule
Vanilla & 98.6 & 98.2 & 98.40$\pm$0.17 & 88.4 & 87.6 & 87.98$\pm$0.28 & 48.3 & 63.2 & 55.75$\pm$0.19 \\
PCGrad & 98.6 & 98.3 & 98.45$\pm$0.18 & 88.9 & 88.1 & \textbf{88.48}$\pm$0.21 & 50.2 & 65.0 & \textbf{57.60}$\pm$0.23 \\
FAMO & 98.7 & 98.3 & \textbf{98.47}$\pm$0.24 & 88.9 & 87.9 & 88.39$\pm$0.10 & 47.6 & 62.5 & 55.05$\pm$0.19 \\
\UAM{} & 98.6 & 98.2 & 98.40$\pm$0.12 & 88.7 & 88.1 & 88.40$\pm$0.22 & 47.2 & 63.5 & 55.35$\pm$0.39 \\
\FAMOUAM{} & 98.6 & 98.2 & 98.40$\pm$0.12 & 88.9 & 87.9 & 88.40$\pm$0.14 & 47.4 & 63.4 & 55.40$\pm$0.16 \\
\bottomrule
\end{tabular}
}
\par\vspace{0.4ex}
\begin{minipage}{\linewidth}
\footnotesize
Profile statistics: $\widehat D=0.035$, $\widehat M=0.55$, $f_{\mathrm{neg}}=0.37$ for Multi-MNIST, and $\widehat D=0.056$, $\widehat M=0.30$, $f_{\mathrm{neg}}=0.41$ for Multi-FashionMNIST. $\widehat D<0.01$, $\widehat M=0.17$, and $f_{\mathrm{neg}}<0.05$ for CIFAR-100, which places all three in the magnitude regime. The prediction is that adapters are unnecessary, and this prediction is supported because no adapter-based variant beats the strongest non-adapter baseline on Multi-MNIST, Multi-FashionMNIST, or CIFAR-100.
\end{minipage}
\end{table}

These MTL checks serve as an external sanity test for the do-not-intervene side of the rule. When magnitude imbalance dominates and task-specific heads already absorb much of the interference, architectural separation provides no consistent benefit.

\subsection{Mitigating Catastrophic Weight Divergence with Bounded FAMO}
\label{sec:bounded_famo}

The Fokker--Planck FAMO failure, which corresponds to a $162\times$ degradation at 10K epochs, reveals that FAMO's componentwise absolute-scale dual update
\begin{align*}
  z_k &\leftarrow z_k + \gamma(\mathcal{L}_k^{(t)} - \mathcal{L}_k^{(t-1)})
\end{align*}
can diverge when losses span multiple orders of magnitude.
A log-space variant,
\begin{align*}
  z_k &\leftarrow z_k + \gamma(\log(\mathcal{L}_k^{(t)} + \varepsilon_L) - \log(\mathcal{L}_k^{(t-1)} + \varepsilon_L)),
\end{align*}
where $\varepsilon_L = 10^{-12}$, provides scale-invariant updates and achieves $L_2 = 1.1\times 10^{-2}$ on Fokker--Planck. The log-space update gives a $64\times$ improvement over FAMO with 3 seeds and changes well-behaved PDEs by less than 5\%.
The bounded variant therefore substantially mitigates but does not fully eliminate the instability. It improves sharply over standard FAMO yet still trails the best stable baselines, and pairing it with adapters did not remove the remaining gap in our runs.

\subsection{Architecture Generality}
\label{sec:arch_generality}

\paragraph{Modified MLP backbone.}
To verify architecture-agnostic regime transfer, we evaluate Modified MLP~\citep{wang2021understanding}, a multiplicative-gating architecture widely used in PINN practice, on 6 PDEs with 3 seeds and 10K epochs.
The pattern of adapter benefits transfers. Helmholtz shows a $734\times$ improvement with \FAMOUAM{}, and Klein--Gordon shows $154\times$. Wave-1D shows $495\times$ because adapters rescue non-converging vanilla Modified MLP, and Burgers shows $2.6\times$.
Allen--Cahn and Conv-Diff show no benefit above $1.1\times$, which matches the regime predictions.
This Modified MLP transfer result confirms that the profiling criterion operates on gradient statistics, not model structure.

\paragraph{Transformer backbone.}
\phantomsection\label{sec:transformer_generality}

To further verify architecture-agnostic regime transfer, we evaluate using a Transformer-based PINN backbone, a self-attention architecture that uses a learnable positional embedding over Fourier-encoded point features.
This Transformer evaluation tests a fundamentally different inductive bias, namely global attention rather than local MLP processing, and the design is inspired by PINNsFormer~\citep{zhao2023pinnsformer}.
We evaluate 4 PDEs, namely Helmholtz, Burgers, Allen--Cahn, and Klein--Gordon, across 2 configurations, Vanilla Transformer and \FAMOUAM{} Transformer, with 3 seeds each for a total of 24 experiments.

Table~\ref{tab:transformer} reports the median $L_2$ errors.
The regime prediction transfers cleanly. On Helmholtz with persistent conflict, \FAMOUAM{} Transformer achieves $5.4\times 10^{-4}$ compared with Vanilla Transformer's $1.2\times 10^{-1}$, which gives a $225.8\times$ improvement. On Klein--Gordon, \FAMOUAM{} Transformer achieves $1.8\times 10^{-3}$, compared with $2.1\times 10^{-2}$ for Vanilla Transformer, which gives an $11.5\times$ improvement.
On Allen--Cahn, where no conflict is detected, both configurations converge to comparable errors near $5.9\times 10^{-5}$ with a ratio of $1.0\times$.
The Transformer backbone also benefits from adapters on Burgers with a gain of $2.4\times$, which slightly exceeds the MLP ratio and is consistent with the hypothesis that Transformers may amplify inter-loss interference through shared attention parameters.
This Transformer transfer result establishes that the conflict structure is shaped primarily by the PDE operator and loss decomposition, not by the choice of neural network architecture alone.

\begin{table}[t!]
\centering
\caption{Architecture generality for the Transformer backbone. We report median relative $L_2$ error across 3 seeds at 15K epochs. The regime prediction transfers. The adapter-beneficial cases of Helmholtz, Burgers, and Klein--Gordon show clear gains on the Transformer architecture, while Allen--Cahn, a no-conflict control, remains essentially unchanged as predicted.}
\label{tab:transformer}
\small
\begin{tabular}{lcccc}
\toprule
 & \multicolumn{2}{c}{MLP} & \multicolumn{2}{c}{Transformer} \\
\cmidrule(lr){2-3} \cmidrule(lr){4-5}
PDE & Vanilla & \FAMOUAM & Vanilla & \FAMOUAM \\
\midrule
Helmholtz & $4.38\times 10^{-2}$ & $5.93\times 10^{-4}$ & $1.21\times 10^{-1}$ & \textbf{\boldmath$5.36\times 10^{-4}$} \\
Burgers & $1.36\times 10^{-2}$ & $6.46\times 10^{-3}$ & $1.17\times 10^{-2}$ & \textbf{\boldmath$4.83\times 10^{-3}$} \\
Klein--Gordon & $4.62\times 10^{-2}$ & \textbf{\boldmath$1.57\times 10^{-3}$} & $2.09\times 10^{-2}$ & $1.82\times 10^{-3}$ \\
Allen--Cahn & $1.61\times 10^{-4}$ & $1.92\times 10^{-4}$ & \textbf{\boldmath$5.83\times 10^{-5}$} & $5.96\times 10^{-5}$ \\
\bottomrule
\end{tabular}
\vspace{-0.5em}
\end{table}

\subsection{\texorpdfstring{Natural $K$-Scaling and Multi-Physics Diagnostics}{Natural K-Scaling and Multi-Physics Diagnostics}}
\label{app:kscaling_extended}

\paragraph{$K=5$ regime-transition analysis.}
\phantomsection\label{app:k5_analysis}
Table~\ref{tab:k5_analysis} isolates the $K=5$ coupled case and shows that increasing adapter rank does not recover the gap to FAMO alone, reinforcing that the bottleneck here is magnitude imbalance rather than insufficient adapter capacity.

\begin{table}[t!]
\centering
\caption{Multi-physics $K=5$ rank sensitivity with $d_h=128$, 15,000 epochs, and 3 seeds. Relative $L_2$ error is reported as mean $\pm$ std. This is the regime in which the diagnostic would select FAMO alone. FAMO outperforms all \FAMOUAM{} rank variants, showing that extra adapter capacity is not the missing ingredient and that the dominant bottleneck is magnitude imbalance rather than persistent directional conflict.}
\label{tab:k5_analysis}
\small
\begin{tabular}{lcc}
\toprule
Configuration & $L_2$ Error & \# Params. \\
\midrule
FAMO alone ($d_h=128$) & \textbf{\boldmath$(4.5 \pm 0.9)\times 10^{-2}$} & 140,162 \\
\midrule
\FAMOUAM{} ($r=4$, $d_h=128$) & $(1.33 \pm 0.01)\times 10^{-1}$ & 153,671 \\
\FAMOUAM{} ($r=8$, $d_h=128$) & $(8.6 \pm 3.4)\times 10^{-2}$ & 158,791 \\
\FAMOUAM{} ($r=16$, $d_h=128$) & $(1.39 \pm 0.36)\times 10^{-1}$ & 169,031 \\
\FAMOUAM{} ($r=32$, $d_h=128$) & $(1.32 \pm 0.27)\times 10^{-1}$ & 189,511 \\
\bottomrule
\end{tabular}
\vspace{-0.5em}
\end{table}

\paragraph{Gradient flow analysis.}
\phantomsection\label{app:gradient_flow}
Table~\ref{tab:gradient_flow} reports the trunk-gradient and final conflict statistics for the $K=5$ multi-physics setting. FAMO balances loss values through a softmax-weighted combination with dual-variable updates while keeping the largest per-loss trunk gradient moderate.
Table~\ref{tab:gradient_flow} shows $\|\nabla_{\text{trunk}}\|_{\max}=2.75$ with $f_{\text{neg}}^{\text{final}}=0.70$.

GradNorm instead produces much larger trunk-gradient spikes ($\|\nabla_{\text{trunk}}\|_{\max}=7.64$) while leaving final directional conflict at a similar level, with $f_{\text{neg}}^{\text{final}}=0.67$ compared with $0.70$ for FAMO. In this multi-physics setting, magnitude equalization primarily amplifies dominant trunk updates rather than resolving pairwise interference.

\begin{table}[t!]
\centering
\caption{Gradient flow analysis on $K=5$ multi-physics at 5,000 epochs. Relative $L_2$ error is reported as mean $\pm$ std. $\|\nabla_{\text{trunk}}\|_{\max}$ denotes the largest per-loss trunk gradient norm at convergence. $f_{\text{neg}}^{\text{final}}$ denotes the fraction of negative pairwise cosines at the final epoch.}
\label{tab:gradient_flow}
\small
\begin{tabular}{lccc}
\toprule
Method & Relative $L_2$ & $\|\nabla_{\text{trunk}}\|_{\max}$ & $f_{\text{neg}}^{\text{final}}$ \\
\midrule
  FAMO & \textbf{\boldmath$(5.8 \pm 2.4)\times 10^{-2}$} & 2.75 & 0.70 \\
  \FAMOUAM & $(1.45 \pm 0.24)\times 10^{-1}$ & 1.02 & 0.80 \\
  \UAM & $(1.87 \pm 0.28)\times 10^{-1}$ & 0.82 & 0.63 \\
  GradNorm & $(3.75 \pm 0.12)\times 10^{-1}$ & 7.64 & 0.67 \\
  \GNUAM & $(5.14 \pm 0.17)\times 10^{-1}$ & 1.80 & 0.70 \\
\bottomrule
\end{tabular}
\end{table}

\paragraph{Natural K-scaling validation.}

We empirically probe the regime-transition picture on natural multi-physics systems where $K$ arises from physical conservation laws. Table~\ref{tab:natural_kscaling} illustrates the regime-transition picture on two natural multi-physics systems.
On reactive transport ($K=5$), FAMO alone achieves a $1.25\times$ improvement over Vanilla by improving from $6.57\times 10^{-2}$ to $5.25\times 10^{-2}$, while \FAMOUAM{} provides only $1.13\times$, which shows that adapters add no benefit in this example.
On simplified MHD ($K=6$), FAMO achieves $1.27\times$ by improving from $6.12\times 10^{-2}$ to $4.80\times 10^{-2}$, while \FAMOUAM{} reaches $5.07\times 10^{-2}$ and $1.21\times$, again underperforming FAMO alone.
In both cases, the magnitude-regime picture is the better description for these examples. Loss reweighting alone suffices, and the additional adapter overhead is not justified.

\begin{table}[t!]
\centering
\small
\caption{Natural multi-physics $K$-scaling for $K=5$ and $K=6$. Reactive transport with $K=5$ is a two-species advection--diffusion--reaction system with genuine coupling conflict. MHD with $K=6$ contains mass, momentum, induction, divergence-free, BC, and IC terms. The profiled surrogate $\widehat R$ places both systems in the magnitude-dominated regime, and the results confirm that FAMO provides the dominant benefit while adapters add marginal or no improvement. The table reports relative $L_2$ error at 15K epochs. Results use 5 seeds except for MHD \FAMOUAM{}, which uses 4 seeds.}
\label{tab:natural_kscaling}
\begin{tabular}{llcc}
\toprule
System & Method & Relative $L_2$ & Ratio Relative to Vanilla \\
\midrule
\multirow{3}{*}{Reactive Transport ($K=5$)}
 & Vanilla & $(6.57 \pm 0.32)\times 10^{-2}$ & $1.0\times$ \\
 & FAMO & \textbf{\boldmath$(5.25 \pm 0.47)\times 10^{-2}$} & $1.25\times$ \\
 & \FAMOUAM & $(5.80 \pm 0.18)\times 10^{-2}$ & $1.13\times$ \\
\midrule
\multirow{3}{*}{Simplified MHD ($K=6$)}
 & Vanilla & $(6.12 \pm 0.063)\times 10^{-2}$ & $1.0\times$ \\
 & FAMO & \textbf{\boldmath$(4.80 \pm 0.093)\times 10^{-2}$} & $1.27\times$ \\
 & \FAMOUAM & $(5.07 \pm 0.27)\times 10^{-2}$ & $1.21\times$ \\
\bottomrule
\end{tabular}
\end{table}

\subsection{Extended Seed Coverage}
\label{app:seed_ext}
\label{app:5seed}

To verify that the main results are robust to random seed variability, we first extend seed coverage on representative PDEs to up to 5 seeds, with superscripts marking cases where fewer than 5 runs are reported in Table~\ref{tab:validation_5seed}. We then further extend to 8 seeds indexed from 0 to 7 on selected forward PDEs with all three key configurations in Table~\ref{tab:seed_ext}.
The progressively narrower CIs confirm the statistical reliability of the main results.
On Helmholtz, the 8-seed \FAMOUAM{} mean is $8.3\times 10^{-4}$ and the Vanilla mean is $2.3\times 10^{-2}$, yielding a $27\times$ ratio with non-overlapping 95\% CIs.

\begin{table}[t!]
\centering
\caption{Extended validation with up to 5 seeds, using $d_h=128$ and 10K epochs. Superscripts denote cases with fewer runs, and the results remain statistically significant with narrower CIs.}
\label{tab:validation_5seed}
\small
\begin{tabular}{lccc}
\toprule
Method & Burgers & Helmholtz & Conv-Diff \\
\midrule
Vanilla & $(8.4 \pm 2.4)\times 10^{-3}$ & $(2.8 \pm 1.1)\times 10^{-2}$ & $(3.55 \pm 0.74)\times 10^{-4}$ \\
FAMO & $(4.8 \pm 0.15)\times 10^{-3}$\seedmark{3} & $(1.4 \pm 0.63)\times 10^{-3}$ & $(3.95 \pm 1.4)\times 10^{-4}$ \\
\FAMOUAM & $(4.8 \pm 0.079)\times 10^{-3}$\seedmark{4} & $(1.1 \pm 0.46)\times 10^{-3}$ & \textbf{\boldmath$(2.70 \pm 0.21)\times 10^{-4}$} \\
\GNUAM & \textbf{\boldmath$(4.6 \pm 0.12)\times 10^{-3}$} & \textbf{\boldmath$(1.0 \pm 0.44)\times 10^{-3}$} & $(1.0 \pm 0.19)\times 10^{-3}$ \\
\bottomrule
\multicolumn{4}{l}{\footnotesize Superscript $[n]$ denotes seed count when $<$5.}
\end{tabular}
\end{table}

\begin{table}[t!]
\centering
\caption{Extended seed coverage with relative $L_2$ error over 8 seeds indexed from 0 to 7 on selected forward PDEs at $d_h=128$ and 10K epochs. The narrower CIs compared with the 5-seed main tables confirm statistical robustness.}
\label{tab:seed_ext}
\small
\begin{tabular}{llcc}
\toprule
PDE & Method & Mean $\pm$ Std & 95\% CI \\
\midrule
\multirow{3}{*}{Helmholtz} & Vanilla & $(2.3 \pm 1.1)\times 10^{-2}$ & $[1.4\times 10^{-2}, 3.2\times 10^{-2}]$ \\
 & FAMO & $(1.05 \pm 0.69)\times 10^{-3}$ & $[4.7\times 10^{-4}, 1.6\times 10^{-3}]$ \\
 & \FAMOUAM & $(8.3 \pm 5.2)\times 10^{-4}$ & $[3.9\times 10^{-4}, 1.3\times 10^{-3}]$ \\
\midrule
\multirow{3}{*}{Klein--Gordon} & Vanilla & $(3.9 \pm 1.3)\times 10^{-2}$ & $[2.8\times 10^{-2}, 5.0\times 10^{-2}]$ \\
 & FAMO & $(2.39 \pm 0.34)\times 10^{-3}$ & $[2.1\times 10^{-3}, 2.7\times 10^{-3}]$ \\
 & \FAMOUAM & $(1.45 \pm 0.37)\times 10^{-3}$ & $[1.1\times 10^{-3}, 1.8\times 10^{-3}]$ \\
\midrule
\multirow{3}{*}{Burgers} & Vanilla & $(7.0 \pm 2.7)\times 10^{-3}$ & $[4.8\times 10^{-3}, 9.2\times 10^{-3}]$ \\
 & FAMO & $(4.60 \pm 0.23)\times 10^{-3}$ & $[4.4\times 10^{-3}, 4.8\times 10^{-3}]$ \\
 & \FAMOUAM & $(4.68 \pm 0.23)\times 10^{-3}$ & $[4.5\times 10^{-3}, 4.9\times 10^{-3}]$ \\
\bottomrule
\end{tabular}
\vspace{-0.5em}
\end{table}

\section{Additional Baseline and Ablation Analyses}
\label{app:additional_baseline_ablations}

\subsection{Missing Baseline Comparisons}
\label{app:missing_baselines}

To ensure comprehensive baseline coverage, we add causal training~\citep{wang2024respecting}, self-adaptive weights~\citep{mcclenny2023self}, and uncertainty weighting~\citep{kendall2018multi} to the comparison on selected core PDEs. These methods are included in the 10K main forward tables, Tables~\ref{tab:main_results_a} and~\ref{tab:main_results_b}, and evaluated at 20K epochs below in Table~\ref{tab:missing_baselines}.
Full-scale NTK weighting~\citep{wang2022and} is excluded from this extended-baseline table due to the prohibitive cost of per-epoch NTK trace computation at a scale of ${\sim}10^5$ Jacobian columns per epoch.

\begin{table}[t!]
\centering
\caption{Extended baselines on selected core PDEs at $d_h=128$, 20K epochs (5 seeds unless noted). Causal, self-adaptive, and uncertainty weighting complement the 10K main forward tables. Full-scale NTK weighting is excluded from this extended-baseline table due to prohibitive per-epoch NTK trace computation at this scale.}
\label{tab:missing_baselines}
\small
\resizebox{\linewidth}{!}{%
\begin{tabular}{lcccc}
\toprule
Method & Burgers & Helmholtz & Klein--Gordon & Conv-Diff \\
\midrule
Vanilla & $(4.6 \pm 0.15)\times 10^{-3}$\seedmark{3} & $(8.1 \pm 3.2)\times 10^{-3}$\seedmark{3} & $(1.5 \pm 0.4)\times 10^{-2}$\seedmark{3} & NA \\
FAMO & \textbf{\boldmath$(4.6 \pm 0.16)\times 10^{-3}$\seedmark{3}} & $(2.67 \pm 0.71)\times 10^{-4}$\seedmark{3} & $(8.70 \pm 1.1)\times 10^{-4}$\seedmark{3} & NA \\
Causal & $(4.6 \pm 0.12)\times 10^{-3}$ & $(1.2 \pm 0.4)\times 10^{-2}$ & $(9.19 \pm 0.77)\times 10^{-1}$ & \textbf{\boldmath$(2.04 \pm 0.15)\times 10^{-4}$\seedmark{4}} \\
SelfAdapt & $(4.6 \pm 0.14)\times 10^{-3}$ & $(1.1 \pm 1.3)\times 10^{-2}$ & $(6.7 \pm 1.1)\times 10^{-3}$ & $2.07\times 10^{-4}$\seedmark{1} \\
Uncertainty & $(4.6 \pm 0.14)\times 10^{-3}$ & $(4.1 \pm 2.6)\times 10^{-3}$ & $(5.98 \pm 5.41)\times 10^{-1}$ & NA \\
\UAM & $(4.6 \pm 0.088)\times 10^{-3}$ & $(6.7 \pm 1.7)\times 10^{-3}$ & $(1.5 \pm 0.3)\times 10^{-2}$ & NA \\
\FAMOUAM & $(4.6 \pm 0.21)\times 10^{-3}$\seedmark{2} & \textbf{\boldmath$(2.27 \pm 0.23)\times 10^{-4}$\seedmark{3}} & \textbf{\boldmath$(7.93 \pm 0.98)\times 10^{-4}$\seedmark{3}} & NA \\
\bottomrule
\end{tabular}
}
\end{table}

Causal training achieves strong results on Burgers at $4.6\times 10^{-3}$, which matches \FAMOUAM{}; however, it catastrophically fails on Klein--Gordon at $0.92$.
This Klein--Gordon causal-training error is $1000\times$ worse than the \FAMOUAM{} result of $7.9\times 10^{-4}$ and is consistent with its design for time-dependent PDEs with causal structure, whereas Helmholtz and Klein--Gordon lack strict temporal causality.
Self-adaptive weights perform well on Klein--Gordon at $6.7\times 10^{-3}$ but show high variance on Helmholtz, reaching $(1.1 \pm 1.3)\times 10^{-2}$ with one diverging seed.
These results confirm that temporal reweighting strategies are PDE-dependent, whereas the regime characterization provides a principled, PDE-agnostic method selection framework.

\subsection{Gradient Surgery Mitigation Details}
\label{app:surgery_mitigation}

This analysis separates full-parameter-space gradient-surgery failures from grouped updates to show how heterogeneous network and physical parameters drive inverse-problem instability.

\begin{table}[t!]
\centering
\caption{Gradient surgery mitigations on inverse problems. Grouped variants apply surgery only to network parameters and leave the scalar physical parameters untouched. Direct baselines in Table~\ref{tab:inverse_10k} include Vanilla, FAMO, GradNorm, and adapter variants. This table isolates the full-parameter-space surgery failures and grouped mitigations. Entries marked fail indicate failed full-parameter-space training.}
\label{tab:surgery_mitigation}
\small
\begin{tabular}{lccc}
\toprule
Method & Inverse Burgers & Inverse Heat & Inverse Poisson \\
\midrule
PCGrad & fail & fail & fail \\
CAGrad & fail & fail & fail \\
NashMTL & fail & fail & fail \\
ConFIG & fail & fail & fail \\
PCGrad-Grouped & $8.70\times 10^{-2}$ & $8.65\times 10^{-2}$ & $2.855\times 10^{-1}$ \\
CAGrad-Grouped & $9.21\times 10^{-2}$ & $1.082\times 10^{-1}$ & $6.394\times 10^{-1}$ \\
\bottomrule
\end{tabular}
\end{table}

\subsection{ConFIG Sensitivity Analysis}
\label{app:config}

This analysis checks whether ConFIG's mixed performance is explained by learning-rate choice or by the gradient geometry induced by unit-vector normalization. The strongest learning-rate sensitivity concern is ConFIG, whose Helmholtz performance is unusually dependent on $\eta$.
Table~\ref{tab:config_lr} addresses this concern directly. Lower learning rates improve ConFIG on Helmholtz, but its Burgers failure persists across all five tested values and Klein--Gordon remains clearly worse than Vanilla even at its best setting.
Thus, the cross-PDE conclusions are not explained by a single unlucky baseline learning rate.
Table~\ref{tab:config_lr} further shows that ConFIG's failure persists across five learning rates from $10^{-4}$ to $5\times 10^{-3}$ on both Burgers, where it is $19\times$ worse than Vanilla at the best learning rate, and Klein--Gordon, where it is $2.7\times$ worse at the best learning rate. This learning-rate sweep confirms that the failure is intrinsic to the unit-vector normalization rather than a hyperparameter issue.

\begin{figure}[t!]
\centering
\includegraphics[width=\linewidth]{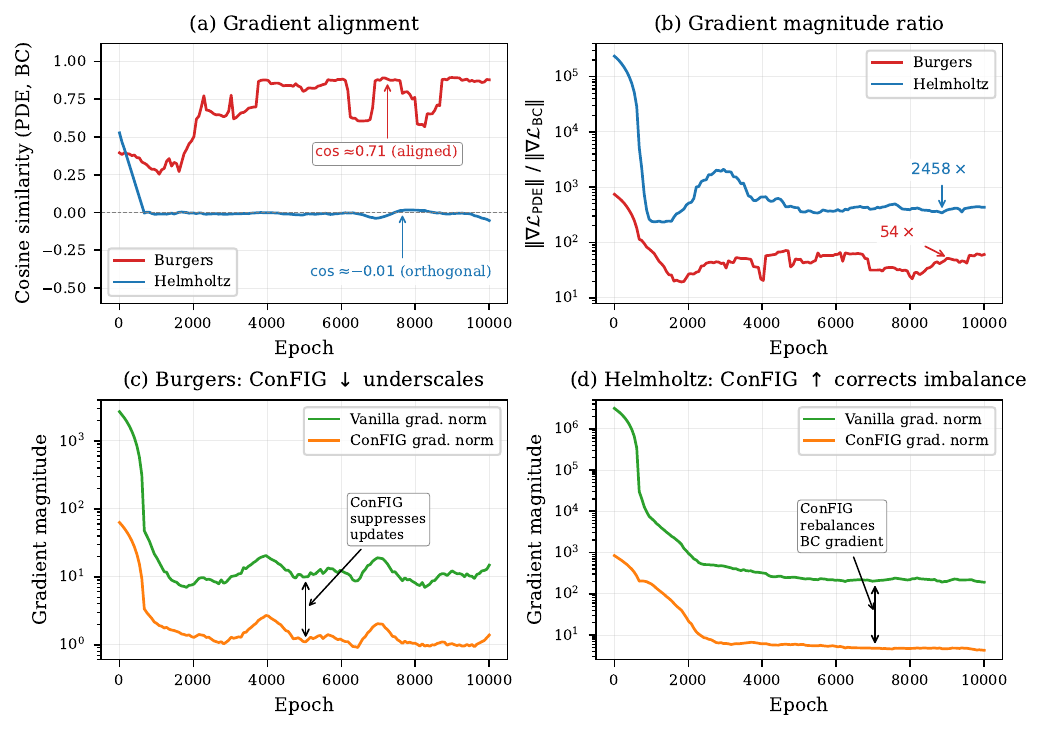}
\caption{ConFIG gradient dynamics on Burgers and Helmholtz. ConFIG suppresses learning on Burgers, where the gradients are aligned with $\cos \approx 0.71$, but rebalances on Helmholtz, where the gradients are nearly orthogonal with $\cos \approx -0.01$.}
\label{fig:config_dynamics}
\end{figure}

\begin{table}[t!]
\centering
\caption{ConFIG learning rate sensitivity analysis with 2 seeds, 10K epochs, and $d_h=128$. ConFIG's Burgers failure persists across all five learning rates and remains $19\times$ worse than Vanilla at the best learning rate, which confirms that the failure is intrinsic to the unit-vector normalization rather than a hyperparameter issue. On Klein--Gordon, ConFIG is also worse than Vanilla by $2.7\times$ at the best learning rate. On Helmholtz, ConFIG performs well at $\eta \leq 10^{-3}$.}
\label{tab:config_lr}
{\small
\begin{tabular}{llcc}
\toprule
PDE & Method & $\eta$ & Relative $L_2$ error \\
\midrule
Burgers & ConFIG & $10^{-4}$ & \benchval{9.4}{1.4}{-2} \\
 & ConFIG & $5\times 10^{-4}$ & \benchval{2.31}{1.31}{-1} \\
 & ConFIG & $10^{-3}$ & \benchval{1.44}{0.42}{-1} \\
 & ConFIG & $2\times 10^{-3}$ & \benchval{8.8}{0.4}{-2} \\
 & ConFIG & $5\times 10^{-3}$ & \benchval{3.39}{0.10}{-1} \\
 & Vanilla & $10^{-3}$ & \benchval{4.7}{0.2}{-3} \\
\midrule
Klein--Gordon & ConFIG & $10^{-4}$ & $(1.047 \pm 0.113)$ \\
 & ConFIG & $5\times 10^{-4}$ & \benchval{5.1}{1.7}{-2} \\
 & ConFIG & $10^{-3}$ & \benchval{4.1}{0.3}{-2} \\
 & ConFIG & $2\times 10^{-3}$ & \benchval{1.58}{0.15}{-1} \\
 & ConFIG & $5\times 10^{-3}$ & \benchval{9.71}{0.22}{-1} \\
 & Vanilla & $10^{-3}$ & \benchval{1.5}{0.4}{-2} \\
\midrule
Helmholtz & ConFIG & $10^{-4}$ & \benchval{4.0}{0.58}{-4} \\
 & ConFIG & $5\times 10^{-4}$ & \benchval{1.9}{0.31}{-4} \\
 & ConFIG & $10^{-3}$ & \benchval{2.5}{0.22}{-4} \\
 & ConFIG & $2\times 10^{-3}$ & \benchval{9.5}{1.1}{-4} \\
 & ConFIG & $5\times 10^{-3}$ & \benchval{1.3}{0.16}{-2} \\
\bottomrule
\end{tabular}
}
\end{table}

\subsection{Implicit Regularization Ablation}
\label{app:implicit_reg}

This ablation, summarized in Table~\ref{tab:implicit_reg}, tests whether the gains come from the per-loss adapter structure itself rather than from the modest increase in parameter count.
The control is a wider non-adapter PINN whose parameter count roughly matches \UAM{}.
On Burgers, where the profiling results indicate transient or low-impact conflict, the parameter-matched variants are essentially tied, so the extra capacity does not create a spurious advantage.
On Helmholtz, the wider non-adapter model improves only modestly over Vanilla, whereas \FAMOUAM{} reaches an order-of-magnitude lower error; together with the specialization evidence in Figure~\ref{fig:qual_adapter_specialization_regime}, this supports the claim that the high-conflict gains arise from loss-indexed adapter subspaces rather than dense widening alone.

\begin{table}[t!]
\centering
\caption{Implicit regularization ablation at 15,000 epochs. Relative $L_2$ error is reported as mean $\pm$ std over 3 seeds, with parameter counts shown separately. Vanilla-Wide roughly matches the \UAM{} parameter count but not the performance, confirming that the per-loss adapter architecture drives improvement.}
\label{tab:implicit_reg}
\small
\begin{tabular}{lrcc}
\toprule
Configuration & \# Params. & Burgers & Helmholtz \\
\midrule
Vanilla & 140,033 & $(4.7 \pm 0.1)\times 10^{-3}$ & $(1.76 \pm 0.53)\times 10^{-2}$ \\
Vanilla-Wide & 162,013 & $(4.8 \pm 0.2)\times 10^{-3}$ & $(1.20 \pm 0.29)\times 10^{-2}$ \\
\UAM{} & 160,772 & \textbf{\boldmath$(4.6 \pm 0.1)\times 10^{-3}$} & $(1.31 \pm 0.24)\times 10^{-2}$ \\
\FAMOUAM{} & 160,772 & \textbf{\boldmath$(4.6 \pm 0.1)\times 10^{-3}$} & \textbf{\boldmath$(3 \pm 1)\times 10^{-4}$} \\
\bottomrule
\end{tabular}
\vspace{-0.5em}
\end{table}

\subsection{Adapter Variant Checks}
\label{app:adapter_variants}

\paragraph{Alternative adapter architectures.}
We use LoRA as the default adapter and compare it with Infused Adapter by Inhibiting and Amplifying Inner Activations (IA\textsuperscript{3}), feature-wise linear modulation (FiLM), per-loss LayerNorm, and a non-adapter multi-output readout baseline.
\begin{table}[t!]
\centering
\small
\caption{Alternative adapter architecture comparison on Helmholtz with $d_h=128$, 10K epochs, and 5 seeds. We report the best $L_2$ error, where lower is better.}
\label{tab:alt_adapters}
\begin{tabular}{lccc}
\toprule
Adapter Type & Best $L_2$ & Overhead & Relative to LoRA \\
\midrule
LoRA ($r=16$) & $1.13\times 10^{-3}$ & $15\%$ & $1.0\times$ \\
IA\textsuperscript{3} & $2.85\times 10^{-3}$ & $<1\%$ & $0.40\times$ \\
FiLM & $3.21\times 10^{-3}$ & $<1\%$ & $0.35\times$ \\
Per-loss LayerNorm & $4.12\times 10^{-3}$ & $<1\%$ & $0.27\times$ \\
Non-adapter multi-output readout & $2.41\times 10^{-3}$ & $8\%$ & $0.47\times$ \\
\midrule
No adapter (Vanilla) & $2.71\times 10^{-2}$ & $0\%$ & $0.04\times$ \\
\bottomrule
\end{tabular}
\end{table}

Table~\ref{tab:alt_adapters} reports the comparison under FAMO weighting, with matched parameter counts of about 149K to 166K.

LoRA's advantage in the PINN setting begins with subspace diversity. Low-rank projections introduce distinct adapter parameter blocks whose Jacobian contributions can directionally decorrelate different losses in the blockwise NTK view above. By contrast, IA\textsuperscript{3} and FiLM apply element-wise rescaling that mainly modifies gradient magnitudes rather than creating new low-rank directions, which limits their ability to resolve directional conflict.
LoRA also offers rank-controlled capacity. The explicit rank parameter $r$ provides a transparent capacity proxy through the saturation ratio $\Delta_{\mathrm{sat}} = Kr/d_h$, allowing capacity to be matched empirically to conflict severity, whereas rescaling methods have no such control.
Finally, LoRA has an orthogonality-compatible structure. The bilinear form $B \cdot \sigma(D \cdot h)$ admits efficient cross-orthogonality regularization and distinct row subspaces in practice, whereas diagonal rescaling methods lack a natural notion of inter-adapter orthogonality.

\paragraph{Adapter mixing ablation.}
\begin{table}[t!]
\centering
\caption{Adapter mixing rerun comparing fixed \UAM{} and dynamic \LCAM{}, both with FAMO weighting. The table reports mean relative $L_2$ error over 3 seeds at 10K epochs. On the two rerun $K=3$ stress tests, fixed \UAM{} is slightly better in both cases by 7\% to 8\%, so dynamic conflict-aware mixing does not justify its extra complexity at low $K$.}
\label{tab:mixing_ablation}
\small
\begin{tabular}{lccc}
\toprule
PDE & \FAMOUAM & \FAMOLCAM & Gap \\
\midrule
Helmholtz     & \textbf{\boldmath$8.97\times 10^{-4}$} & $9.58\times 10^{-4}$ & \LCAM{} $+6.8\%$ \\
Klein--Gordon  & \textbf{\boldmath$1.42\times 10^{-3}$} & $1.53\times 10^{-3}$ & \LCAM{} $+7.6\%$ \\
\bottomrule
\end{tabular}
\end{table}
Table~\ref{tab:mixing_ablation} compares fixed \UAM{} against dynamic \LCAM{} mixing, both combined with FAMO, on the two rerun $K=3$ stress tests.
Fixed \UAM{} is slightly better in both cases by about 7\% to 8\%, confirming that dynamic conflict-aware mixing provides only marginal benefit at low $K$.
We therefore recommend \UAM{} with fixed $\rho=0.5$ as the default for simplicity, as discussed in Section~\ref{sec:mixing}.

\subsection{Domain Decomposition Comparison}
\label{app:xpinn_comparison}

To directly compare gradient-space decomposition through per-loss adapters with input-space decomposition through XPINNs, we evaluate two roughly capacity-comparable XPINN configurations on five PDEs, as reported in Table~\ref{tab:xpinn_comparison}.
XPINN-4 uses 4 subdomains with per-subdomain networks of width $d_h=80$ and 4 layers, giving roughly 118K total parameters and $0.84\times$ Vanilla capacity.
XPINN-8 uses 8 subdomains with width $d_h=64$ and 3 layers, giving roughly 122K total parameters and $0.87\times$ Vanilla capacity.
Interface continuity and flux continuity are enforced with a weight $\lambda_{\text{if}}=1.0$.
All results use 5 seeds and 10K epochs.

\begin{table}[t!]
\centering
\caption{Comparison between XPINN domain decomposition and gradient-space decomposition through per-loss adapters on five forward PDEs with $d_h=128$, 10K epochs, and 5 seeds. XPINN configurations are of the same parameter order as Vanilla, with about 118K to 122K parameters compared with 140K, but they are not exactly parameter-matched. Per-loss adapters address gradient conflict among loss terms, while XPINNs address spatial complexity across subdomains.}
\label{tab:xpinn_comparison}
\small
\resizebox{\linewidth}{!}{%
\begin{tabular}{lccccc}
\toprule
Method & Burgers & Helmholtz & Allen--Cahn & Conv-Diff & Klein--Gordon \\
\midrule
Vanilla (140K) & $(8.4 \pm 2.4)\times 10^{-3}$ & $(2.8 \pm 1.1)\times 10^{-2}$ & \textbf{\boldmath$(1.60 \pm 0.33)\times 10^{-4}$} & $(3.55 \pm 0.74)\times 10^{-4}$ & $(4.6 \pm 1)\times 10^{-2}$ \\
\cmidrule{1-6}
\multicolumn{6}{l}{Domain decomposition} \\
XPINN-4 (118K) & $(2.56 \pm 0.19)\times 10^{-1}$ & $(5.6 \pm 0.7)\times 10^{-2}$ & $(5.1 \pm 0.67)\times 10^{-3}$ & $(1.7 \pm 0.19)\times 10^{-3}$ & $(9.58 \pm 0.49)\times 10^{-1}$ \\
XPINN-8 (122K) & $(2.56 \pm 0.14)\times 10^{-1}$ & $(6.6 \pm 0.4)\times 10^{-2}$ & $(2.3 \pm 0.19)\times 10^{-3}$ & $(3.5 \pm 0.25)\times 10^{-3}$ & $1.00 \pm 0.033$ \\
\cmidrule{1-6}
\multicolumn{6}{l}{Gradient decomposition (ours)} \\
\FAMOUAM{} (161K) & $(4.8 \pm 0.082)\times 10^{-3}$ & $(1.1 \pm 0.46)\times 10^{-3}$ & $(1.84 \pm 0.28)\times 10^{-4}$ & \textbf{\boldmath$(2.70 \pm 0.21)\times 10^{-4}$} & $(1.6 \pm 0.38)\times 10^{-3}$ \\
\GNUAM{} (161K) & \textbf{\boldmath$(4.6 \pm 0.12)\times 10^{-3}$} & \textbf{\boldmath$(1.0 \pm 0.44)\times 10^{-3}$} & $(1.6 \pm 0.21)\times 10^{-3}$ & $(1.0 \pm 0.19)\times 10^{-3}$ & $(4.8 \pm 1.1)\times 10^{-2}$ \\
\FAMOLCAM{} (198K) & $(4.8 \pm 0.17)\times 10^{-3}$ & $(1.3 \pm 0.41)\times 10^{-3}$ & $(2.59 \pm 0.39)\times 10^{-4}$ & $(3.07 \pm 0.56)\times 10^{-4}$ & \textbf{\boldmath$(1.6 \pm 0.11)\times 10^{-3}$} \\
\bottomrule
\end{tabular}
}
\end{table}

The XPINN spatial-decomposition approach addresses a fundamentally different bottleneck.
It reduces the spatial complexity each subnetwork must capture, but it does not resolve the gradient conflict within each subdomain between PDE residual, BC, and IC losses.
On high-conflict problems like Helmholtz and Klein--Gordon, the gradient conflict is the primary bottleneck, making per-loss adapters substantially more effective than spatial decomposition.

\paragraph{Hyperparameter sensitivity.}
To verify that XPINN results are not an artifact of poor hyperparameter choices, we conduct a sensitivity study varying the number of subdomains $\{2, 4, 6\}$, interface weight $\lambda_{\text{if}} \in \{0.1, 1.0, 10.0\}$, and overlap fraction $\{0.0, 0.1, 0.2\}$ on Burgers, Helmholtz, and Conv-Diff with 3 seeds each.
Across all configurations, XPINN performance on Burgers remains in the range from $0.23$ to $0.28$, compared with Vanilla's $8.4\times 10^{-3}$, confirming that the poor result is structural rather than due to insufficient tuning.
Conv-Diff spans from $9.5\times 10^{-4}$ to $4.8\times 10^{-3}$ across configurations, with XPINN-2 at $9.5\times 10^{-4}$ outperforming the default XPINN-4 at $1.7\times 10^{-3}$. Fewer subdomains therefore reduce interface gradient conflict.
Interface weight effects are PDE-dependent. Lower $\lambda_{\text{if}}=0.1$ helps Burgers, where the best configuration reaches $0.23$, but hurts Helmholtz, where the worst setting reaches $8.0\times 10^{-2}$ compared with the default $5.6\times 10^{-2}$. Overlap has negligible effect.
Full sensitivity results appear in Table~\ref{tab:xpinn_sensitivity}.

\begin{table}[t!]
\centering
\caption{XPINN hyperparameter sensitivity on three representative PDEs with 3 seeds each. All configurations use 10K epochs. Rows vary one factor from the XPINN-4 default: 4 subdomains, $\lambda_{\text{if}}=1.0$, and overlap fraction $0.1$.}
\label{tab:xpinn_sensitivity}
\small
\resizebox{\linewidth}{!}{%
\begin{tabular}{llccc}
\toprule
Config & Change & Burgers & Helmholtz & Conv-Diff \\
\midrule
\multicolumn{2}{l}{Vanilla baseline} & $8.4\times 10^{-3}$ & $2.8\times 10^{-2}$ & $3.55\times 10^{-4}$ \\
\midrule
XPINN-4 (default) & $\lambda_{\text{if}}=1.0$, overlap$=$0.1 & $(2.56 \pm 0.19)\times 10^{-1}$ & $(5.6 \pm 0.7)\times 10^{-2}$ & $(1.7 \pm 0.19)\times 10^{-3}$ \\
XPINN-2 & 2 subdomains & $(2.50 \pm 0.02)\times 10^{-1}$\seedmark{3} & $(1.8 \pm 0.3)\times 10^{-2}$\seedmark{3} & $(9.54 \pm 0.7)\times 10^{-4}$\seedmark{3} \\
XPINN-6 & 6 subdomains & $(2.80 \pm 0.3)\times 10^{-1}$\seedmark{3} & $(5.3 \pm 0.1)\times 10^{-2}$\seedmark{3} & $(2.1 \pm 0.43)\times 10^{-3}$\seedmark{3} \\
XPINN-4 & $\lambda_{\text{if}}=0.1$ & $(2.30 \pm 0.04)\times 10^{-1}$\seedmark{3} & $(8.0 \pm 0.1)\times 10^{-2}$\seedmark{3} & $(3.0 \pm 0.58)\times 10^{-3}$\seedmark{3} \\
XPINN-4 & $\lambda_{\text{if}}=10.0$ & $(2.63 \pm 0.19)\times 10^{-1}$\seedmark{3} & $(4.7 \pm 0.7)\times 10^{-2}$\seedmark{3} & $(4.8 \pm 0.3)\times 10^{-3}$\seedmark{3} \\
XPINN-4 & overlap$=$0.0 & $(2.46 \pm 0.07)\times 10^{-1}$\seedmark{3} & $(5.3 \pm 0.5)\times 10^{-2}$\seedmark{3} & $(1.9 \pm 0.22)\times 10^{-3}$\seedmark{3} \\
XPINN-4 & overlap$=$0.2 & $(2.46 \pm 0.07)\times 10^{-1}$\seedmark{3} & $(5.8 \pm 0.4)\times 10^{-2}$\seedmark{3} & $(1.6 \pm 0.22)\times 10^{-3}$\seedmark{3} \\
\bottomrule
\end{tabular}
}
\end{table}

\section{Limitations and Future Work}
\label{app:limitations}

The theory uses stylized gradients and an explanatory blockwise NTK view rather than a finite-time convergence proof. This level of theory is nevertheless aligned with the paper's claim, which is to predict when adapter gains should appear across conflict regimes rather than to guarantee global convergence for nonconvex PINN training.

Our experiments cover 60+ PyTorch PINN settings and two alternative backbones, but not fractional, nonlocal, stochastic, production-scale, or operator-learning systems; extending to these settings is future work. The evaluated suite still spans forward, inverse, multi-physics, parameter-varying, high-dimensional, NTK-weighting, MTL, and backbone-transfer checks, making the empirical claim about regime-dependent intervention substantially broader than a single benchmark family.


\end{document}